\documentclass[3p,computermodern,11pt]{elsarticle}
\usepackage{fullpage}
\usepackage{booktabs} 
\usepackage{algorithm}
\usepackage{algorithmic}
\usepackage[utf8]{inputenc}
\usepackage[english]{babel}
\usepackage{amsmath,amsfonts,amssymb,amsthm}
\usepackage{ulem}

\newtheorem{proposition}{Proposition}[section]

\newtheorem{remark}{Remark}[section] 
\usepackage{color}
\usepackage{xcolor}
\usepackage{mathtools}
\usepackage[titletoc,title]{appendix}
\usepackage{float}
\usepackage{graphicx} 
\usepackage{epstopdf} 
\usepackage{caption}[font=small,skip=-2pt]
\usepackage{subcaption}[font=small,skip=-2pt]
\usepackage{pdfpages}
\usepackage{comment}
\usepackage{tabularx}
\usepackage{listings}

\usepackage{hyperref}
\hypersetup{
    colorlinks=true,
    linkcolor=blue,
    filecolor=magenta,      
    urlcolor=cyan,
}
\usepackage{tikz}
\usetikzlibrary{positioning, fit, arrows.meta, shapes}

\usetikzlibrary{
  arrows.meta, 
  fit, 
  positioning, 
}
\tikzset{
  neuron/.style={ 
    circle,draw,thick, 
    inner sep=0pt, 
    minimum size=3.5em, 
    node distance=1ex and 2em, 
  },
  group/.style={ 
    rectangle,draw,thick, 
    inner sep=0pt, 
  },
  io/.style={ 
    neuron, 
    fill=gray!15, 
  },
  conn/.style={ 
    -{Straight Barb[angle=60:2pt 3]}, 
    thick, 
  },
}
\numberwithin{equation}{section}

\lstdefinestyle{nnopinf}{
  language=Python,
  basicstyle=\ttfamily\small,
  keywordstyle=\color{blue!60!black},
  commentstyle=\color{green!50!black},
  stringstyle=\color{orange!70!black},
  backgroundcolor=\color{gray!10},
  frame=single,
  rulecolor=\color{gray!60},
  breaklines=true,
  showstringspaces=false
}

\newcommand{\nnopinf}{NN-OpInf}

\newcommand{\HTwo}{\mathrm{H}_2}
\newcommand{\OTwo}{\mathrm{O}_2}
\newcommand{\HTwoO}{\mathrm{H}_2\mathrm{O}}
\newcommand{\code}[1]{{\ttfamily\color{black} #1}}

\newcommand{\reducedOperatorLinear}{\hat{\mathbf{A}}}
\newcommand{\reducedOperatorQuadratic}{\hat{\mathbf{H}}}
\newcommand{\range}[1]{\mathrm{range}\left(#1\right)}

\newcommand{\nBfgsSteps}{n_{\mathrm{lbfgs}}}

\newcommand{\reducedLhs}{\dot{\hat{\state}}}
\newcommand{\reducedState}{\hat{\state}}
\newcommand{\reducedStateDot}{\dot{\hat{\state}}}

\newcommand{\reducedStateCoDomain}{\RR{\romDim}}

\newcommand{\EP}[1]{{\color{red}EP: #1}}
\definecolor{blue}{cmyk}{1.00, .34, .0, .02}

\newcommand{\reducedSkewMatrix}{\hat{\boldsymbol S}}

\newcommand{\reducedOperatorConstant}{\hat{\mathbf{c}}}

\newcommand{\snapshotsLhs}{\dot{\mathbf{X}}}
\newcommand{\reducedSnapshotsLhs}{\dot{\hat{\snapshotsState}}}
\newcommand{\penaltyFunction}{p}
\newcommand{\penaltyParameter}{\lambda}
\newcommand{\dt}{\Delta t}
\newcommand{\nOperators}{M}
\newcommand{\reducedVelocityNN}{\hat{\boldsymbol g}}
\newcommand{\networkInputs}{\boldsymbol \eta}

\newcommand{\nHidden}{n_{\mathrm{h}}}
\newcommand{\nNeurons}{n_{\mathrm{n}}}
\newcommand{\nOpsActivation}{n_{o,a}}
\newcommand{\ops}{\hat{\boldsymbol O}}

\newcommand{\reducedStiffnessMatrixNN}{\hat{\boldsymbol A}}
\newcommand{\reducedSpdMatrixNN}{\hat{\boldsymbol K}}
\newcommand{\reducedSpdMatrixLowerNN}{\hat{\boldsymbol L}}

\newcommand{\reducedSkewMatrixLowerNN}{\hat{\boldsymbol S}}

\newcommand{\weights}{\mathbf{w}}

\newcommand{\reducedSpdMatrix}{\hat{\boldsymbol K}}

\newcommand{\reducedMatrix}{\hat{\boldsymbol A}}
\newcommand{\reducedMlMatrix}{\hat{\boldsymbol A}_{\mathrm{ML}}}

\newcommand{\nt}{N_t}
\newcommand{\sqr}[1]{\mathrm{sqr}\big(#1 \big)}
\newcommand{\snapshotsState}{\mathbf{X}}

\newcommand{\reducedSnapshotsState}{\hat{\mathbf{X}}}

\newcommand{\mass}{\boldsymbol M}

\newcommand{\reducedVelocity}{\skew{4}\hat{\velocity}}

\newcommand{\defeq}{:=}

\newcommand{\state}{\boldsymbol x}

\newcommand{\stateDummy}{\boldsymbol y}

\newcommand{\approxState}{\tilde{\state}}

\newcommand{\romTrialSpace}{\mathcal{V}}
\newcommand{\romBasis}{\mathbf{V}}

\newcommand{\stateDot}{\dot{\state}}
\newcommand{\velocity}{\boldsymbol f}

\newcommand{\params}{\boldsymbol \mu}
\newcommand{\paramDomain}{\mathcal{D}}
\newcommand{\paramDomainTrain}{\paramDomain_{\mathrm{train}}}
\newcommand{\paramDomainTest}{\paramDomain_{\mathrm{test}}}

\newcommand{\timeDomain}{[0,T]}
\newcommand{\spatialDomain}{\Omega}
\newcommand{\RR}[1]{\mathbb{R}^{#1}}
\newcommand{\fomDim}{N}
\newcommand{\romDim}{K}

\newcommand{\stateCoDomain}{\RR{\fomDim}}
\newcommand{\nParams}{N_{\params}}
\newcommand{\paramFullDomain}{\RR{\nParams}}

\newcommand{\reducedStateDummy}{\hat{\stateDummy}}

\newcommand{\nTrain}{N_{\text{train}}}

\begin{document}
\begin{frontmatter}

\title{NN-OpInf: an operator inference approach using structure-preserving composable neural networks}
\author[a]{Eric Parish}
\ead{ejparis@sandia.gov}
\author[b]{Anthony Gruber}
\author[a]{Patrick Blonigan}
\author[a]{Irina Tezaur}

\address[a]{Sandia National Laboratories, Livermore, CA, USA}
\address[b]{Sandia National Laboratories, Albuquerque, NM, USA}

\begin{abstract}
We propose neural network operator inference (NN-OpInf): a structure-preserving, composable, and minimally restrictive operator inference framework for the non-intrusive reduced-order modeling of dynamical systems. The approach learns latent dynamics from snapshot data, enforcing local operator structure such as skew-symmetry, positive (semi-)definiteness, and gradient preservation, while also reflecting complex dynamics by supporting additive compositions of heterogeneous operators. We present practical training strategies and analyze computational costs relative to linear and quadratic polynomial OpInf (P-OpInf). Numerical experiments across several nonlinear and parametric problems demonstrate improved accuracy, stability, and robustness over P-OpInf and prior NN-ROM formulations, particularly when the dynamics are not well represented by polynomial models.  These results suggest that NN-OpInf can serve as an effective drop-in replacement for P-OpInf when the dynamics to be modeled contain non-polynomial nonlinearities, offering potential gains in accuracy and out-of-distribution performance at the expense of higher training computational costs and a more difficult, non-convex learning problem.
\end{abstract}
\end{frontmatter}

\section{Introduction}
Data-driven model order reduction (MOR) remains an important tool for enabling many-query analyses of high-fidelity and high-cost computational simulations. Broadly speaking, such reduced-order models (ROMs) separate into two categories based on whether they require direct access to the high-fidelity full-order model (FOM) code: intrusive and non-intrusive. Intrusive ROMs, such as Galerkin projection using proper orthogonal decomposition (POD-Galerkin) and least-squares Petrov--Galerkin (LSPG)~\cite{CaBaAn17}, directly project the FOM governing equations onto a ``maximally expressive'' reduced-subspace identified from data. Ideally, this projection process drastically reduces the computational cost in exchange for a small amount of approximation error. While effective, these methods rely on access to the FOM operators, and their intrusive nature can be prohibitive for certain types of codes where this access is not practical. 

\subsection{Related work}

Non-intrusive ROMs address the aforementioned challenge by constructing ROMs using ``snapshot'' data alone, i.e., instantaneous samples of representative FOM solutions at various parameter values. Non-intrusive model reduction methods typically have two steps. First, and similar to intrusive ROMs, a low-dimensional subspace is identified from high-fidelity snapshot data via, e.g., the POD~\cite{BeHoLu93}. Second, the snapshots are projected onto this low-dimensional subspace and a model for the latent dynamics is learned via a regression procedure. The choice of latent space dynamics model to be learned plays a critical role in the accuracy of a non-intrusive ROM. Polynomial operator inference (P-OpInf or simply OpInf)~\cite{PeWi16}, a particularly popular non-intrusive model reduction technique, enforces these dynamics to be polynomial in the reduced state. This structure is motivated by the observation that many models arising in computational physics can be written in a polynomial form; examples include linear elasticity, the shallow water equations, and the Euler equations. Imposing polynomial structure on the latent space dynamics simplifies the inference process to a convex least-squares problem and further enables guarantees on the inferred operators under certain conditions. For example, in the case of a polynomial-based FOM, the P-OpInf ROM can exactly recover the intrusive Galerkin ROM via targeted snapshot collection procedures such as ``re-projection''~\cite{Pe20} and ``exact'' OpInf~\cite{RoSaSt25}. 

A myriad of extensions to P-OpInf have been proposed since~\cite{PeWi16}. The so-called ``Lift and learn''~\cite{QiKrPe20} approach was introduced as a systematic framework for learning ROMs in augmented state spaces to mitigate non-polynomial non-linearities. P-OpInf was formulated for partial differential equations in~\cite{QiFaIo22} from a function space viewpoint. Structure-preserving P-OpInf approaches were proposed for Hamiltonian systems~\cite{ShWaKr22,ViMcGr25,GrTe23,GrTe25}, Lagrangian systems~\cite{ShKa24,ShNaTo24}, energy-conserving systems~\cite{KoQi24,GkDuGo26}, and systems with gradient structure~\cite{GeSiJu24}.  
Recently, methods to couple subdomain-local P-OpInf ROMs with other subdomain-local P-OpInf ROMs and/or full-order models to improve predictive accuracy and robustness have been explored \cite{TePaGr26, Moore:2024, Rodriguez:2025, Gkimisis:2025, Farcas:2024}.
While the original P-OpInf relied on interpolation to handle parametric dependence, more efficient extensions to ``affine'' systems have also been proposed~\cite{McKhWi23,ViMcGr25}.  Under the affine parametric assumption, FOM dynamics are linear in a known nonlinear function of the parameters, leading to effective P-OpInf ROMs which retain convexity of the learning problem.

While P-OpInf has seen success, many systems of practical interest do not admit a polynomial model form. Common examples include finite deformation solid mechanics, the Reynolds-averaged Navier--Stokes equations, and fluid flow with reacting gases. While lifting procedures~\cite{QiKrPe20} provide an approach to circumvent this issue in certain situations, an exact transformation does not always exist, or it may be infeasible to derive. Non-affine parametric dependence is also common in the form of exponential or rational nonlinearities involving both states and parameters. In these circumstances, P-OpInf methods with affine parametric dependence have limited accuracy due to their strong inductive bias, and it is desirable to have a higher capacity inference method for learning latent space dynamics. To that end, many works have replaced the global polynomial approximation in P-OpInf with a kernel surrogate or artificial neural network.  
Hesthaven and Ubbiali~\cite{HeUb18} and Wang et
al.~\cite{WaHeRa19}, for instance, propose neural-network-based non-intrusive model
reduction approach for steady~\cite{HeUb18} and transient parameterized
systems~\cite{WaHeRa19}. Their approaches employ POD for dimension reduction
and use feed-forward neural networks to map input parameters, and time in the
transient setting, to the target reduced coefficients. Unlike P-OpInf, however, this approach does not learn dynamics, but rather treats time as a parameter. 

An early dynamics learning approach was the work of San et al.~\cite{SaMaAh19}, which examined a neural-network-based surrogate modeling approach similar to operator inference and demonstrated it on Burgers' equation. The approach employed POD for dimension reduction and fully connected neural networks to learn the one-step dynamics mapping between POD coefficients supplemented with problem-relevant parameters; the time coordinate and Reynolds number were used for the viscous Burgers' equation example. Two network architectures were considered in \cite{SaMaAh19}: a ``sequential net'' that directly targeted the POD coefficients at the next step, and a ``residual net'' that targeted the difference between the POD coefficients from one step to the next, which can be viewed as an explicit Euler update.  This is related to the approach taken in \cite{FrDeMa21}, which uses two deep neural networks in sequence: an encoder-decoder to produce an informative latent space, followed by a dynamics predictor which handles future latent state prediction.  Since this procedure is highly sensitive to the network architecture used, a comparison of various encoder-decoder architectures was presented in \cite{GrGuJu22}.
These works demonstrated that residual-based neural networks can lead to ROMs capable of out-performing Galerkin projection. However, comparisons were not made with polynomial-based OpInf~\cite{PeWi16}, which can also outperform Galerkin projection in certain circumstances. Furthermore, these approaches did not embed any physical structure within the ROM.

A similar neural-network-based effort was proposed in Gao et al.~\cite{GaWaZh20}, where neural networks were used to learn a right-hand side operator with the goal of bypassing hyper-reduction in intrusive model reduction
Unlike the works~\cite{SaMaAh19,FrDeMa21}, which learn the mapping from one time step to another, Gao et al. directly target learning the right-hand side ``velocity'' operator for the reduced model. POD is again used for dimension reduction and training data for the neural network comprises snapshot data of the right-hand side operator evaluated about the FOM solution projected onto the POD subspace. The proposed approach was demonstrated on two different problems: Burgers' equation with non-affine parametric dependence and a nonlinear convection-diffusion-reaction system based on a premixed $H_2$-air flame model. Parametric variation and convergence with basis dimension were considered. Results showed that the neural network formulation tended to outperform hyper-reduced intrusive ROMs based on the discrete empirical interpolation method (DEIM)~\cite{ChSo10}, particularly for lower basis dimensions. Gao et al. further observed that the NN-ROM formulation did not exhibit ``deep convergence'' for large numbers of basis vectors and stagnated at relative errors of around $1\%$; the NN-ROMs became an order of magnitude less accurate than DEIM-based ROMs and several orders of magnitude less accurate than non-hyper-reduced ROMs for higher basis dimensions. Comparisons were not made to polynomial-based operator inference approaches. To remedy some of these deficiencies, the related work~\cite{DiNeTe25} considered replacing the right-hand velocity operator with a kernel surrogate instead of a neural network, demonstrating that this class of approximators compares favorably to P-OpInf in various non-polynomial cases of interest.

\subsection{Our contributions}
This work aims to advance OpInf-based approaches by introducing a
structure-preserving, composable, minimally restrictive, and non-intrusive
framework for the model reduction of dynamical systems.  The proposed neural
network operator inference (NN-OpInf) approach is motivated by fundamental
limitations of existing operator inference methodologies and advances the state
of the art in several ways.

First, as previously discussed, many systems of practical interest do not admit
a polynomial operator structure, limiting the utility of traditional P-OpInf
approaches. To address this limitation, 
NN-OpInf uses generic feed-forward neural networks capable of representing
general nonlinear operators. While prior studies have also explored approaches analogous to NN-based
operator inference~\cite{SaMaAh19,GaWaZh20}, this work is distinguished in its
modular approach: we introduce novel neural architectures with prescribed
algebraic structure, which can be additively composed to improve learning
stability and expressiveness.  This composability provides a flexible
reduced-order modeling strategy when structure is known for individual
contributions but a suitable  global formalism is unavailable or is not
assumed.

Second, discretizations of many governing equations yield operators with intrinsic mathematical structure, such as skew-symmetry, (symmetric) positive definiteness, and the property of being an exact derivative (gradient structure).  This implicitly encodes physical consequences including dynamical stability, conservation laws, and dissipation inequalities.  Since preserving this structure in ROMs can be essential for robustness, interpretability, and accuracy, we tailor NN-OpInf to structure-preserving discretizations.  Besides offering generic vector- and matrix-valued surrogates, we implement skew-symmetric, symmetric positive-definite, and gradient operators, enabling NN-OpInf ROMs that explicitly respect these key properties by construction. While NN-OpInf is inherently non-intrusive, we emphasize that enforcing such algebraic structure implicitly assumes some level of knowledge of the governing equations. 

Third, we emphasize modularity: while individual operators in a full-order model (FOM) may exhibit distinguished algebraic structure, the fully discretized system rarely admits a global physical characterization in terms of this structure. For example, hyper-elastic solid mechanics models typically combine symmetric positive definite stiffness operators with unstructured forcing and boundary condition operators, precluding monolithic structure-preserving formulations based on, e.g., global Hamiltonian dynamics. To overcome this limitation, we propose a composable operator inference paradigm in which the NN-OpInf reduced model is expressed as an additive sum of learned operators. Each operator can enforce its own distinct mathematical structure, accept different inputs, and employ a tailored neural architecture. This composability enables flexible yet principled reduced-order modeling of complex systems that would otherwise be inaccessible to existing approaches.

Fourth, to support the development, training, and deployment of NN-OpInf models, we developed \code{NN-OpInf}~\cite{nnopinf}, an open-source software package for constructing, training, and deploying composable, structure-preserving neural network operator inference models. The package provides a modular API built around variables, operators, and models, enabling users to mix and match heterogeneous operators in a single ROM. We provide utilities for data handling, loss construction, and training.  

Lastly, we provide a comprehensive assessment of the accuracy of the proposed NN-OpInf approach on a range of nonlinear and parametric problems. We assess convergence with basis dimensions and predictive accuracy within both reproductive and predictive regimes. We compare NN-OpInf with intrusive ROMs, P-OpInf models, and several other structure preserving OpInf models, where appropriate; to our knowledge, we provide the first systematic comparison between neural-network-based and polynomial-based operator inference methods across multiple problem classes.


 Before proceeding, we emphasize that the primary target of this work
is to advance polynomial operator inference
(P-OpInf)~\cite{PeWi16,Pe20,RoSaSt25,McKhWi23,ViMcGr25,QiKrPe20} and baseline
neural network operator inference approaches~\cite{SaMaAh19,FrDeMa21}.
There are, however, a number of closely related approaches in the broader
data-driven modeling literature. Neural ODEs~\cite{ChRuBe19} (NODEs), for
example, similarly seek to learn the right-hand side of a dynamical system
using a neural network. Unlike standard P-OpInf, however, NODEs are typically
trained on state data using backpropagation of gradient information through
time. This training approach can improve predictive reliability, but also
substantially increases training cost. Parameterized NODEs, which learn a
parameterized right-hand side, were considered in~\cite{LePa21}. We
additionally note that Ref.~\cite{WaHeRa19} considers traditional P-OpInf
models learned using backpropagation through time, referred to as rollouts, and
showed improved predictive accuracy. 

The flexibility of neural networks has also enabled many approaches for
incorporating known physical and mathematical structure into learned dynamics,
including Hamiltonian neural networks and their
generalizations~\cite{GrDzYo19,ElGaHu24,DuAlSt24}, Lagrangian neural
networks~\cite{CrGrHo20,FiWaWi20}, and metriplectic neural
networks~\cite{LeTrSt21,GrLeTr23,GrLeLi25}. Numerous variations on these ideas
exist in the literature. Park et al.~\cite{PaChCh24}, for example, propose the
thermodynamics-informed latent space dynamics identification (tLaSDI) approach,
based on their previous LaSDI approach~\cite{FrHeCh22}. In tLaSDI, nonlinear
dimension reduction is used to identify latent states, and a GENERIC formalism
informed neural network~\cite{ZhShKa22} (GFINN) is used to construct a
thermodynamics-preserving latent dynamics surrogate. 
As OpInf shares many similarities with these broader modeling approaches,
NN-OpInf also shares similarities with efforts to develop data-driven and
reduced-order models that incorporate known structure into the learned
dynamics.
The distinction lies both in modeling scope and training approach. First,
global formalisms assume that the full dynamics admit a prescribed geometric or
thermodynamic representation and enforce the corresponding system-level
compatibility conditions. In contrast, NN-OpInf targets the operator-inference
setting and represents the reduced velocity as a sum of heterogeneous operator
blocks. Algebraic constraints can be imposed separately on individual blocks,
while other contributions, such as external forcing, boundary effects, or
parametric terms, may remain unconstrained. This provides greater flexibility
in terms of modeling capability, but it generally offers
weaker global physical guarantees than a correctly specified Hamiltonian or
GENERIC architecture. NN-OpInf should therefore be viewed as complementary to,
rather than competing with, these more global formulations. Indeed, several existing global formalisms can 
be implemented with the machinery present in NN-OpInf.

Second, in the present work we focus on NN-OpInf models that are trained
without rollouts. Specifically, the reduced operators are learned directly from
snapshot states and corresponding velocity data, so evaluation of the training
objective does not require integrating the learned dynamical system forward in
time. Neural ODE approaches, in contrast, commonly formulate training in terms
of the mismatch between observed trajectories and trajectories obtained by
integrating the inferred ROM. In this setting, a rollout of the learned
dynamics is embedded directly within each evaluation of the training objective,
and gradients are propagated through the associated time-integration procedure.
This can provide a direct mechanism for optimizing finite-time trajectory
accuracy, but also couples training to the
numerical time integrator and increases the computational cost of
each optimization iteration. The rollout-free formulation considered here
retains the more traditional operator inference structure, in which the right-hand side dynamics are identified locally from state--velocity data and time integration is
performed only after training when the resulting reduced-order model is
simulated. We emphasize, however, that it is straightforward to implement rollout-based loss functions within the broader NN-OpInf framework when trajectory-level
training is desirable.

The structure of this paper is as follows. Section~\ref{sec:fom} introduces the full-order model, notation, and problem setup. Section~\ref{sec:modelreduction} reviews projection-based ROMs, summarizes polynomial OpInf, and introduces a baseline neural network OpInf formulation. Section~\ref{sec:nnopinf} presents the proposed NN-OpInf framework, including structured operators and training considerations. Section~\ref{sec:analysis} analyzes computational costs and problem convexity. Section~\ref{sec:results} reports numerical experiments across a range of nonlinear and parametric problems. Finally, Section~\ref{sec:conclusions} concludes and outlines perspectives.

\section{Full-order model}\label{sec:fom}
We consider non-intrusive reduced-order modeling of first- and second-order dynamical systems in this work. For simplicity of presentation, the approach is now presented in the context of the first-order system of ordinary differential equations (ODEs) 
\begin{equation}\label{eq:fom_ode}
\begin{split}
\dot{\state} \left(t,\params\right) = \velocity(\state(t,\params),\params),\\
\end{split}
\end{equation}
supplemented with appropriate initial conditions $\state_0(\params) \coloneqq \state(0,\params)$. In the above, $\state: \timeDomain \times \paramDomain \rightarrow \stateCoDomain$ is the state, $\params \in \paramDomain \subset \paramFullDomain$ are system parameters, $t \in [0,T]$ is the time coordinate, and $\velocity:  \RR{\fomDim} \times \paramDomain \rightarrow \RR{\fomDim}$ is the right-hand side operator. We use the notation $\dot{\state} = \frac{d}{dt} \state$. The FOM  \eqref{eq:fom_ode} is assumed to arise, for example, from the semidiscretization of a system of partial differential equations (PDEs), so that $N \gg 0$ is large and \eqref{eq:fom_ode} is expensive to simulate. While we do not explicitly consider the case where the operator $\velocity$ has exogenous inputs, this setting can be addressed with straightforward modifications to what follows.
\section{Model reduction}\label{sec:modelreduction}
Our objective is to develop a ROM for the FOM~\eqref{eq:fom_ode} of the form
\begin{equation}\label{eq:rom_ode}
\begin{split}
\dot{\reducedState}\left(t,\params \right) =& \reducedVelocity(\reducedState(t,\params),\params),\\
\end{split}
\end{equation}
where $\reducedState: \timeDomain \times \paramDomain \rightarrow \reducedStateCoDomain$ is the reduced state of dimension $K\ll N$ and $\reducedVelocity: \RR{\romDim} \times \paramDomain  \rightarrow \RR{\romDim}$ is the ``reduced'' velocity.
Classic model reduction methods follow an offline-online paradigm, where comparatively expensive operations offline are used to dramatically reduce simulation costs online. In the offline stage, a low-dimensional trial space for the state is constructed from full-order simulation data, i.e., snapshots of $\state$ at various time and parameter instances. This trial subspace is then used to build a ROM via projection (e.g., an intrusive Galerkin ROM) or inference (e.g., an OpInf ROM). In the online stage, the resulting reduced-order model is queried to generate fast approximate solutions for new testing configurations.
\subsection{Reduced basis approximation}
Many popular ROM methods generate approximate solutions $\approxState(t,\params) \approx \state(t,\params)$ to the FOM ODE~\eqref{eq:fom_ode} by restricting the state at time $t \in \timeDomain$ to a trial subspace $\romTrialSpace$ of dimension $\romDim \ll \fomDim$\footnote{We note the use of an affine trial subspace is also common.}.  These so-called reduced basis methods often consider an approximation $\approxState \in  \romTrialSpace \subset \RR{\fomDim}$ to the full-order state, where $\text{dim}(\romTrialSpace) = \romDim$. The most common approach for constructing $\romTrialSpace$ is POD~\cite{BeHoLu93}, which assumes access to a set of snapshots corresponding to realizations of the state at $\nt$ time instances for $\nTrain$ parameter realizations.  Collecting these data into a global snapshot matrix,
\begin{equation}\label{eq:state_snapshots}
\snapshotsState = \left[ \snapshotsState^{\params_1} , \cdots, \snapshotsState^{\params_{\nTrain}} \right] \in \RR{\fomDim \times \nt \nTrain},
\end{equation}
where
$$\snapshotsState^{\params_i} = \left[ \state(t_1,\params_i), \cdots, \state(t_{\nt},\params_{i}) \right] \in \RR{\fomDim \times \nt},$$
POD identifies a $\romDim$-dimensional trial subspace $\romTrialSpace = {\rm span}\,\romBasis$ described by an orthonormal basis matrix $\romBasis\in\mathbb{R}^{\fomDim\times\romDim}$ that minimizes the projection error,
$$\underset{\romBasis^\intercal \romBasis = \mathbf{I} }{\mathrm{minimize} }\,\, \big\|  \snapshotsState - \romBasis \romBasis^\intercal  \snapshotsState  \big\|_F^2,$$
where $|\cdot|_F$ denotes the Frobenius matrix norm.
This optimization problem is classical and can be effectively solved via the singular value decomposition of the snapshot matrix $\snapshotsState$, 
yielding the POD approximation
\begin{equation}\label{eq:sdrApproximation}
\romBasis \reducedState(t,\params) = \approxState(t,\params) \approx \state(t,\params),
\end{equation}
where $\reducedState : \timeDomain \times \paramDomain \rightarrow \RR{\romDim}$ is the reduced state, i.e., the coordinates of the approximation $\approxState$ in the basis $\romBasis$.

\subsection{Intrusive Galerkin projection and motivation for operator inference}
Given a reduced basis, many common intrusive model reduction approaches employ Galerkin projection. The Galerkin ROM is constructed by (1) inserting the approximation~\eqref{eq:sdrApproximation} into the FOM ODE~\eqref{eq:fom_ode} and (2) restricting the residual to be orthogonal to the spatial trial subspace $\romBasis$. This process yields the reduced system
\begin{equation}\label{eq:galerkinRom}
\reducedStateDot(t,\params) = \romBasis^\intercal \velocity\big(\romBasis \reducedState(t,\params),\params\big). 
\end{equation}
This process results in a $\romDim$-dimensional system that is variationally consistent\footnote{i.e., the ROM residual is constrained to be orthogonal to the trial subspace, so the ROM reproduces any FOM solution that lies in that subspace} with the original FOM. We refer to the system~\eqref{eq:galerkinRom} as the Galerkin ROM.

While the Galerkin ROM is effective in numerous circumstances, it has several drawbacks that motivate OpInf approaches. First, it is intrusive: simulation of $\eqref{eq:galerkinRom}$ requires access to the FOM operator $\velocity$, which is often impractical for, e.g., legacy or commercial codes. Second, the reduction in system dimension enabled by Galerkin projection does not always yield a reduction in computational cost, particularly in non-polynomially nonlinear systems. Specifically, computing the right-hand side of the Galerkin ROM~\eqref{eq:galerkinRom} requires a series of steps. First, the $\fomDim$-dimensional state is evaluated from the basis matrix $\romBasis$ and the reduced state $\reducedState$. Next, the $\fomDim$-dimensional velocity operator $\velocity$ is computed. Finally, the matrix--vector product $\romBasis^\intercal \velocity$ is formed. Unfortunately, all of these operations scale with FOM dimension $N$ and are computational bottlenecks for the ROM. Hyper-reduction techniques~\cite{BARRAULT2004667,ChSo10,YANO20191104,ecsw,Carlberg:2011, CarlbergGnat} have been devised to address these bottlenecks, but typically result in less accurate ROMs due to their use of additional (potentially variationally inconsistent) approximations in the governing equations. On the other hand, non-intrusive model reduction approaches can address both of these issues by pre-supposing a scalable model form and fitting the constituent operators to data. Note that such non-intrusive approaches can remain advantageous even when an alternative intrusive approach is available, due to their ability to bypass hyper-reduction~\cite{GaWaZh20}. 

\subsubsection{Operator inference}
Non-intrusive ROMs circumvent the aforementioned challenges by constructing a ROM from snapshot data alone. Arguably, the most popular non-intrusive ROMs are based on P-OpInf \cite{PeWi16}, which is a three-step procedure.  First, a reduced basis is constructed as in intrusive model reduction.  Then, an ansatz for the functional form of the ROM is made in terms of learnable operators.  Finally, an inference problem for these operators is solved to obtain a non-intrusive ROM. Assuming a quadratic functional form for the state dynamics, a generic P-OpInf ROM takes the form
\begin{equation}\label{eq:opinf_quadratic}
\dot \reducedState(t,\params)= \reducedOperatorLinear(\params) \reducedState(t,\params) + \reducedOperatorQuadratic(\params)\, \sqr{\reducedState(t,\params)} + \reducedOperatorConstant(\params) 
\end{equation}
where sqr encodes the $s = K(K+1)/2$ unique terms in the symmetric Kronecker square of its argument:
$$\sqr{\reducedStateDummy} = \left[ \reducedStateDummy_1 \reducedStateDummy_1, \reducedStateDummy_2 \reducedStateDummy_1, \reducedStateDummy_2 \reducedStateDummy_2, \reducedStateDummy_3 \reducedStateDummy_1, \reducedStateDummy_3 \reducedStateDummy_2 ,\reducedStateDummy_3 \reducedStateDummy_3,\reducedStateDummy_4 \reducedStateDummy_1,\cdots, \reducedStateDummy_\romDim \reducedStateDummy_\romDim \right]^\intercal \in \RR{s}.$$
Here, $\reducedOperatorLinear : \paramDomain \rightarrow \RR{\romDim \times \romDim}$ is a reduced matrix modeling the linear dependence of the dynamics on the state, $\reducedOperatorQuadratic : \paramDomain \rightarrow \RR{\romDim \times s}$ models quadratic dependence of the dynamics on the state, and $\reducedOperatorConstant: \paramDomain \rightarrow \RR{\romDim}$ models a potentially parameter-dependent constant forcing. Note that higher-order polynomial nonlinearities are applicable but omitted here for brevity.

Operator inference requires training data in the form of state snapshots, along with a dictionary of parameters and time coordinates associated with them. Accommodating parametric variation in the polynomial operators is not trivial but can be handled in several ways. In the initial work~\cite{PeWi16}, parametric dependence was handled by solving a P-OpInf problem for each parameter realization and forming an interpolant of the learned operators. In subsequent years, different approaches have been proposed for handling ``affine'' parametric dependence~\cite{McKhWi23,GoBeKa21,ViMcGr25}, where the operators depend linearly on a known nonlinear function of the parameters. For simplicity, the present work considers the former interpolatory case. To this end, along with the $\nTrain$ snapshot matrices 
given in Eq.~\eqref{eq:state_snapshots}, we assume access to the associated parameter realizations
$\left[\params_1,\ldots,\params_{\nTrain} \right]$,
as well as to (possibly approximate) snapshot data of the state velocity at these same temporal and parametric indices, 
$$\dot{\snapshotsState}^{\params_i} \approx  \left[ \dot{\state}(t_1,\params_i), \cdots, \dot{\state}(t_{\nt},\params_{i}) \right] \in \RR{\fomDim \times \nt}, \qquad i=1,\ldots,\nTrain.$$
In practice, the approximate derivatives $\dot{\snapshotsState}^{\params_i}$ are often obtained by numerically differentiating the states in $\snapshotsState^{\params_i}$, but in some cases they may be directly available from the high-fidelity FOM solver.
 
With these data, P-OpInf is a two-step procedure for inferring the polynomial operators in \eqref{eq:opinf_quadratic}.  First, the training snapshot matrices are projected into the reduced space, i.e., for $i=1,\ldots,\nTrain$, 
 $$\reducedSnapshotsState^{\params_i} = \romBasis^\intercal \snapshotsState^{\params_i} \in \RR{\romDim \times \nt}, \qquad \reducedSnapshotsLhs^{\params_i} = \romBasis^\intercal \snapshotsLhs^{\params_i} \in \RR{\romDim \times \nt}.$$
Then, $\nTrain$ optimization problems are solved to obtain the reduced operators $\ops^{\params_i} = \{\reducedOperatorConstant^{\params_i},\reducedOperatorLinear^{\params_i},\reducedOperatorQuadratic^{\params_i}\}$ for each parameter realization $\params_i$, $i=1,\ldots,\nTrain$,
\begin{equation}\label{eq:opinf_optimization}
\underset{\ops^{\params_i}}{\mathrm{minimize}}\, \sum_{j=1}^{\nt}\Big\| \reducedOperatorLinear^{\params_i} \reducedSnapshotsState^{\params_i}_j  + \reducedOperatorQuadratic^{\params_i}\, \sqr{ \reducedSnapshotsState^{\params_i}_j } + \reducedOperatorConstant^{\params_i} - \reducedSnapshotsLhs^{\params_i}_j \Big\|_2^2 + \penaltyFunction \big(\penaltyParameter,\ops^{\params_i} \big),
\end{equation}
where we use $\reducedSnapshotsState^{\params_i}_j$ to denote the $j$th column of $\reducedSnapshotsState^{\params_i}$ and the same for $\reducedSnapshotsLhs^{\params_i}$. The term $\penaltyFunction(\penaltyParameter,\ops^{\params_i}) \in \RR{+}$ is a penalty function (e.g., Tikhonov regularization) depending on input parameter $\penaltyParameter \in \RR{+}$ that provides regularization and enhances stability.\footnote{We note that it is possible to have multiple regularization parameters in P-OpInf. This is particularly true for higher-order P-OpInf models, wherein you can have a different regularization parameter for each operator. We omit this here for brevity.} 
After solving the problems \eqref{eq:opinf_optimization}, global operators for $\reducedOperatorLinear$, $\reducedOperatorQuadratic$, and $\reducedOperatorConstant$ are formed via interpolation, and equation \eqref{eq:opinf_quadratic} is solved for new parameter instances $\params$ not seen during training. 

\subsection{Vanilla OpInf with neural networks}
While the P-OpInf just described is relatively fast to execute and has been effective as a data-driven ROM for numerous systems, its assumption of polynomial structure and interpolatory treatment of parametric dependence limits its expected accuracy in many cases of practical interest. Neural network-based OpInf approaches can, in theory, address this challenge. 

Neural networks employ repeated function composition to create nonlinear input--output maps, providing a straightforward approach to mitigate the limited capacity of P-OpInf for non-polynomial systems. A direct extension of the previous P-OpInf to this setting replaces the polynomial ansatz in~\eqref{eq:opinf_quadratic} with a neural network $(\reducedState,\params)\mapsto\reducedVelocity$ that maps reduced states (and, when relevant, parameters) to the reduced velocity. In this ``vanilla'' NN-OpInf formulation, the reduced dynamics are modeled as
\begin{equation}\label{eq:vanilla_nn_opinf}
\dot\reducedState(t,\params) = \reducedVelocityNN(\reducedState,\params;\weights),
\end{equation}
where $\reducedVelocityNN(\reducedState,\params;\weights)$ is a standard feed-forward neural network with trainable weights $\weights$ that approximates the reduced velocity $\reducedVelocity(\reducedState,\params)$. The model is trained by minimizing a regression loss between the network outputs 
at the given state data and the corresponding reduced snapshot 
of the time derivative. More precisely, the (optionally regularized) inference problem becomes
\begin{equation}\label{eq:vanilla_nn_opinf_opt}
\underset{\weights}{\mathrm{minimize}} \;\frac{\displaystyle\sum_{i=1}^{\nTrain}\sum_{j=1}^{\nt} \Big\| \reducedVelocityNN\big(\reducedSnapshotsState^{\params_i}_j,\params_i;\weights\big) - \reducedSnapshotsLhs^{\params_i}_j \Big\|_2^2}{\displaystyle\sum_{i=1}^{\nTrain}\sum_{j=1}^{\nt} \Big\| \reducedSnapshotsLhs^{\params_i}_j \Big\|_2^2} + \penaltyFunction(\penaltyParameter,\weights).
\end{equation}
This formulation provides substantially higher expressive power than P-OpInf, but has the drawback of yielding a generally non-convex optimization problem that can be difficult to solve. This vanilla NN-OpInf formulation aligns with the approach proposed in~\cite{GaWaZh20}, which employed deep neural networks for inference and 
highlighted that accuracy can stagnate for larger basis dimensions (limiting ``deep convergence''). This work did not compare the neural network approach to well-studied P-OpInf approaches. Thus, the comparative performance of NN-OpInf to P-OpInf in realistic ROM scenarios remains an unaddressed question.


\section{NN-OpInf: Neural network operator inference}\label{sec:nnopinf}
Inspired by previous OpInf work, we propose NN-OpInf as a structure-preserving, composable alternative to P-OpInf for generic nonlinear problems.  Under this framework, the reduced right-hand side $\reducedVelocity$ is represented as a sum of operator blocks, each with a targeted structure and input dependence. While a myriad of works have investigated the use of neural networks for dynamics regression in contexts similar to OpInf~\cite{GaWaZh20,SaMaAh19}, such as the above ``vanilla'' NN-OpInf formulation, these methods typically model the right-hand side dynamics operator with a single neural network $\reducedVelocityNN$. This results in black-box models which are difficult to train, interpret, and analyze, which yields inaccurate and potentially unstable models in practice.  

These challenges are addressed by our NN-OpInf framework for non-intrusive model reduction in the following ways. First, the universal approximation property of neural networks ensures that NN-OpInf is capable of approximating models with non-polynomial non-linearities. Second, NN-OpInf is modular and can produce operators with key structural properties including symmetry and definiteness, leading to the preservation of important physical invariants (e.g., global conservation laws) in the ROM.  While assigning such structure requires some underlying knowledge of the governing equations, we will show that this embedded structure and flexibility results in ROMs that are easier to train and are more accurate and robust than past NN-based efforts. 

The NN-OpInf framework has several distinguishing factors. Similar to P-OpInf, we make an ansatz that the reduced velocity $\reducedVelocity$ is described by a summation of $\nOperators$ independent terms, 
\begin{equation}\label{eq:composable_nn}
\begin{split}
\reducedLhs(t,\params) &= \sum_{r=1}^{\nOperators} \reducedVelocityNN_r (\networkInputs_r ; \weights_r )  , \\
\end{split}
\end{equation}
where $\reducedVelocityNN_r (\cdot,\cdot) \in \reducedStateCoDomain$, $r=1,\ldots,\nOperators$ are operators with trainable parameters $\weights_r$ and inputs $\networkInputs_r \equiv \networkInputs_r(\reducedState,\params)$ comprising a pre-specified subset of the components of the reduced state $\reducedState$ and the parameters $\params$ depending on the problem to be modeled. However, we employ operators $\reducedVelocityNN_r$ that (i) are potentially parameterized to enforce a specific algebraic structure, e.g., skew-symmetry, (ii) can be linear or nonlinear in both the state and parameters, and (iii) can have different numbers of inputs and different levels of complexity (e.g., number of learnable parameters).  
Additional details of these parameterizations will be discussed in subsequent sections. At training time, the parameters $\{\weights_1,\ldots,\weights_M\}$ of~\eqref{eq:composable_nn} are inferred by solving a regularized minimization problem of the form
\begin{equation}\label{eq:nn_opinf_optimization}
\underset{ \weights_1, \ldots, \weights_\nOperators}{\mathrm{minimize}} \;\frac{\displaystyle\sum_{i=1}^{\nTrain} \sum_{j=1}^{\nt} \bigg\|  \sum_{r=1}^{\nOperators} \reducedVelocityNN_r \Big(\networkInputs_r\big( \reducedSnapshotsState^{\params_i}_j , \params_i \big); \weights_r \Big)  - \reducedSnapshotsLhs^{\params_i}_j \bigg\|_2^2}{\displaystyle\sum_{i=1}^{\nTrain} \sum_{j=1}^{\nt} \Big\| \reducedSnapshotsLhs^{\params_i}_j \Big\|_2^2} + \penaltyFunction \left(\penaltyParameter,\weights_1,\ldots,\weights_M \right).
\end{equation}
The objective in~\eqref{eq:nn_opinf_optimization} is a relative mean-square error: its denominator is a fixed constant determined by the data, so minimizing it is equivalent to minimizing the corresponding absolute objective up to a rescaling of the penalty parameter $\penaltyParameter$.
The modular structure~\eqref{eq:composable_nn} for the NN-OpInf ROM is motivated by the observation that many discretizations of physical systems result in right-hand side operators that are comprised of a summation of different terms, e.g., a diffusion term, an advection term, and a body forcing term. Furthermore, while the entire right-hand side dynamics operator often does not have any noteworthy algebraic structure, its constituent parts do.  For example, a diffusion operator is symmetric positive definite in the presence of appropriate boundary conditions. 

Of course, a clear drawback of the NN-OpInf formulation~\eqref{eq:composable_nn} is that the optimization problem~\eqref{eq:nn_opinf_optimization} is considerably more difficult to solve than the linear least-squares problem~\eqref{eq:opinf_optimization} associated with P-OpInf, since it can be nonlinear, is generally non-convex for nonlinear operators, and will almost certainly have no closed form solutions. 
Fortunately, advances in optimization algorithms, automatic differentiation, and software have made problems such as~\eqref{eq:nn_opinf_optimization} relatively tractable.

\subsection{Operators}
A key enabling concept in NN-OpInf is the use of operators $\reducedVelocityNN_r$ that are parameterized to enforce specific algebraic properties. As a concrete example, consider a neural-network-based operator that is parameterized to be symmetric positive semi-definite (SPSD):
 \begin{equation}\label{eq:spd_nn}
\reducedState \mapsto \reducedVelocityNN( \networkInputs ; \weights )\reducedState, \qquad  \reducedVelocityNN( \networkInputs ; \weights ) = \reducedSpdMatrixLowerNN(\networkInputs;\weights) \reducedSpdMatrixLowerNN^\intercal(\networkInputs; \weights), 
\end{equation}
where $\reducedSpdMatrixLowerNN(\cdot)$ is a lower triangular matrix (diagonal included) that is the output from a feed forward neural network with learnable weights $\weights$.  By construction, the matrix $\reducedSpdMatrixNN =   \reducedSpdMatrixLowerNN \reducedSpdMatrixLowerNN^\intercal$ is (1) symmetric positive semi-definite and (2) is potentially nonlinear in both the state and parameters.

Parameterizations like~\eqref{eq:spd_nn} can be used to construct a wide variety of linear and nonlinear operators. These operators can range in complexity from a linear operator that is affine in its input parameters to a deep neural network parameterization of a skew-symmetric matrix that is nonlinear in both the state and input parameters. Table~\ref{tab:operators} summarizes a subset of operators that we consider, including SPSD operators, skew-symmetric operators, and potential function operators.
\begin{table*}[t]
    \centering
    \small
    \caption{Subset of operators used in NN-OpInf. }
    \label{tab:operators}
    \begin{tabularx}{\textwidth}{@{}l l X@{}}
        \toprule
        \textbf{Name} & \textbf{Expression} & \textbf{Description} \\
        \midrule
        Standard  &
        $\mathcal{N}(\networkInputs;\weights)$ &
        Fully connected neural network mapping reduced inputs to the reduced velocity. This is the most flexible, least constrained operator and serves as the baseline NN-OpInf formulation in~\cite{GaWaZh20}. \\
        Matrix &
        $\reducedStiffnessMatrixNN(\networkInputs;\weights)\reducedState$ &
        A generic 
        matrix-valued network function acting on the reduced state. \\
        SPSD &
        $\reducedSpdMatrixLowerNN(\networkInputs;\weights)\reducedSpdMatrixLowerNN(\networkInputs;\weights)^\intercal\reducedState$ &
        A Cholesky-inspired product of lower-triangular matrix operators enforcing
        symmetry and positive (semi-) definiteness. 
        Designed to yield dissipative dynamics consistent with energy decay in homogeneous systems. \\
        Skew &
        $\big(\reducedSkewMatrixLowerNN(\networkInputs;\weights)-\reducedSkewMatrixLowerNN(\networkInputs;\weights)^\intercal\big)\reducedState$ &
        Strictly lower-triangular matrix operator (uniquely) enforcing skew-symmetry by construction, leading to energy-preserving dynamics in the absence of forcing. \\
        Vector &
        $\mathbf{c}(\params)\ \text{or}\ \mathbf{c}(\networkInputs;\weights)$ &
        Operator representing an additive offset or forcing term. Used to represent constant inputs or purely parametric forcing independent of the reduced state. \\
        SPSD Potential &
        $\begin{aligned}[t]&-\nabla_{\reducedState}\mathcal{P}(\reducedState), \\ &\mathcal{P}(\reducedState)=\frac{1}{2}\reducedState^\intercal \reducedStiffnessMatrixNN(\networkInputs;\weights)  \reducedState \\ &\reducedStiffnessMatrixNN(\networkInputs;\weights) =  \reducedSpdMatrixLowerNN(\networkInputs;\weights)\reducedSpdMatrixLowerNN(\networkInputs;\weights)^\intercal\end{aligned}$ &
        Gradient-based model computed with automatic differentiation and derived from a scalar potential built from an SPSD operator. The resulting dynamics can be used to enforce Lagrangian structure and guarantee the conservation of energy.   \\
        \bottomrule
    \end{tabularx}
\end{table*}

\subsection{Pragmatic approaches to increase robustness}
Neural network training problems such as \eqref{eq:nn_opinf_optimization} are  known to be highly dependent on the choice of hyper-parameters, data normalization, optimization strategy, etc. 
We have found these sensitivities can occasionally degrade the performance of the inferred models. 
Here, we outline several practical strategies to improve robustness. 

\subsubsection{Optimization approach}
In P-OpInf, the inference problem can be solved using linear least-squares methods due to the polynomial structure of the inferred operators. This is not the case for the present approach due to the use of nonlinear parameterizations, which render the resulting optimization non-convex.  Note that this can even be the case for linear operators when additional structure is enforced.  For example, inference for the Cholesky parameterization of an SPSD linear operator is not convex.

 
First-order stochastic gradient descent (SGD)-based approaches, such as Adam optimization, offer standard algorithms for training neural network models. SGD scales well with increasing data and has been empirically shown to enhance generalization performance and mitigate the effect of local minima when training NN models. 
However, 
stochastic first-order methods
neglect curvature information in the loss landscape and 
can encounter convergence problems;
past works have found it difficult to obtain ``deep'' convergence with OpInf for neural-network-based approaches for non-intrusive learning~\cite{GaWaZh20}. 
To partially mitigate the slow convergence and lack of scale-invariance associated with SGD, we employ a simple, pragmatic hybrid optimization approach that utilizes the L-BFGS and SGD implementations
readily available in PyTorch. As described in Algorithm~\ref{alg:hybrid}, the approach interleaves $\nBfgsSteps$ L-BFGS steps with $n_f$ steps of Adam.
The motivation behind this alternating strategy 
is to use L-BFGS to quickly reach a local minimum, followed by Adam to ``escape'' this minimum and
further drive down the loss. An additional advantage of this approach is its effectiveness for linear and mildly nonlinear model architectures, which can result in convex optimization problems with a global minimizer. We emphasize that, while we have observed this approach to be more robust than a uniform Adam optimization, L-BFGS typically assumes full-batch gradient evaluations and enforces an SPD Hessian approximation, which can limit its effectiveness for large-scale, non-convex optimization problems. Investigation of more sophisticated optimization algorithms, such as stochastic SR1~\cite{LiLiDe26}, will be a topic of future work. 

\begin{algorithm}
    \caption{Training Algorithm with L-BFGS and Adam}
    \label{alg:hybrid}
    \begin{algorithmic}[1] 
        \STATE \textbf{Input:} Number of epochs \( n_{\text{epochs}} \), L-BFGS frequency \( n_f \), number of BFGS steps \(\nBfgsSteps \)
        \FOR{epoch = 0 to \( n_{\text{epochs}}\)}
            \IF{epoch \% $n_f$ = 0}
              \FOR{j = 1 to \( \nBfgsSteps\)}
                \STATE do\_lbfgs
              \ENDFOR
            \ENDIF
            \STATE do\_adam
        \ENDFOR
    \end{algorithmic}
\end{algorithm}

\subsubsection{Weight regularization}
To encourage convergence, we also employ a simple $\ell_2$ weight regularization (equivalently, a Tikhonov regularization) on the network parameters, which can be implemented directly as weight decay in the optimizer. Weight regularization is a natural extension of the typical Tikhonov regularization used in OpInf, although we note the subtle difference in that, in standard OpInf, the optimization variables are the operator coefficients themselves,
whereas in NN-OpInf the optimization variables are the parameters. In our experiments, this regularization has a relatively small impact on performance and robustness, especially compared to the much stronger sensitivity to regularization observed in standard P-OpInf.

\subsubsection{Data normalization}
For data-normalization we employ a max-abs normalization strategy for the state and its (approximate) time derivative. For the max-abs strategy, the inputs and response are normalized by their maximum absolute value across the training set. To describe this mathematically, let $\reducedSnapshotsState_{\mathrm{train}} = \romBasis^\intercal \snapshotsState_{\mathrm{train}}  \in \RR{\romDim \times N_s^{\mathrm{train}} }$, where $N_s^{\mathrm{train}}$ is the number of snapshots in the training split, denote the 
projection of the training snapshots onto the span of the reduced basis, where $\snapshotsState_{\mathrm{train}}$ denotes the training data matrix obtained after a training/ validation split on the snapshot data $\reducedSnapshotsState$. These projected data are scaled by
$$\reducedSnapshotsState_{\mathrm{train}}^* = \frac{\reducedSnapshotsState_{\mathrm{train}} }{\max{\Big( \mathrm{abs} \big(  \reducedSnapshotsState_{\mathrm{train}}  \big) \Big) } }, $$
and the same scaling is applied for the target response (e.g., reduced velocity). 

As opposed to standard normalization, i.e., subtracting the mean of the feature/response and scaling by its standard deviation, we have found the max-abs strategy to be preferable for several reasons.
First, a max-abs normalization strategy ensures that the algebraic structure imposed on the normalized operators is the same as the structure on the non-normalized operators.  Moreover, it does not suffer from common pitfalls of component-wise normalization.
Focusing on the normalization of the response, we note that component-wise normalization strategies disrupt the hierarchical structure of the POD coefficients and break the equivalence between norms in the reduced and full spaces. Specifically, for two vectors $\mathbf{x}, \mathbf{y} \in \range{\romBasis}$ and an orthonormal basis, it is straightforward to show 
the distance between the reduced  and full vectors is the same, i.e.,  
$$\|\mathbf{x} - \mathbf{y} \|_2 = \big\|\romBasis^\intercal \mathbf{x} - \romBasis^\intercal \mathbf{y} \big\|_2.$$
Employing a component-based normalization strategy disrupts this equivalence and limits a physical interpretation of the loss function. 

 
\subsubsection{Ensembling}
Lastly, we have found ensembling during NN training to be a simple and effective tool to further improve robustness and performance. In an ensemble formulation, multiple models are trained independently, and their predictions are averaged at inference time. To describe this, let $\{\reducedVelocityNN(\networkInputs;\weights^{(m)})\}_{m=1}^{L}$ denote $L$ independently trained NN-OpInf operators (e.g., different random initializations, minibatch orderings, training/validation splits). We define the ensemble prediction of the operator as the uniform average
\begin{equation}\label{eq:composable_nn_ensemble}
\begin{split}
\reducedVelocityNN(\networkInputs ; \weights) &= \frac{1}{L}\sum_{m=1}^L\reducedVelocityNN^{(m)} \big(\networkInputs ; \weights^{(m)} \big) ,\\
\end{split}
\end{equation}
where each $\reducedVelocityNN^{(m)}$ is a composite operator in the form of equation~\eqref{eq:composable_nn}. 
At inference time, we evaluate the reduced-order dynamics using the averaged right-hand side above. Note that ensemble averaging reduces variance and mitigates sensitivity to a particular training realization. Moreover, it is a simple linear transformation of the outputs, and is therefore compatible with any of the algebraic structures enforced by our network parameterizations.
In practice, we have found that even small ensembles ($L=2$) can substantially improve performance. We note that more sophisticated ensembling approaches are possible, including ones that yield uncertainty estimates~\cite{lakshminarayanan2017simplescalablepredictiveuncertainty}.

\subsection{Software implementation in NN-OpInf}
We developed an open-source Python software package, \code{NN-OpInf}~\cite{nnopinf}, to support training and deployment of the described models. The package is built on PyTorch and organized around three core abstractions---\emph{variables}, \emph{operators}, and \emph{models}---that provide a thin, composable layer over standard PyTorch modules.

The \emph{variable} class encapsulates the data and normalization strategy for each input and output quantity, such as the reduced state $\reducedState$, parameters $\params$, or reduced velocity $\reducedVelocity$. Variables store training snapshots and metadata (size and name), and expose a normalization strategy. During training, we perform a standard randomized train/validation split, apply the requested normalizer (e.g., max-abs scaling), and return normalized tensors together with the scaling factors for inverse transforms. This is the mechanism used to normalize both inputs and outputs prior to solving the regression problems described in Section~\ref{sec:nnopinf}.
The \emph{operator} layer implements the structured building blocks introduced in Table~\ref{tab:operators}. Each operator is a PyTorch \code{nn.Module} with a \code{forward} method that consumes a dictionary of named inputs (e.g., \code{inputs['x']}, \code{inputs['mu']}) and returns the operator evaluation and potentially its Jacobian. 
The \emph{model} class is a lightweight container that aggregates a list of operators and evaluates their sum to produce the reduced right-hand side. The model also propagates input/output scalings to each operator and provides a utility for saving learned operators to disk. This design allows heterogeneous operator blocks (different inputs, architectures, and structures) to coexist in a single ROM.

Training is handled by a small set of utilities in \code{nnopinf.training}. The default settings include Adam optimization with optional L-BFGS acceleration, weight decay for $\ell_2$ regularization, learning-rate scheduling, and periodic checkpointing. The training loop constructs mini-batches from normalized data, evaluates the model forward pass, and minimizes a relative mean-square error objective consistent with~\eqref{eq:nn_opinf_optimization}. Gradients are computed via PyTorch's automatic differentiation, so backpropagation through the operator graph is handled automatically. This makes it straightforward to prototype, insert, and evaluate new operator parameterizations by defining a new \code{nn.Module} with a \code{forward} method. We additionally note that, while not considered in the present work, it is straightforward to modify the training framework to utilize backpropagation through time for dynamics-constrained training. Below is a short example using \code{nnopinf} to construct and train a model with an SPSD term and an additive forcing term:

\begin{lstlisting}[style=nnopinf]
import nnopinf
import nnopinf.operators as operators
import nnopinf.models as models
import nnopinf.training as training

x_var = nnopinf.Variable(size=rom_dim, name="x", normalization_strategy="MaxAbs")
mu_var = nnopinf.Variable(size=n_params, name="mu", normalization_strategy="MaxAbs")
y_var = nnopinf.Variable(size=rom_dim, name="y", normalization_strategy="MaxAbs")
x_var.set_data(x_snapshots)
mu_var.set_data(mu_snapshots)
y_var.set_data(xdot_snapshots)

spd_op = operators.SpdOperator(
    acts_on=x_var,
    depends_on=(x_var, mu_var),
    n_hidden_layers=2,
    n_neurons_per_layer=64,
    positive=True,
)
forcing_op = operators.VectorOffsetOperator(n_outputs=rom_dim)

model = models.Model([spd_op, forcing_op])
settings = training.get_default_settings()
training.train(model, variables=[x_var, mu_var], y=y_var, training_settings=settings)
\end{lstlisting}

\section{Analysis}\label{sec:analysis}
This section summarizes computational and optimization considerations for the proposed NN-OpInf framework. We first compare the dominant FLOP counts for structured operators against standard linear and quadratic OpInf models, and then consider training costs. A summary of these costs is provided in Table~\ref{tab:analysis_cost_summary}.  We conclude with a brief discussion of the convexity properties of NN-OpInf optimization problems. We consider the non-parametric case in what follows.

\subsection{Computational cost (online evaluation)}
We compare the dominant FLOP counts for evaluating different reduced operators, focusing on the dependence on the reduced dimension $\romDim$. 

\paragraph{Linear OpInf}
Evaluating a linear OpInf model requires a matrix--vector product $\mathbf{A}\reducedState$, with cost (neglecting linear-in-$K$ terms)
$$C_{\mathrm{linear}} \approx  2\romDim^2 .$$

\paragraph{Quadratic OpInf}
For the quadratic model~\eqref{eq:opinf_quadratic}, the quadratic term requires forming $\sqr{\reducedState}$ and multiplying by $\reducedOperatorQuadratic$. The resulting cost is
$$C_{\mathrm{quadratic}} \approx  \romDim^3 + \frac{3}{2} \romDim^2 ,$$
which scales cubically in $\romDim$. Note that we count only the quadratic term, as its $\mathcal{O}(\romDim^3)$ cost dominates the $\mathcal{O}(\romDim^2)$ cost of the linear term $\reducedOperatorLinear\reducedState$ in~\eqref{eq:opinf_quadratic}.

\paragraph{Linear SPSD operator}
For an SPSD operator of the form $\skew{-4}\hat{\mathbf{L}}\skew{-4} \hat{\mathbf{L}}^\intercal \reducedState$, where $\skew{-4}\hat{\mathbf{L}}$ is lower triangular, the evaluation cost is two triangular matrix--vector products:
$$C_{\mathrm{SPSD}} \approx \romDim^2,$$
since each triangular product costs $\romDim^2/2 + \mathcal{O}(\romDim)$ FLOPs. 
Note that this cost assumes the iterative application $\skew{-4} \hat{\mathbf{L}}(\skew{-4} \hat{\mathbf{L}}^\intercal \hat{\mathbf{x}})$, as opposed to explicitly forming 
$\hat{\mathbf{A}}=\skew{-4}\hat{\mathbf{L}}\skew{-4}\hat{\mathbf{L}}^\intercal$ (which would be $\mathcal{O}(K^3)$).

\paragraph{Linear Skew operator}
For a skew-symmetric operator $\left(\hat{\mathbf{S}}-\hat{\mathbf{S}}^\intercal\right)\reducedState$ with $\hat{\mathbf{S}}$ strictly lower triangular, the evaluation cost is also two triangular products:
$$C_{\mathrm{skew}} \approx \romDim^2.$$
As with the SPSD operator, this is quadratic in $\romDim$.

\paragraph{NN-OpInf added cost}
In NN-OpInf, the total cost includes the neural network forward pass used to produce, e.g., $\skew{-4}\hat{\mathbf{L}}$ or $\hat{\mathbf{S}}$. For a fully connected network with $\nHidden$ hidden layers, $\nNeurons$ neurons per layer, and activation cost $\nOpsActivation$, the forward-pass FLOPs for input dimension $n_{\mathrm{in}}$ and output dimension $n_{\mathrm{out}}$ are
$$
C_{\mathrm{NN}} = 2 \nNeurons n_{\mathrm{in}} + \nOpsActivation \nNeurons + \nHidden \left(2 \nNeurons^2 + \nOpsActivation \nNeurons \right) + 2 \nNeurons n_{\mathrm{out}},
$$
assuming no activation on the output layer. The output dimension depends on the operator:
$$
n_{\mathrm{out}} \approx
\begin{cases}
\romDim & \text{standard (vector) operator},\\
\romDim^2 & \text{matrix operator},\\
\romDim(\romDim+1)/2 & \text{SPSD operator (lower-triangular output)},\\
\romDim(\romDim-1)/2 & \text{skew operator (strictly lower-triangular output)}.
\end{cases}
$$
The operator application then adds the quadratic-in-$\romDim$ terms above.
For the choice $\nNeurons=\romDim$, $\nHidden=3$, and $n_{\mathrm{in}}=\romDim$ used in most numerical examples later in this work, this simplifies to
$$
C_{\mathrm{NN}} = 8\romDim^2 + 4\nOpsActivation \romDim + 2\romDim n_{\mathrm{out}}.
$$
This cost is cubic in $K$ for the matrix, SPSD, and skew operators. The total cost of the NN OpInf operator then adds cost of the application of the operator to the state, as described above. We write this as
$$
C_{\mathrm{NN-OpInf}} = C_{\mathrm{NN}} + C_{\mathrm{apply}}, 
$$
where $C_{\mathrm{apply}}$ is, e.g., $K^2$ for an SPSD operator.
\paragraph{SPSD potential operator}
Lastly, for an SPSD potential operator defined by $\reducedVelocity=-\nabla_{\reducedState}\mathcal{P}(\reducedState)$ with $\mathcal{P}(\reducedState)=\frac{1}{2}\reducedState^\intercal\hat{\mathbf{A}}(\reducedState)\reducedState$ and $\hat{\mathbf{A}}=\skew{-3}\hat{\mathbf{L}}\skew{-3}\hat{\mathbf{L}}^\intercal$, the forward pass first evaluates $\skew{-3}\hat{\mathbf{L}}$ (via an NN) and applies two triangular products to form $\hat{\mathbf{A}}\reducedState$. The scalar potential then adds a single inner product. If $\hat{\mathbf{A}}$ is treated as fixed with respect to $\reducedState$, the gradient reduces to $\hat{\mathbf{A}}\reducedState$ and the cost matches the SPSD operator, i.e., $C_{\mathrm{SPSD\text{-}Potential}} \approx \romDim^2$ (plus the NN cost to produce $\skew{-3}\hat{\mathbf{L}}$). When $\hat{\mathbf{A}}$ depends on $\reducedState$, the full gradient includes additional terms from $\partial \hat{\mathbf{A}}/\partial \reducedState$. In practice this is computed with reverse-mode autodifferentiation. For standard dense layers in PyTorch, the backward pass typically costs about twice the forward pass. Let $C_{\mathrm{NN}}^{\mathrm{SPSD}}$ denote the NN cost for the SPSD operator (i.e., using $n_{\mathrm{out}}=\romDim(\romDim+1)/2$). A representative complexity is
$$
C_{\mathrm{SPSD\text{-}Potential}} \approx C_{\mathrm{NN}}^{\mathrm{SPSD}} + 2C_{\mathrm{NN}}^{\mathrm{SPSD}} + \mathcal{O}(\romDim^2)
$$
so $C_{\mathrm{SPSD\text{-}Potential}} \approx 3C_{\mathrm{NN}}^{\mathrm{SPSD}} + \mathcal{O}(\romDim^2)$.

Figure~\ref{fig:operator_cost_ratio} summarizes the resulting operator cost ratios relative to linear and quadratic OpInf as a function of $K$ for an assumed network structure ($\nHidden=3,\nNeurons=\romDim,$ and $\nOpsActivation=1$). We find that the NN-OpInf models are substantially more expensive than linear P-OpInf. When compared to quadratic P-OpInf, we find that the SPSD, Skew, and SPSD-Potential NN-OpInf operators all display similar scaling to quadratic OpInf. A standard dense neural network can be more efficient, as it avoids the $K^3$ scaling. We emphasize that, while NN-OpInf ROMs are more expensive than P-OpInf, they provide a more expressive modeling approach that can capture generic non-linearities. 

\begin{figure}
\begin{center}
\includegraphics[trim={0cm 0cm 0cm 0cm},clip,width=0.85\linewidth]{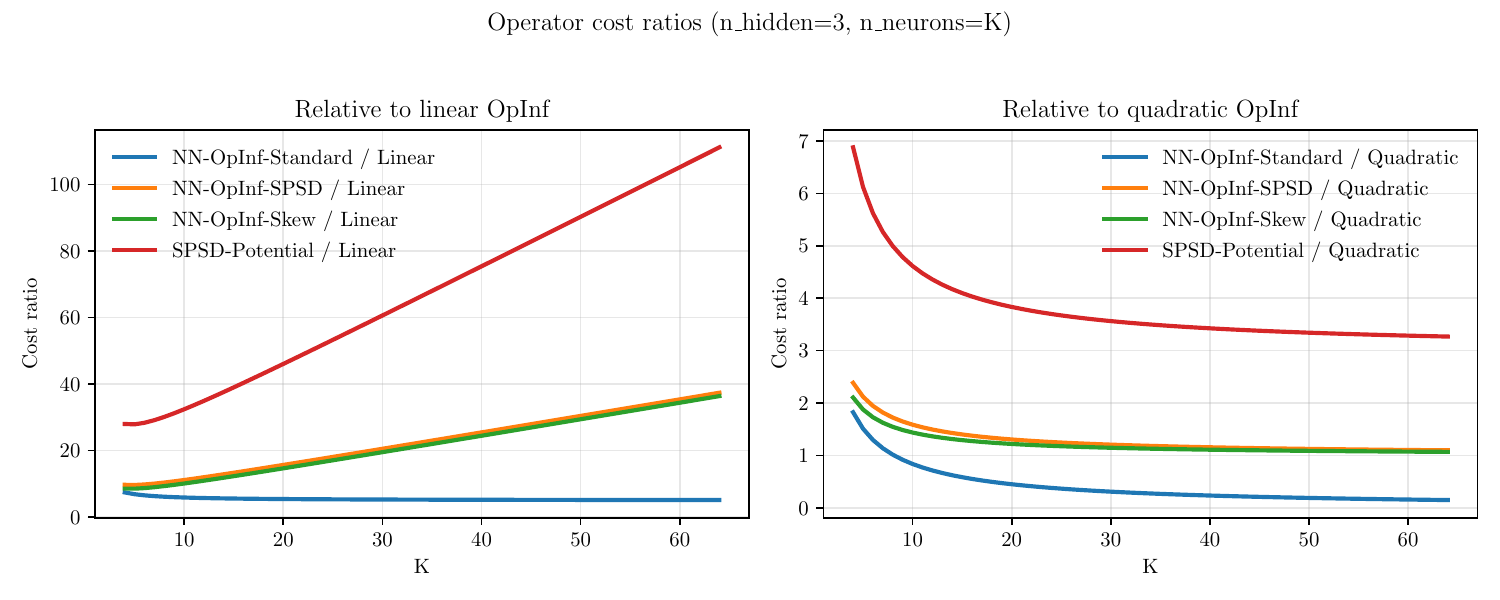}
\caption{Operator evaluation cost ratios relative to linear and quadratic OpInf baselines as a function of the reduced dimension $\romDim$, assuming $\nHidden=3$, $\nNeurons=\romDim$, and $\nOpsActivation=1$.}
\label{fig:operator_cost_ratio}
\end{center}
\end{figure}

\subsection{Computational cost (training)}
We now contrast the \emph{training} cost of NN-OpInf with P-OpInf. We do not include the costs of projecting data onto the reduced subspaces, as these costs are equivalent in both approaches. For NN-OpInf, we consider a model comprising only a single right-hand side operator, and do not account for extra cost of, e.g., hyper-parameter tuning in either method, which typically adds a constant multiplier to the total cost. Let $N_t N_{\text{train}}$ denote the total number of snapshot pairs used for training and $n_{\text{epochs}}$ the number of epochs.

\paragraph{P-OpInf (linear/quadratic)}
P-OpInf training reduces to a series of linear least-squares problems with $p$ unknown coefficients per output component.  For example, $p=\romDim$ for a linear operator while $p=\romDim(\romDim+1)/2$ for a quadratic operator. This inference problem is decoupled across output components, so the total cost is $K$ independent least-squares solves with the same design matrix but different right-hand sides. A representative scaling on the cost for a QR- or normal-equation-based least-squares solve is\footnote{We note that Ref.~\cite{PeWi16} approximates the cost as $\mathcal{O}(K N_s p^3)$, where $N_s=N_tN_{\text{train}}$ is the number of snapshots, which is conservative.}
$$
C_{\text{train}}^{\text{OpInf}} = \mathcal{O}\!\left(K\left(N_t N_{\text{train}}\, p^2 + p^3\right)\right),
$$
where the first term corresponds to forming normal equations (or an equivalent QR factorization) and the second term to solving the resulting system. Note that this training process is one-shot: the cost is incurred only once. While the training cost is cheap for small $K$, we do note poor scaling with the ROM dimension, particularly for quadratic operators, in which case training goes as $\mathcal{O}(K^7)$.\footnote{This cost can be reduced to be $\mathcal{O}(K^6)$ through a shared factorization of the data matrix, assuming that $\penaltyParameter$ is the same for each output component.}

\paragraph{NN-OpInf}
NN-OpInf is trained with iterative gradient-based methods, and the cost is dominated by one forward pass and one backward pass. Using the forward-pass cost $C_{\mathrm{NN\text{-}OpInf}}$ (defined earlier as the sum of $C_{\mathrm{NN}}$ and $C_{\text{apply}}$) and the standard estimate that the backward pass is about twice the forward pass, a representative per-sample cost is $3C_{\mathrm{NN\text{-}OpInf}}$. Thus a representative total training cost is
$$
C_{\text{train}}^{\text{NN-OpInf}} \approx 3 N_t N_{\text{train}}\, n_{\text{epochs}} C_{\mathrm{NN\text{-}OpInf}}.
$$
Figure~\ref{fig:training_cost_ratio} summarizes the resulting training-cost ratios relative to linear and quadratic P-OpInf baselines as a function of $\romDim$, assuming $n_{\text{epochs}}=10{,}000$, $\nHidden=3$, $\nNeurons=\romDim$, and $N_t N_{\text{train}}=10{,}000$ snapshot pairs. We observe that NN-OpInf incurs orders of magnitude higher cost than P-OpInf for small basis dimensions. This increase in cost highlights one of the main tradeoffs of NN-OpInf: a potential increase in accuracy at the expense of an oftentimes sharp increase in offline training cost. We note that this increase in cost becomes more balanced with respect to quadratic P-OpInf as $K$ grows, due to the poor scaling noted above.

\begin{figure}
\begin{center}
\includegraphics[trim={0cm 0cm 0cm 0cm},clip,width=0.85\linewidth]{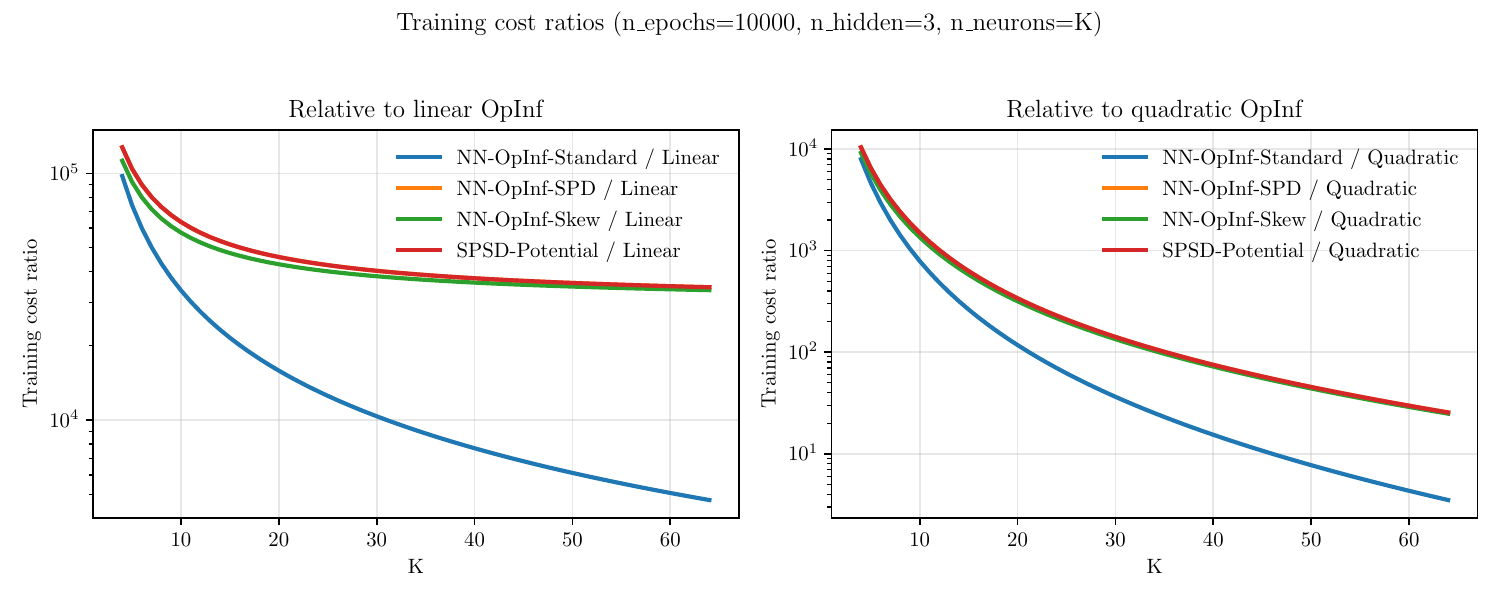}
\caption{Training cost ratios relative to linear and quadratic P-OpInf baselines as a function of reduced dimension $\romDim$, assuming $n_{\text{epochs}}=10{,}000$, $\nHidden=3$, $\nNeurons=\romDim$, and $N_t N_{\text{train}}=10{,}000$ snapshot pairs.}
\label{fig:training_cost_ratio}
\end{center}
\end{figure}


\begin{table}[t]
    \centering
    \caption{Summary of dominant computational costs in the analysis section. Here $N_s=N_tN_{\text{train}}$. For NN-operator costs, we assume $\nNeurons \propto K$, and $n_{\mathrm{in}}=K$.}
    \label{tab:analysis_cost_summary}
    \begin{tabularx}{\textwidth}{@{}l l X@{}}
        \toprule
        \textbf{Category} & \textbf{Model/operator} & \textbf{Dominant cost} \\
        \midrule
        Online evaluation & Linear OpInf & $\mathcal{O}(K^2)$ \\
        Online evaluation & Quadratic OpInf & $\mathcal{O}(K^3)$ \\
        Online evaluation & SPSD & $\mathcal{O}(K^2)$ \\
        Online evaluation & Skew & $\mathcal{O}(K^2)$ \\
        Online evaluation & Standard NN & $\mathcal{O}(K^2)$ \\
        Online evaluation & Matrix NN operator & $\mathcal{O}(K^3)$ \\
        Online evaluation & SPSD NN operator & $\mathcal{O}(K^3)$ \\
        Online evaluation & Skew NN operator & $\mathcal{O}(K^3)$ \\
        Online evaluation & SPSD potential & $\mathcal{O}(K^3)$ \\
        \midrule
        Training & P-OpInf (linear) & $\mathcal{O}\!\left(N_sK^3 + K^4\right)$ \\
        Training & P-OpInf (quadratic) & $\mathcal{O}\!\left(N_sK^5 + K^7\right)$ \\
        Training & NN-OpInf (Standard NN) & $\mathcal{O}\!\left(N_s\,n_{\text{epochs}}\,K^2\right)$ \\
        Training & NN-OpInf (Matrix NN operator) & $\mathcal{O}\!\left(N_s\,n_{\text{epochs}}K^3\right)$ \\
        Training & NN-OpInf (SPSD NN operator) & $\mathcal{O}\!\left(N_s\,n_{\text{epochs}}K^3\right)$ \\
        Training & NN-OpInf (Skew NN operator) & $\mathcal{O}\!\left(N_s\,n_{\text{epochs}}K^3\right)$ \\
        Training & NN-OpInf (SPSD potential) & $\mathcal{O}\!\left(N_s\,n_{\text{epochs}}K^3\right)$ \\
        \bottomrule
    \end{tabularx}
\end{table}

\subsection{Structure of the NN-OpInf problem}
NN-OpInf generally results in nonconvex optimization problems due to both the nonlinearities introduced by neural networks and the choice of operator parameterization. To provide a simple characterization of how these parameterizations can affect the structure and convexity of the resulting optimization problem, we briefly consider the case of state-linear operators with no parametric dependence.

Consider reduced snapshot pairs $\{(\mathbf{x}_\ell,\mathbf{y}_\ell)\}_{\ell=1}^{N_s}$, with $\mathbf{x}_\ell,\mathbf{y}_\ell\in\RR{\romDim}$, and define the objective
\begin{equation}
J(\mathbf{A})=
\sum_{\ell=1}^{N_s}\| \mathbf{A}\mathbf{x}_\ell - \mathbf{y}_\ell \|_2^2,
\qquad
\mathbf{A}\in\RR{\romDim\times\romDim}.
\end{equation}
Let $\mathbf{X}_{\rm data}=[\mathbf{x}_1,\ldots,\mathbf{x}_{N_s}]\in\RR{\romDim\times N_s}$ and $\mathbf{y}_{\rm data}=[\mathbf{y}_1^\intercal,\ldots,\mathbf{y}_{N_s}^\intercal]^\intercal\in\RR{N_s\romDim}$. Using the ``vec trick'' 
\[
\operatorname{vec}(\mathbf{A}\mathbf{X}_{\rm data})=
(\mathbf{X}_{\rm data}^{\intercal}\otimes\mathbf{I})\operatorname{vec}(\mathbf{A}),
\]
we can write the OpInf optimization problem (neglecting regularization) as
\begin{equation}
J(\mathbf{A})=\|
(\mathbf{X}_{\rm data}^{\intercal}\otimes\mathbf{I})\operatorname{vec}(\mathbf{A})-\mathbf{y}_{\rm data}
\|_2^2,
\end{equation}
which is convex (quadratic) in the vectorized variable $\operatorname{vec}(\mathbf{A})\in\mathbb{R}^{K^2}$.

\begin{proposition}[Convex linear-operator cases]
For the objective $J(\mathbf{A})$, each of the following optimization problems is convex:
\begin{enumerate}
    \item[(i)] $\min_{\mathbf{A}\in\RR{\romDim\times\romDim}} J(\mathbf{A})$ (unconstrained linear operator),
    \item[(ii)] $\min_{\mathbf{S}\in\mathcal{L}_{\mathrm{sl}}} J(\mathbf{S}-\mathbf{S}^{\intercal})$ (skew-symmetric operator with $\mathcal{L}_{\mathrm{sl}}$ the strictly lower-triangular matrices),
    \item[(iii)] $\min_{\mathbf{A}=\mathbf{A}^{\intercal},\,\mathbf{A}\succeq 0} J(\mathbf{A})$ (SPSD operator),
    \item[(iv)] $\min_{\substack{\mathbf{S}\in\mathcal{L}_{\mathrm{sl}},\\
        \mathbf{K}=\mathbf{K}^{\intercal},\,\mathbf{K}\succeq 0}}
        J(\mathbf{S}-\mathbf{S}^{\intercal}+\mathbf{K})$
        (additive skew-symmetric and SPSD operators).
\end{enumerate}
\end{proposition}
\begin{proof}
Case~(i) has a convex feasible set and convex objective. In case~(ii), the map $\mathbf{S}\mapsto \mathbf{S}-\mathbf{S}^{\intercal}$ is linear and $\mathcal{L}_{\mathrm{sl}}$ is a linear subspace, so $J(\mathbf{S}-\mathbf{S}^{\intercal})$ is convex in $\mathbf{S}$. In case~(iii), the feasible set is the intersection of the symmetric subspace and the PSD cone, hence convex. In case~(iv), the feasible set is the Cartesian product of $\mathcal{L}_{\mathrm{sl}}$ and the SPSD cone, and the map $(\mathbf{S},\mathbf{K})\mapsto\mathbf{S}-\mathbf{S}^{\intercal}+\mathbf{K}$ is linear. Therefore, all four problems are convex~\cite{BoydVandenberghe04}.
\end{proof}

\begin{remark}[When convexity is lost]
Linear and nonlinear parameterization of $\mathbf{A}$ can destroy (strong) convexity. For example, enforcing $\mathbf{A}\succeq 0$ via $\mathbf{A}=\mathbf{L}\mathbf{L}^\intercal$ gives an objective that is generally nonconvex in $\mathbf{L}$. Likewise, enforcing skew symmetry via $\mathbf{A}=\mathbf{S}-\mathbf{S}^\intercal$ introduces non-uniqueness and loss of strong convexity unless $\mathbf{S}$ is restricted (e.g., strictly lower triangular).
\end{remark}


\paragraph{Neural network operators}
In the full NN-OpInf setting, operators depend nonlinearly on trainable weights, so the least-squares objective is nonconvex. 

\subsection{Uniqueness and identifiability}

Building on the previous section, we briefly remark upon both the uniqueness of the proposed additive operator structure and its identifiability.
First, focusing on uniqueness, optimization for the neural network parameters
themselves is generally not a unique optimization problem. For example, nonlinearities, scalings, and symmetries can cause different weights to
represent the same function. This parameter non-uniqueness does not necessarily
affect the learned composite reduced velocity. A separate question is whether
the composite velocity is uniquely determined by the available data, and a third is
whether its individual operator blocks are uniquely determined.

To provide insight into these aspects, we examine the linear setting. Consider
\begin{equation}\label{eq:identifiability_linear}
    \reducedStateDot=(\reducedSpdMatrix+\reducedSkewMatrix)\reducedState,
    \qquad
    \reducedSpdMatrix=\reducedSpdMatrix^{\intercal},
    \qquad
    \reducedSkewMatrix=-\reducedSkewMatrix^{\intercal},
\end{equation}
where one may additionally require $\reducedSpdMatrix\succeq 0$. If the state snapshot
matrix $\reducedSnapshotsState$ is full rank, then the least-squares objective
$\big\|\reducedSnapshotsLhs-(\reducedSpdMatrix+\reducedSkewMatrix)\reducedSnapshotsState\big\|_F^2$ is strictly
convex with respect to the composite matrix $\reducedMatrix=\reducedSpdMatrix+\reducedSkewMatrix$. The composite
matrix is therefore unique, and its components can be uniquely identified as 
\begin{equation}\label{eq:identifiability_sym_skew}
    \reducedSpdMatrix=\tfrac{1}{2}(\reducedMatrix+\reducedMatrix^{\intercal}),
    \qquad
    \reducedSkewMatrix=\tfrac{1}{2}(\reducedMatrix-\reducedMatrix^{\intercal}).
\end{equation}

Such uniqueness does not hold for arbitrary choices of additive
blocks. For example, if $\reducedMatrix=\reducedMlMatrix+\reducedSpdMatrix$ with $\reducedMlMatrix$
unconstrained and $\reducedSpdMatrix$ SPSD, the two operator classes are not mutually distinguishable as, for example, symmetric
contributions can be transferred between the two blocks without
changing $\reducedMatrix$.

There is no general result guaranteeing uniqueness or identifiability in the nonlinear case. However, many of the same considerations from the linear setting carry over to nonlinear, state-dependent matrix operators. In particular, the different operator classes need to be sufficiently distinct for their individual contributions to be uniquely identified. Algebraic structure can help with this. For example, at any fixed state $\reducedState$, a known matrix $\reducedMatrix(\reducedState)$ has a unique symmetric--skew decomposition.

\section{Numerical experiments}\label{sec:results}
We now assess the performance of the proposed NN-OpInf on a suite of numerical experiments, comparing against (1) the NN ROM models proposed in Ref.~\cite{GaWaZh20} (which we refer to as ``NN-OpInf-NN''), (2) linear and quadratic P-OpInf ROMs, (3) several structure preserving OpInf formulations~\cite{KoQi24,ShKa24,ShNaTo24} when appropriate, and, where possible, (4) an intrusive (non-hyper-reduced) POD-Galerkin ROM. 
Table~\ref{tab:model_summary} summarizes the ROMs considered in this section. All presented numerical experiments, along with environment details for reproducibility, are publicly available on the \href{https://github.com/sandialabs/nnopinf}{\code{NN-OpInf}} repository. 

\begin{table}[t]
    \centering
    \caption{Summary of models used in the numerical experiments.}
    \label{tab:model_summary}
    \begin{tabularx}{\textwidth}{@{}l X l@{}}
        \toprule
        \textbf{Name} & \textbf{Description} & \textbf{Implementation} \\
        \midrule
        POD-Galerkin & Intrusive POD ROM (no hyper-reduction) & In-house \\
        P-OpInf-A & Linear P-OpInf model & OpInf~\cite{OpInfPackage} \\
        P-OpInf-AH & Quadratic P-OpInf model & OpInf~\cite{OpInfPackage} \\
        P-OpInf-cA & Linear P-OpInf with affine forcing & OpInf~\cite{OpInfPackage} \\
        P-OpInf-cAH & Quadratic P-OpInf with affine forcing & OpInf~\cite{OpInfPackage} \\
        P-OpInf-EP-H & Energy-Preserving Quadratic P-OpInf & In house (adapted from~\cite{KoQi24,koike_liftandlearn}). \\
        P-OpInf-EP-cAH & Energy-Preserving Quadratic P-OpInf with affine forcing & In house (adapted from~\cite{KoQi24,koike_liftandlearn}). \\
        NN-OpInf-NN & Standard NN operator from Table~\ref{tab:operators} & NN-OpInf~\cite{nnopinf} \\
        NN-OpInf-SS & Skew-symmetric NN operator & NN-OpInf~\cite{nnopinf} \\
        NN-OpInf-PSD-f & Skew operator $-$ SPSD operator $+$ vector forcing operator & NN-OpInf~\cite{nnopinf} \\
        NN-OpInf-SPSD-f & SPSD NN operator $+$ vector forcing operator & NN-OpInf~\cite{nnopinf} \\
        NN-OpInf-SPSD-Potential & SPSD potential operator & NN-OpInf~\cite{nnopinf} \\
        NN-OpInf-Linear-Lagrangian & Constant SPD stiffness operator in second-order Lagrangian form; analogue of L-OpInf~\cite{ShKa24} & NN-OpInf~\cite{nnopinf} \\
        NN-OpInf-SpML-Lagrangian & Constant SPD stiffness operator plus a learned scalar potential; analogue of LOpInf-SpML~\cite{ShNaTo24} & NN-OpInf~\cite{nnopinf} \\
        \bottomrule
    \end{tabularx}
\end{table}

For all NN-OpInf ROMs, the available data are separated randomly 
in a standard 80/20 training/validation split. Each layer is initialized with PyTorch's default initialization (Kaiming Uniform for dense linear layers).
Unless otherwise noted, the networks employ 3 layers, and the number of hidden units per layer is set to be the ROM dimension, $K$. We use an ensemble size of $L=2$. The networks are trained with the Adam optimizer for 10000 epochs and a batch size of $50$. We accelerate training by performing 50 L-BFGS iterations with a line search twice: at the start and after 5000 epochs.
The initial learning rate is $5 \times 10^{-3}$ with a learning-rate decay parameter of $0.9998$, meaning that the learning rate is $\left(5 \times 10^{-3}\right) \times (0.9998)^n$ after the $n$-th epoch.  The networks are tanh-activated, and $\ell^2$ regularization is employed with a penalty parameter of $1 \times 10^{-6}$. While not shown, we note that the sensitivity of NN-OpInf to regularization hyperparameters was found to be mild. This contrasts with P-OpInf ROMs, which are notoriously sensitive to the choice of hyperparameter. 

All P-OpInf models are implemented using the OpInf Python package~\cite{OpInfPackage}, with the exception of the energy preserving approaches, for which we added an in-house implementation based on the approach provided in Refs.~\cite{KoQi24,koike_liftandlearn}. We consider six P-OpInf models in the following experiments: linear P-OpInf (P-OpInf-A), quadratic P-OpInf (P-OpInf-AH), energy-preserving quadratic OpInf (P-OpInf-EP-H), linear P-OpInf with an affine forcing (P-OpInf-cA), quadratic OpInf with an affine forcing (P-OpInf-cAH), and quadratic OpInf with an energy-preserving quadratic operator and affine forcing (P-OpInf-EP-cAH). In the parametric setting, we use the interpolatory OpInf approach of~\cite{PeWi16}, which is implemented in the OpInf Python package. 

All ROMs are assessed using the $\ell^2$ relative state error summed over all time steps,
$$e = \frac{ \sum_{i=1}^{N_t}\| \state^{\mathrm{OpInf}}(t_i;\params) - \state^{\mathrm{FOM}}(t_i;\params) \|_2}{ \sum_{i=1}^{N_t} \| \state^{\mathrm{FOM}}(t_i;\params) \|_2}.$$

\subsubsection{Hyper-parameter selection}
All OpInf and NN-OpInf methods shown here have hyper-parameters. For OpInf, we follow Ref.~\cite{McKhWi23} for selecting standard $\ell^2$ regularization parameters via a grid search: 
for each candidate regularization value, a corresponding ROM is trained and subsequently evaluated to determine the best-performing model with respect to the $\ell^2$ error measured across the training data. Our grid search considers 40 values for $\lambda$ on a logarithmic scale from $1 \times 10^{-8}$ to $1 \times 10^3$. For higher-order P-OpInf ROMs, the linear and quadratic terms share the same regularization parameter.

For NN-OpInf, we do not consider any hyper-parameter searches. Indeed, there are a large number of hyperparameters that can have small impacts on model performance, including weight regularization, activation function, network depth, etc. In experiments, we have found these hyper-parameters to have a small impact as compared to the specified algebraic structure and ROM dimension. Unless otherwise noted, all NN-OpInf networks employ the configuration stated above. This configuration was not chosen as a result of exhaustive experiments, but rather from our past experience with neural networks. 
\subsection{Burgers equation}
We first consider the one-dimensional Burgers' equation on the periodic domain $x \in [0,2\pi]$,
$$\frac{\partial u}{\partial t} + \frac{1}{2} \frac{\partial u^2}{\partial x} = 0,$$
subject to the initial condition $u(x,0) = 0.025(\sin(4x) + 2\cos(6x)) + 2$. This initial condition results in an advecting wave with small disturbances that grow in time. In the absence of shocks, Burgers' equation conserves energy, $\frac{d}{dt}\frac{1}{2}\int_{0}^{2 \pi} u(x,t)^2 \, dx = 0$. The FOM is discretized with 500 spatial grid points and 
an energy-conserving symmetric flux for the convective term. Integration in time is performed with a fourth-order Runge--Kutta (RK4) scheme using $\Delta t=0.01$ and snapshots are collected at every time step.  Note that this problem is quadratic in the state $u$, and therefore a quadratic P-OpInf is expected to perform well.
To construct an \nnopinf\ ROM that satisfies the energy conservation constraint, we parameterize the right-hand side with a skew-symmetric operator,
\begin{equation}\label{eq:nnopinf_skew}
\reducedVelocity(\reducedState,\params) =   \left[ \reducedSkewMatrixLowerNN(\reducedState,\params;\weights) - \reducedSkewMatrixLowerNN^\intercal(\reducedState,\params;\weights) \right] \reducedState.
\end{equation}
We refer to this formulation as NN-OpInf-Skew-Symmetric (NN-OpInf-SS).
Two setups are considered, both of which remain shock-free:
\begin{itemize}
\item Reproductive: We collect snapshots and make predictions over $t \in [0,4]$. 
\item Future-state prediction: We collect snapshots over $t \in [0,1]$ and predict over $t \in [0,4]$. 
\end{itemize}
Figure~\ref{fig:burgers_training_error_converge} summarizes results for the reproductive and future-state cases. Figures~\ref{fig:burgers_training_error_converge_a} and~\ref{fig:burgers_training_error_converge_b} report relative error versus basis dimension, Figures~\ref{fig:burgers_training_error_converge_c} and~\ref{fig:burgers_training_error_converge_d} show the final-time solutions, and Figures~\ref{fig:burgers_training_error_converge_e} and~\ref{fig:burgers_training_error_converge_f} quantify energy-conservation violations.
Across both configurations, NN-OpInf-SS consistently outperforms the baseline NN-OpInf-NN. While we do not observe convergence with increasing basis dimension, the relative error is low (around $0.1\%$) in all cases. In the reproductive setting, NN-OpInf-SS is competitive with P-OpInf, while in the future-state configuration it outperforms the linear, quadratic, and energy-preserving quadratic P-OpInf models. The skew-symmetric parameterization preserves energy at the semi-discrete level, which improves stability and predictive consistency in the future-state setting. We observe that both P-OpInf-A and P-OpInf-AH do a reasonable job at preserving energy throughout the training time window, but yield energy violations for future-state predictions. Interestingly, we observe that P-OpInf-EP-H results in substantial energy violations in the future-state-prediction case. While not shown, we note that these violations are a result of the use of the RK4 time integrator, which does not conserve energy at the fully discrete level. Use of a mid-point integrator resulted in energy conservation down to numerical precision for both NN-OpInf-SS and P-OpInf-EP-H, but solution errors were similar to those reported here. Lastly, for this example, the Galerkin ROM outperforms all OpInf models in both the reproductive and future-state configurations.  
\begin{figure}
\begin{center}
\begin{subfigure}[t]{0.49\textwidth}
\includegraphics[trim={0cm 0cm 0cm 0cm},clip,width=1.0\linewidth]{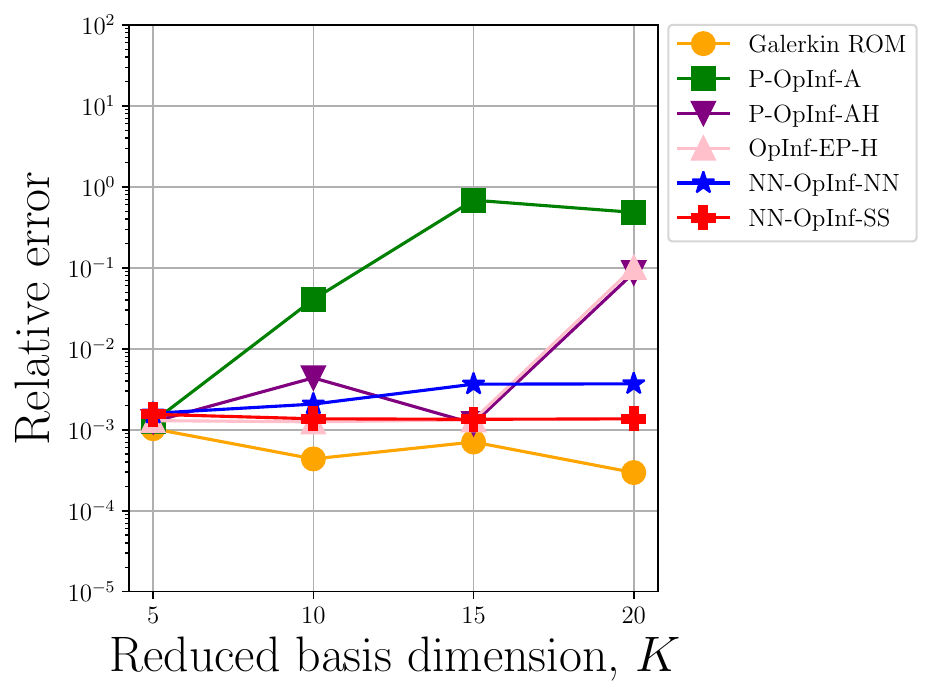}
\caption{Relative error vs. basis dimension (reproductive)}
\label{fig:burgers_training_error_converge_a}
\end{subfigure}
\begin{subfigure}[t]{0.49\textwidth}
\includegraphics[trim={0cm 0cm 0cm 0cm},clip,width=1.0\linewidth]{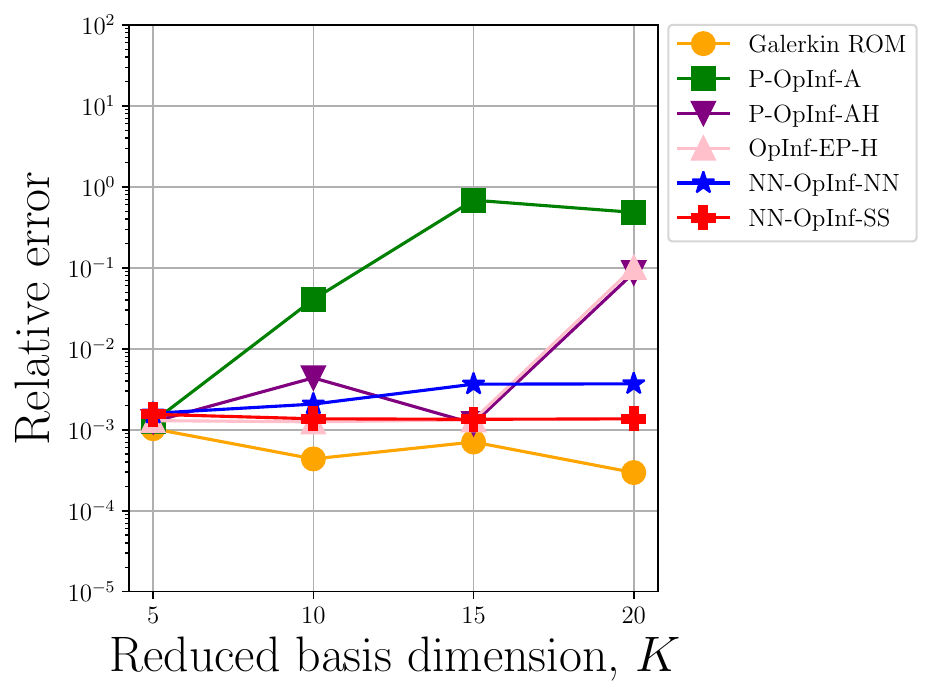}
\caption{Relative error vs. basis dimension (future-state)}
\label{fig:burgers_training_error_converge_b}
\end{subfigure}
\begin{subfigure}[t]{0.49\textwidth}
\includegraphics[trim={0cm 0cm 0cm 0cm},clip,width=1.0\linewidth]{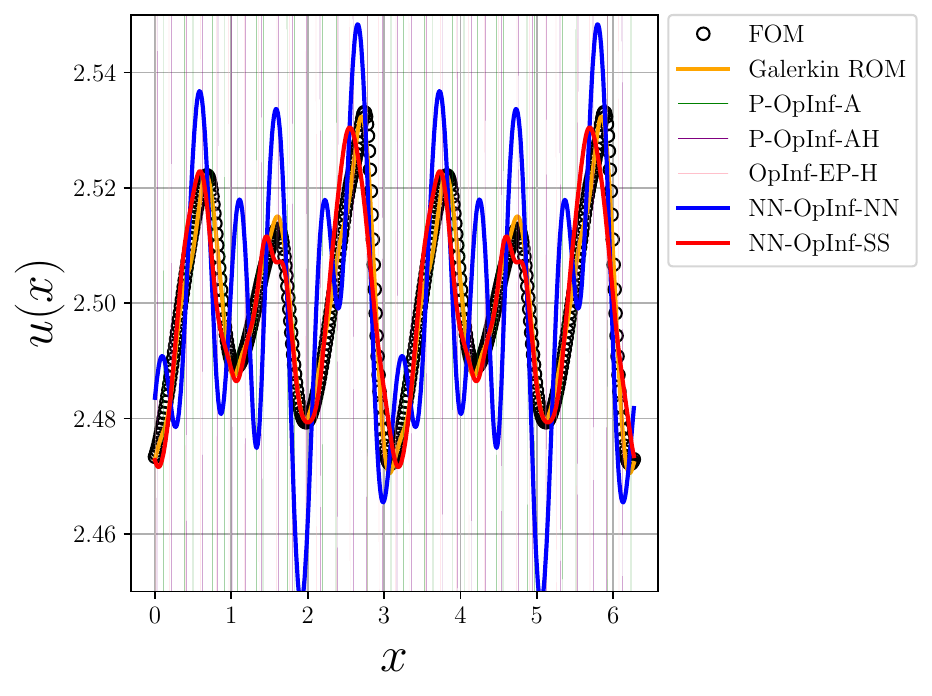}
\caption{Solution at $t=4.0$ (reproductive)}
\label{fig:burgers_training_error_converge_c}
\end{subfigure}
\begin{subfigure}[t]{0.49\textwidth}
\includegraphics[trim={0cm 0cm 0cm 0cm},clip,width=1.0\linewidth]{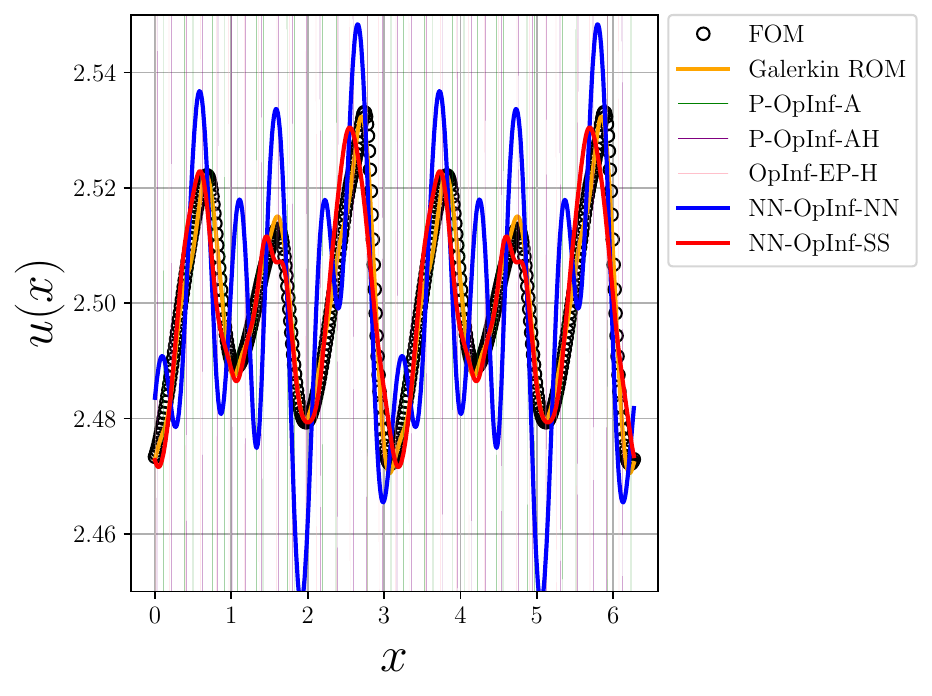}
\caption{Solution at $t=4.0$ (future-state)}
\label{fig:burgers_training_error_converge_d}
\end{subfigure}
\begin{subfigure}[t]{0.49\textwidth}
\includegraphics[trim={0cm 0cm 0cm 0cm},clip,width=1.0\linewidth]{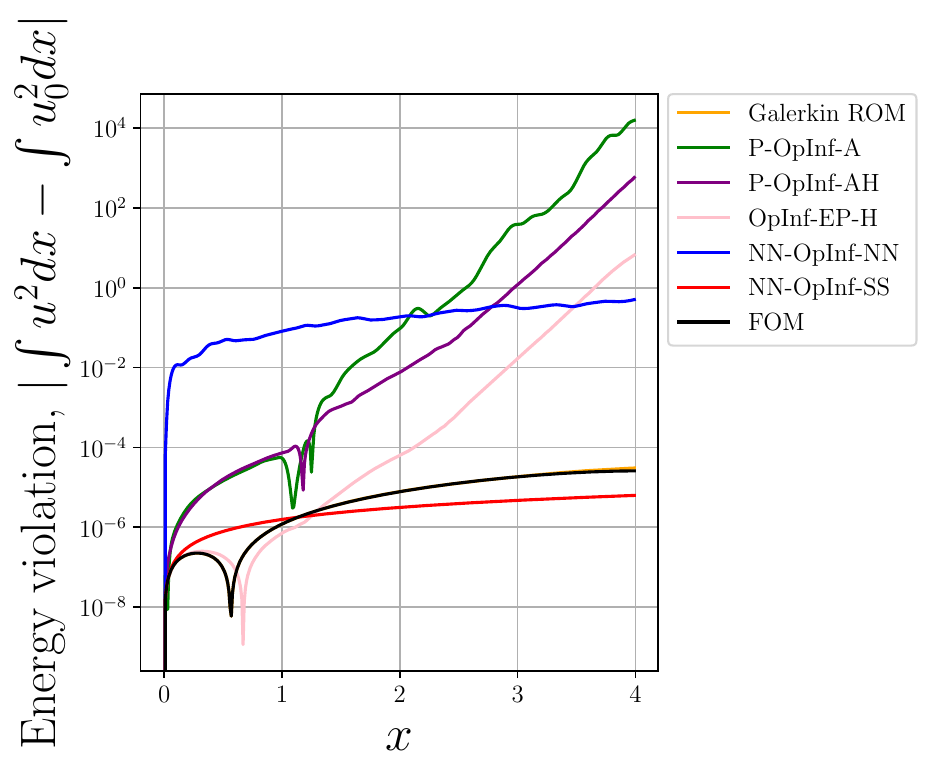}
\caption{Energy violation (reproductive)}
\label{fig:burgers_training_error_converge_e}
\end{subfigure}
\begin{subfigure}[t]{0.49\textwidth}
\includegraphics[trim={0cm 0cm 0cm 0cm},clip,width=1.0\linewidth]{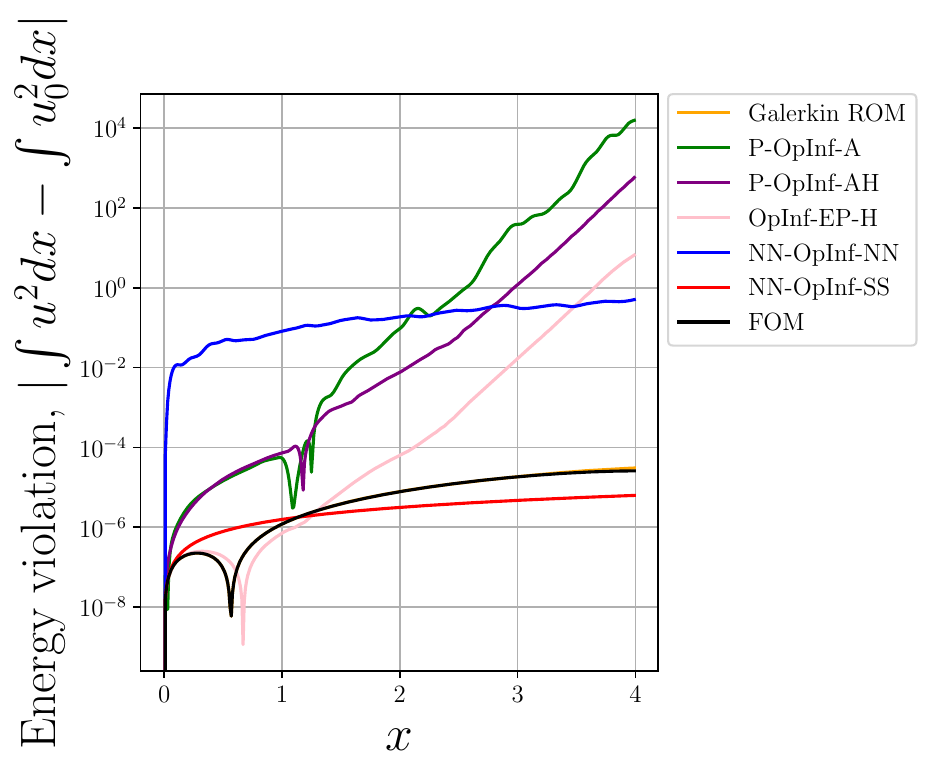}
\caption{Energy violation (future-state)}
\label{fig:burgers_training_error_converge_f}
\end{subfigure}

\caption{Burgers' equation. Relative errors, final-time solutions, and energy-conservation violations for various OpInf formulations.}
\label{fig:burgers_training_error_converge}
\end{center}
\end{figure}

To assess the sensitivity of the neural network results to stochastic training, we
repeat both configurations using five independently seeded realizations for each
combination of reduced basis dimension $K\in\{5,10,15,20\}$ and ensemble size
$L\in\{1,2,5\}$.  Each ensemble realization averages the predictions of $L$
independently trained networks.  Figures~\ref{fig:burgers_reproductive_variance}
and~\ref{fig:burgers_future_variance} show the distributions of the relative
space--time error over these five realizations.  Each box marks the interquartile
range and median, while the overlaid points show the individual realizations.

\begin{figure}
\begin{center}
\includegraphics[width=0.95\linewidth]{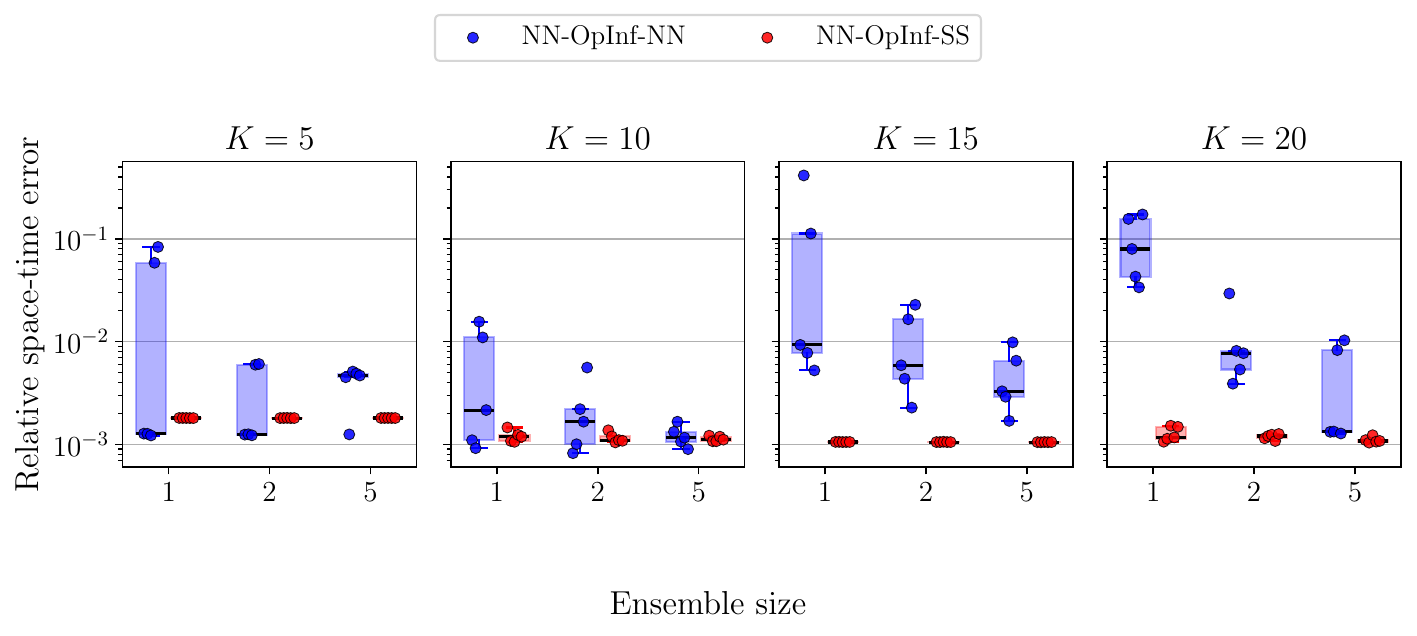}
\caption{Burgers' equation, reproductive configuration: distributions of the
relative space--time error over five independently seeded realizations for each
reduced basis dimension $K$ and ensemble size $L$.}
\label{fig:burgers_reproductive_variance}
\end{center}
\end{figure}

\begin{figure}
\begin{center}
\includegraphics[width=0.95\linewidth]{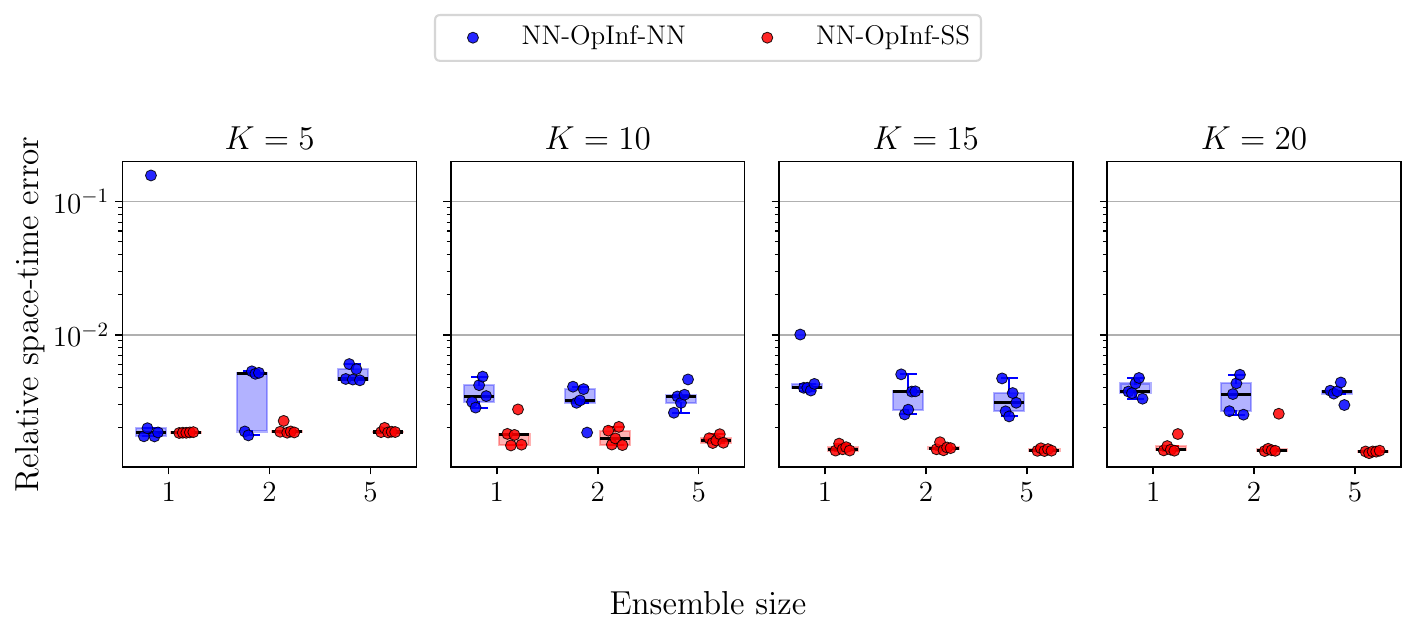}
\caption{Burgers' equation, future-state configuration: distributions of the
relative space--time error over five independently seeded realizations for each
reduced basis dimension $K$ and ensemble size $L$.}
\label{fig:burgers_future_variance}
\end{center}
\end{figure}

The contrast between the two neural network formulations is clear across both
configurations. NN-OpInf-SS produces consistently low errors with tightly
clustered realizations for every $K$ and $L$, demonstrating low sensitivity to
the stochastic components of training for this experiment. Its mean error remains between
$1.1\times10^{-3}$ and $1.8\times10^{-3}$ in the reproductive configuration and
between $1.3\times10^{-3}$ and $1.9\times10^{-3}$ in the future-state
configuration. In contrast, the NN-OpInf-NN results exhibit substantially larger
seed-to-seed variability. Increasing
the ensemble size reduces the influence of stochastic training, but NN-OpInf-NN
remains less consistent than NN-OpInf-SS.

\subsection{Nonlinear convection-diffusion-reaction system}
We next consider a nonlinear convection-diffusion-reaction (CDR) system with multiple terms and non-polynomial nonlinearities, which highlights the expressiveness and composability of NN-OpInf.
The governing equations on the domain $\spatialDomain \coloneqq [0,1] \times [0,1]$ with $t \in [0,2.5]$ are given by
$$\frac{\partial u}{\partial t} =  - u \left(\mathbf{b} \cdot \nabla u \right)  - \sigma  u \exp(-\eta u^2) +  \nu \nabla^2 u + q,$$
with homogeneous initial and boundary conditions.  The forcing term $q=5\exp(-r)$ with $r = \sqrt{ (x_1 - 0.5)^2 + (x_2 - 0.5)^2 }$, and the parameters are defined on $\sigma \sim U[2,3]$, $\nu \sim U[1 \times 10^{-5},1\times 10^{-3}]$, $\mathbf{b} = [0.5 \cos(\theta),0.5\sin(\theta) ]$ with $\theta \sim U[\frac{\pi}{6},\frac{\pi}{2}]$, $\eta \sim U[1,3]$.
The FOM is discretized with $100 \times 100$ spatial grid points. A second-order upwind scheme is used for the nonlinear convection terms, with a second-order central scheme for diffusion. We integrate with an RK4 method using $\dt = 0.002$ and collect snapshots at every time step. Figure~\ref{fig:cdr_fom} depicts the FOM solution at $t=0.625$ (left), $t=1.25$ (center), and $t=2.50$ (right) for a random parameter realization. 

In the absence of forcing, the nonlinear CDR system satisfies the energy stability property
$$\frac{d}{dt}\frac{1}{2}\int_{\Omega} u(x_1,x_2)^2 \, dA \le 0.$$ We construct an NN-OpInf model that enforces this property and admits a forcing vector by writing the right-hand side as a sum of a negative symmetric positive semi-definite operator, a skew-symmetric operator, and a constant forcing operator,
\begin{equation}\label{eq:nnopinf_pd_f}
\reducedVelocity(\reducedState,\params) =  -\reducedSpdMatrixLowerNN(\reducedState,\params;\weights_1) \reducedSpdMatrixLowerNN^\intercal(\reducedState,\params; \weights_1) \reducedState + 
\left[ \reducedSkewMatrixLowerNN(\reducedState,\params;\weights_2) - \reducedSkewMatrixLowerNN^\intercal(\reducedState,\params;\weights_2) \right] \reducedState + \skew{-4}\hat{\mathbf{b}}.
\end{equation}
We refer to this as the NN-OpInf positive semi-definite forcing (NN-OpInf-PSD-f) ROM.
We also consider a linear OpInf model with forcing (P-OpInf-cA), a quadratic OpInf model with forcing (P-OpInf-cAH), a quadratic OpInf model with an energy-preserving quadratic operator and affine forcing (P-OpInf-EP-cAH), the baseline NN-OpInf-NN model, and the intrusive POD-Galerkin ROM. We note that the Galerkin ROM does not employ hyper-reduction.
 
Two setups are considered:
\begin{itemize}
\item Reproductive: We collect snapshots and make predictions over $t \in [0,2.5]$ for a single parameter realization $\sigma = 2.551$, $\nu = 7.111 \times 10^{-4}$, $\theta = 0.8282$, and $\eta = 2.0217$.
\item Parametric prediction: We simulate the FOM for 36 parameter samples drawn from the parameter domain. The first 16 parameter instances are obtained from sampling the corners of the hypercube of the parameter domain, while the remaining 20 are obtained using uniform sampling. We collect snapshots over $t \in [0,2.5]$. We test on 20 novel parameter instances obtained using random sampling.
\end{itemize} 
\begin{figure}
\begin{center}
\begin{subfigure}[t]{0.32\textwidth}
\includegraphics[trim={6cm 2cm 4cm 2cm},clip,width=1.0\linewidth]{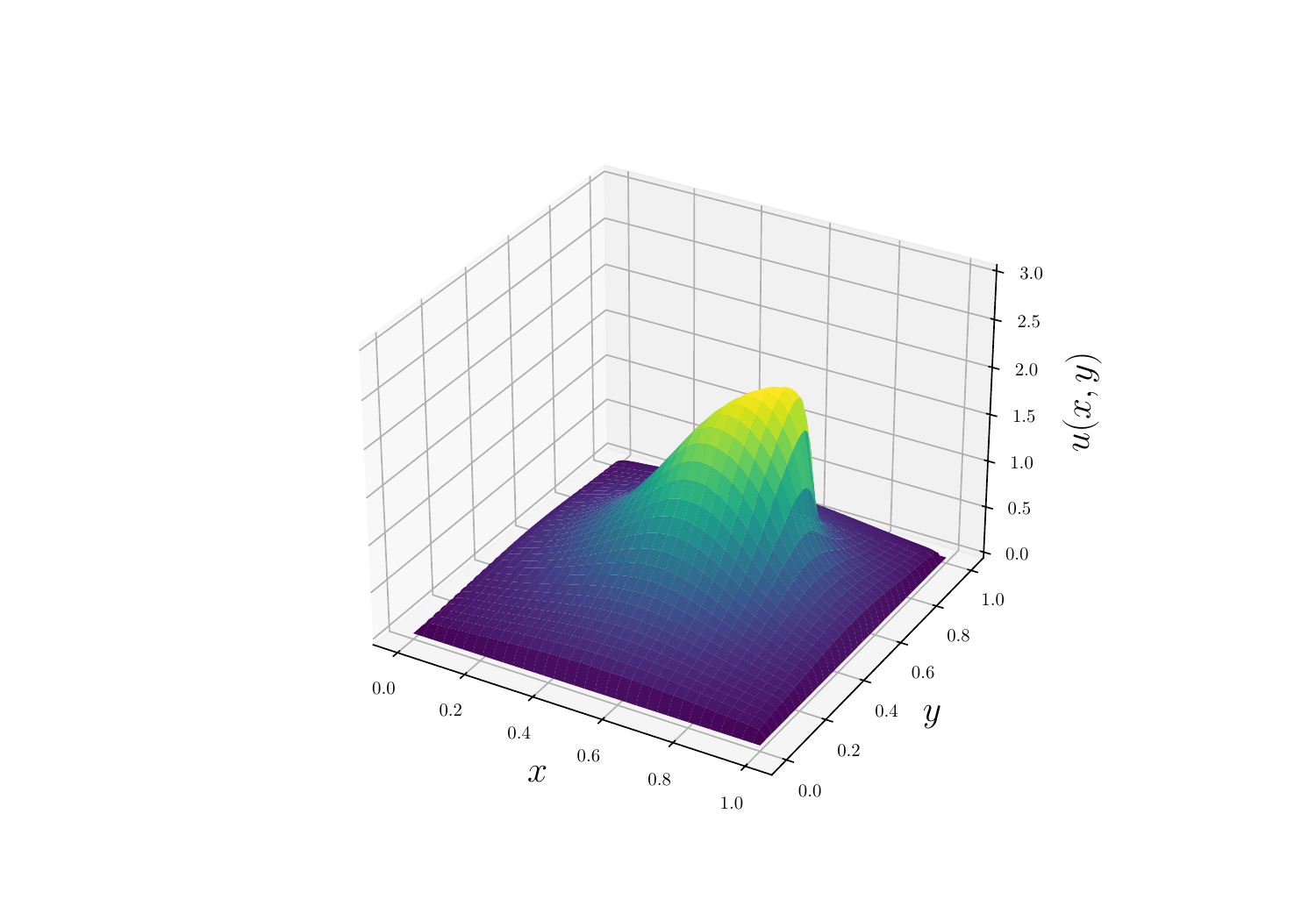}
\caption{$t=0.625$}
\end{subfigure}
\begin{subfigure}[t]{0.32\textwidth}
\includegraphics[trim={6cm 2cm 4cm 2cm},clip,width=1.0\linewidth]{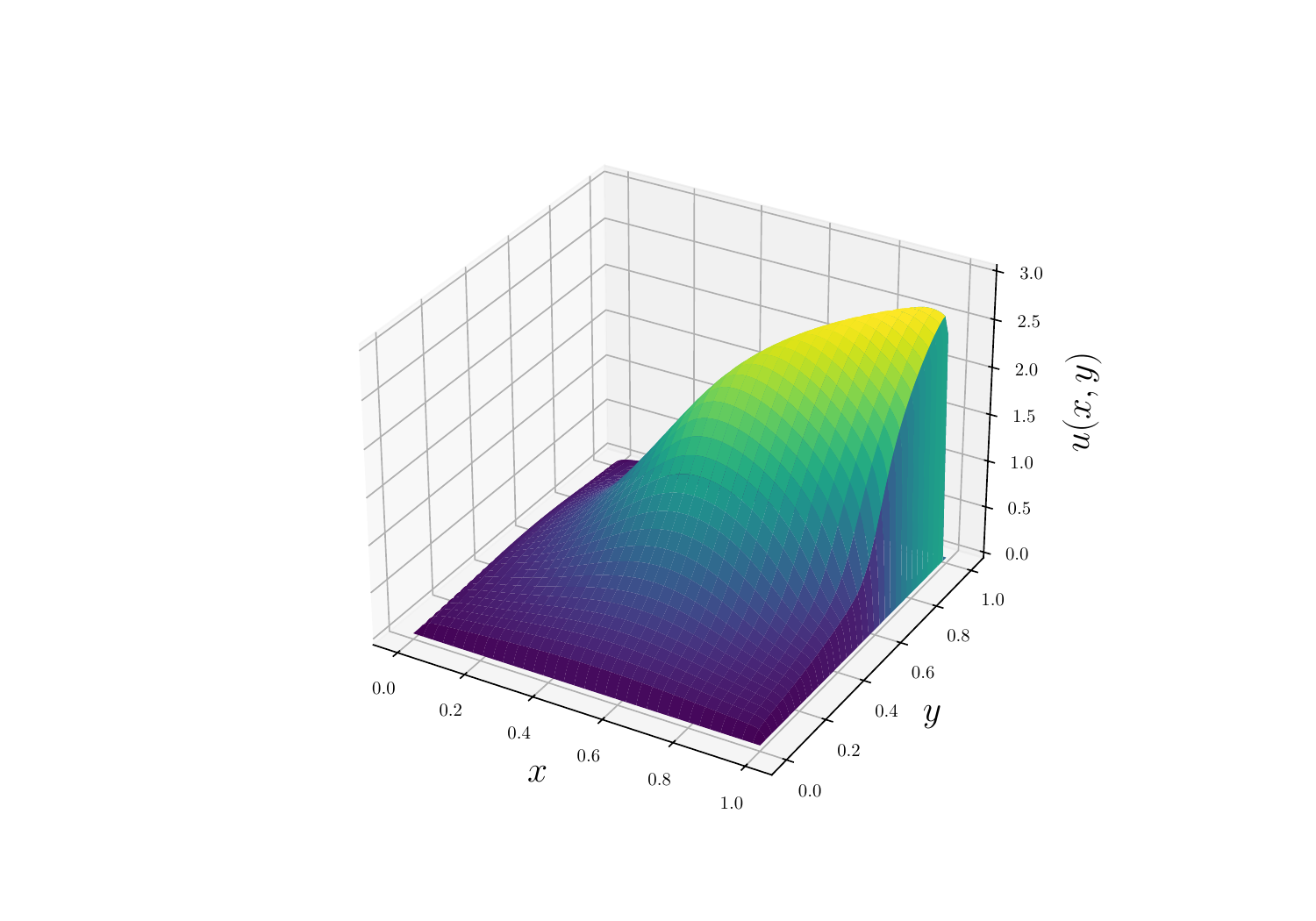}
\caption{$t=1.25$}
\end{subfigure}
\begin{subfigure}[t]{0.32\textwidth}
\includegraphics[trim={6cm 2cm 4cm 2cm},clip,width=1.0\linewidth]{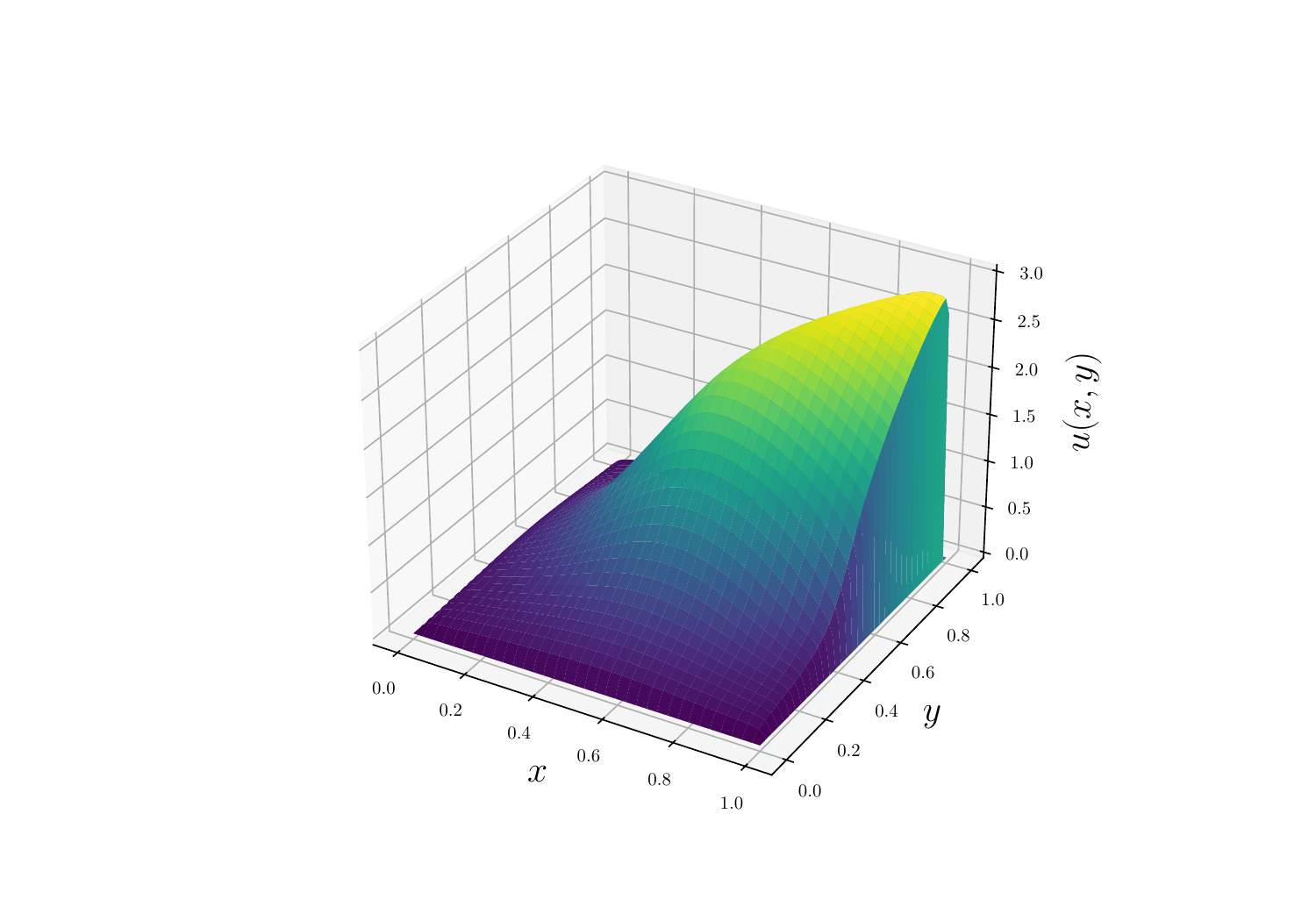}
\caption{$t=2.50$}
\end{subfigure}
\caption{Nonlinear CDR example. FOM solution for $\sigma = 2.551$, $\nu = 7.111 \times 10^{-4}$, $\theta = 0.8282$, and $\eta = 2.0217$.}
\label{fig:cdr_fom}
\end{center}
\end{figure}
Figure~\ref{fig:ex2_reproductive} shows results for the reproductive configuration. In Figure~\ref{fig:ex2_reproductive_a}, all three P-OpInf models yield relative errors greater than $5\%$ for every basis dimension, reflecting the limitations of their global polynomial model forms for this non-polynomial system. P-OpInf-EP-cAH is less accurate than P-OpInf-cA and P-OpInf-cAH for $K<25$; although its error decreases substantially with basis dimension, it is only comparable to the other P-OpInf models at $K=25$. NN-OpInf-PSD-f is the most accurate model and outperforms both the Galerkin ROM and all three P-OpInf models across the basis dimensions considered. Conversely, the standard NN-OpInf-NN model is unstable for most configurations. In the physical-space solutions at $t=1.0$ (Figure~\ref{fig:ex2_reproductive_b}), the P-OpInf models capture the overall profile but differ in their prediction of the moving shock. P-OpInf-EP-cAH predicts the shock location more accurately than P-OpInf-cA and P-OpInf-cAH, but does not improve the total space--time error. NN-OpInf-PSD-f is in close agreement with the FOM, while the Galerkin ROM exhibits oscillations near the shock.
\begin{figure}
\begin{center}
\begin{subfigure}[t]{0.49\textwidth}
\includegraphics[trim={0cm 0cm 0cm 0cm},clip,width=1.0\linewidth]{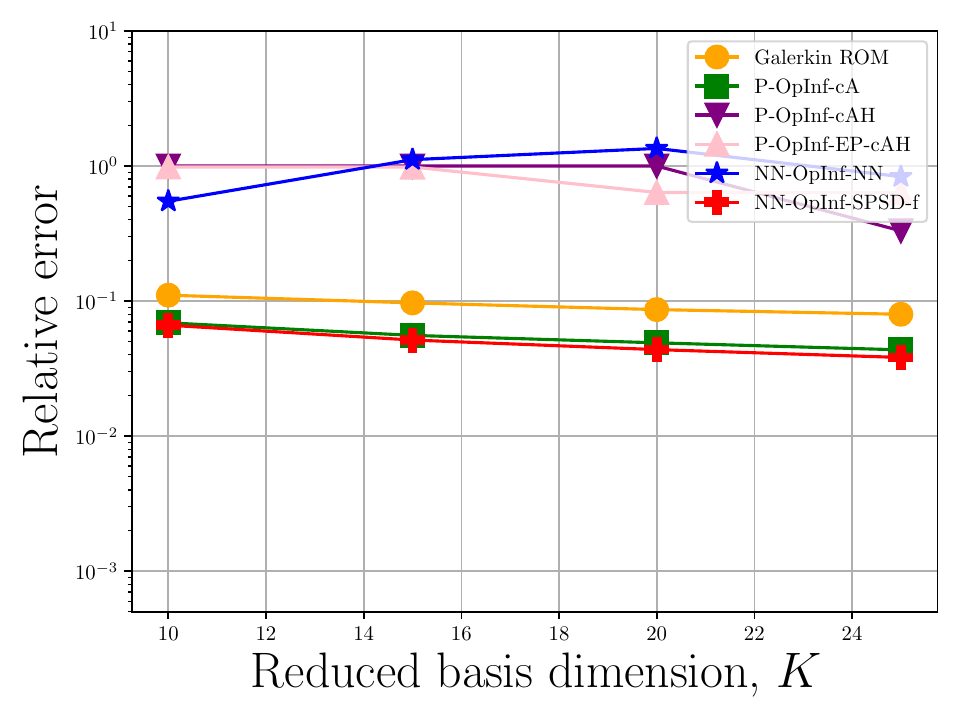}
\caption{Error as a function of reduced-basis dimension}
\label{fig:ex2_reproductive_a}
\end{subfigure}
\begin{subfigure}[t]{0.49\textwidth}
\includegraphics[trim={0cm 0cm 0cm 0cm},clip,width=1.0\linewidth]{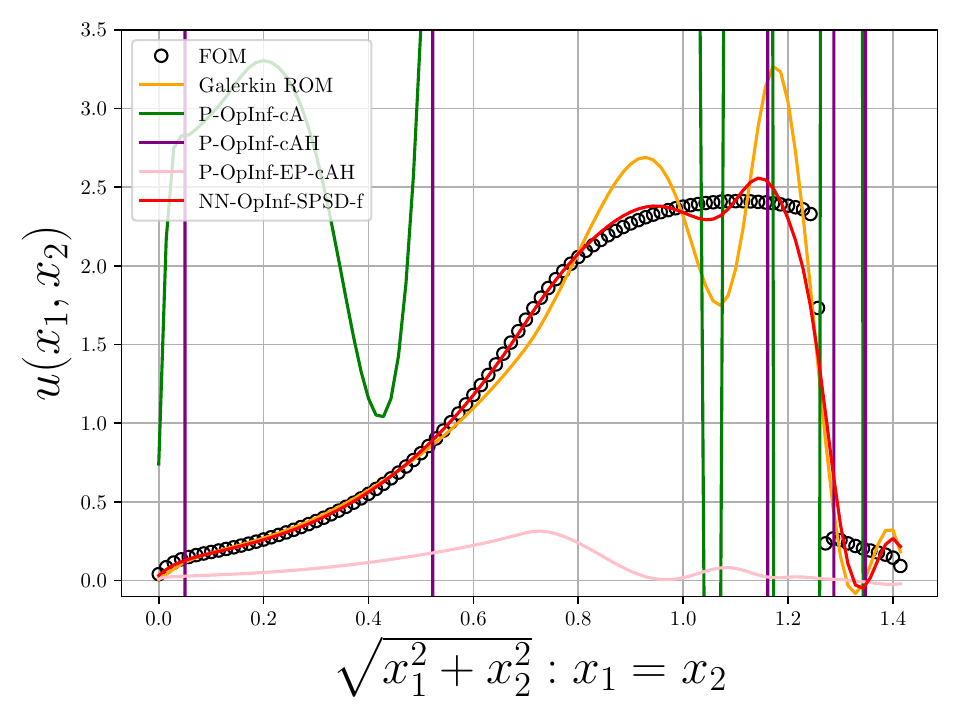}
\caption{Solution profiles at $t=1.0$}
\label{fig:ex2_reproductive_b}
\end{subfigure}
\caption{Nonlinear CDR example in a reproductive configuration. Relative errors for various OpInf formulations (left) and solutions for the finest ROMs at $t = 1.0$ (right). }
\label{fig:ex2_reproductive}
\end{center}
\end{figure}
%
%

Figure~\ref{fig:ex2_parametric} shows results for the parametric configuration. Figures~\ref{fig:ex2_parametric_a}--\ref{fig:ex2_parametric_b} report reproductive results on the training parameter set, while Figures~\ref{fig:ex2_parametric_c}--\ref{fig:ex2_parametric_d} report predictive results on the testing set. When evaluated on the training set (Figure~\ref{fig:ex2_parametric_a}), NN-OpInf-PSD-f is the best-performing model, followed closely by P-OpInf-cA and then the Galerkin ROM. P-OpInf-cAH and P-OpInf-EP-cAH are stable but inaccurate, and NN-OpInf-NN is unstable. On the testing set, both unconstrained P-OpInf models become unstable, while P-OpInf-EP-cAH remains bounded over most of the basis dimensions considered but is inaccurate. Thus, enforcing energy preservation improves the robustness of the quadratic P-OpInf model but does not overcome the limitations of its polynomial model form and the interpolatory parametric treatment for these non-polynomial dynamics. NN-OpInf-NN remains unstable, whereas NN-OpInf-PSD-f is the most accurate model and slightly outperforms the Galerkin ROM across all basis dimensions.
\begin{figure}
\begin{center}
\begin{subfigure}[t]{0.49\textwidth}
\includegraphics[trim={0cm 0cm 0cm 0cm},clip,width=1.0\linewidth]{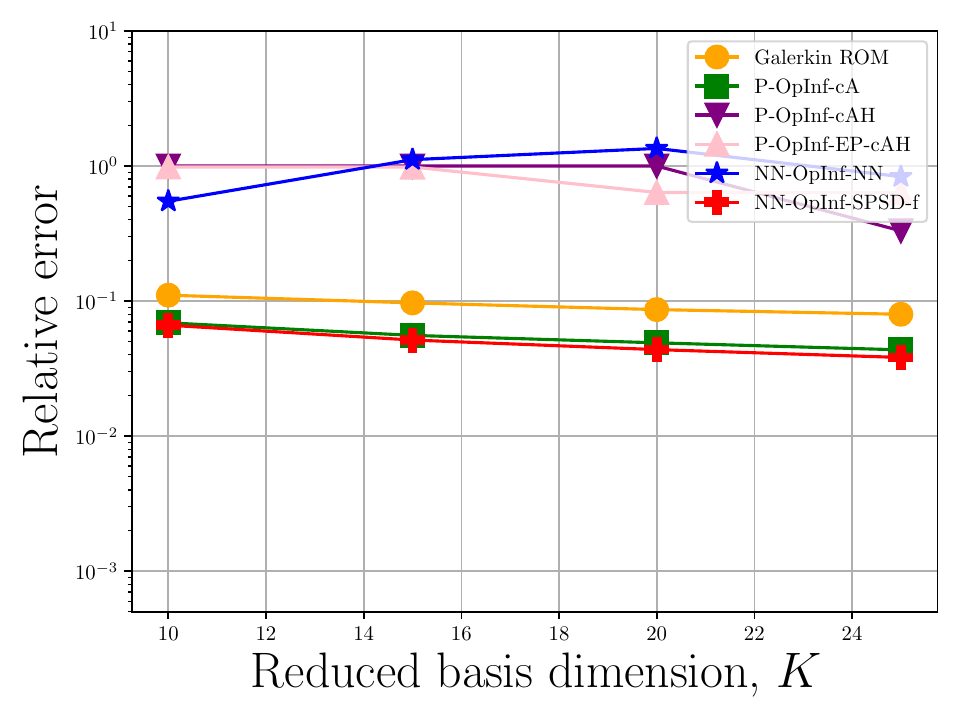}
\caption{Error on the training set as a function of reduced-basis dimension}
\label{fig:ex2_parametric_a}
\end{subfigure}
\begin{subfigure}[t]{0.49\textwidth}
\includegraphics[trim={0cm 0cm 0cm 0cm},clip,width=1.0\linewidth]{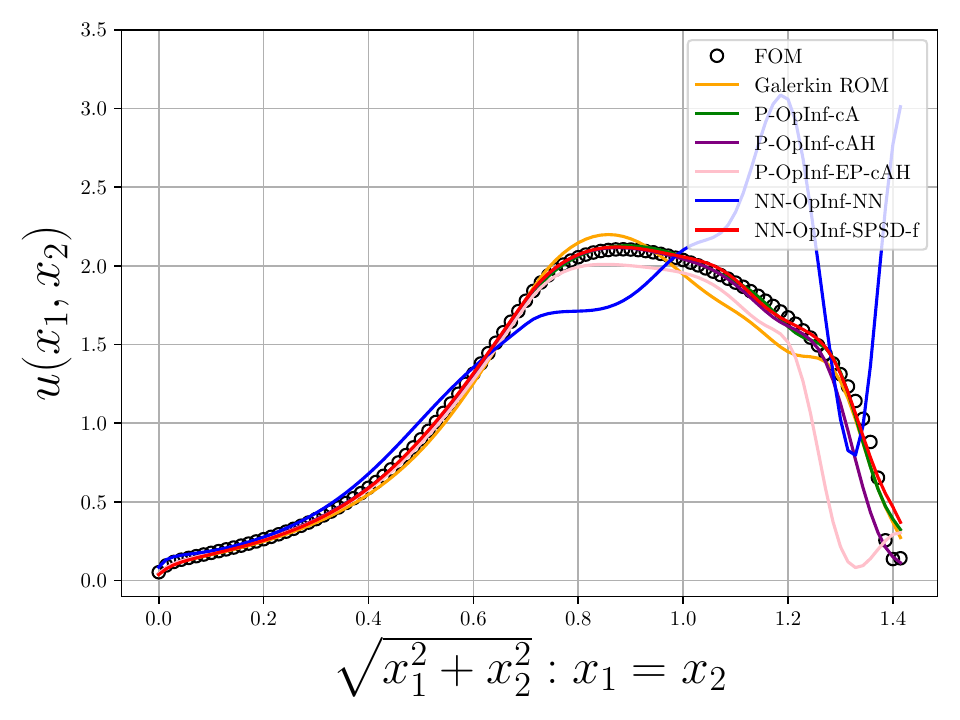}
\caption{Solution profiles at $t=2.5$ for a parameter instance randomly selected from the training set}
\label{fig:ex2_parametric_b}
\end{subfigure}
\begin{subfigure}[t]{0.49\textwidth}
\includegraphics[trim={0cm 0cm 0cm 0cm},clip,width=1.0\linewidth]{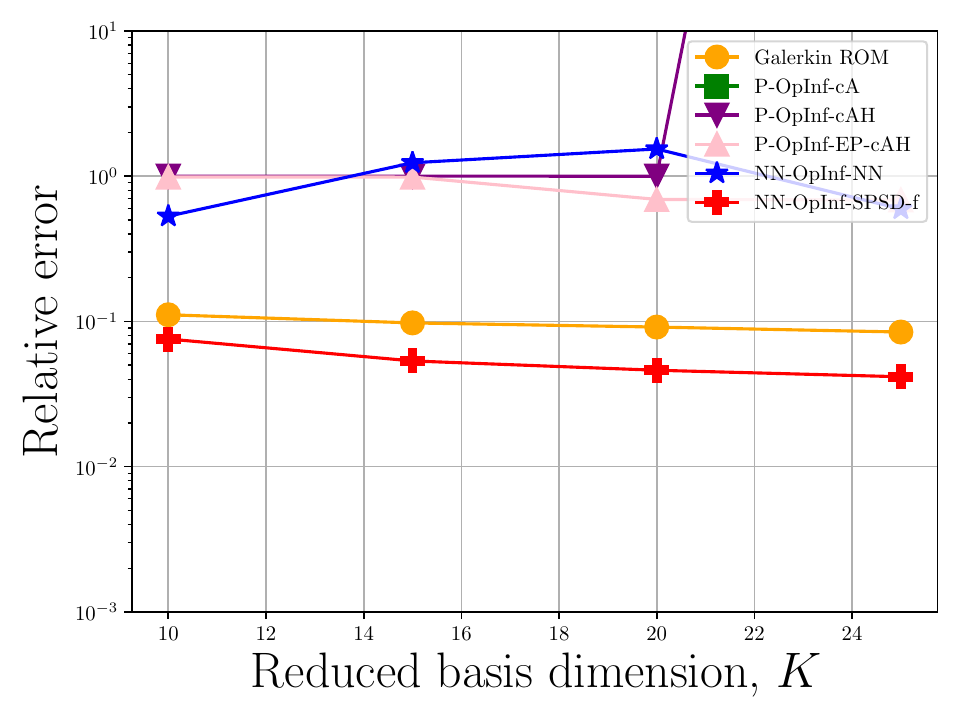}
\caption{Error on the testing set as a function of reduced-basis dimension}
\label{fig:ex2_parametric_c}
\end{subfigure}
\begin{subfigure}[t]{0.49\textwidth}
\includegraphics[trim={0cm 0cm 0cm 0cm},clip,width=1.0\linewidth]{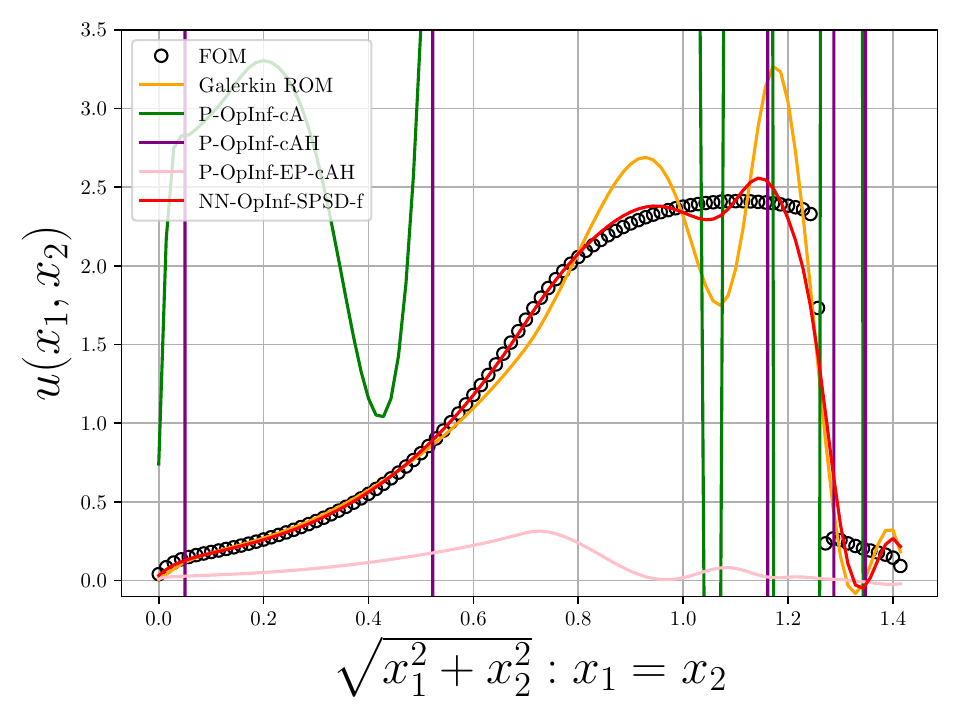}
\caption{Solution profiles at $t=2.5$ for a parameter instance randomly selected from the testing set}
\label{fig:ex2_parametric_d}
\end{subfigure}
\caption{Nonlinear CDR example in a parametric configuration. Relative errors for various OpInf formulations (left) and solutions for the finest ROMs at $t = 2.5$ (right). Solutions across the training set are shown on the top, while the testing set is shown on the bottom. In the testing regime, the P-OpInf-cAH ROM went unstable for $K>20$.}
\label{fig:ex2_parametric}
\end{center}
\end{figure}


\subsection{2D nonlinear heat conduction}
We next consider a two-dimensional (2D) nonlinear heat-conduction problem on the unit square $[0,1]^2$ with temperature-dependent diffusivity $\kappa(T)$ and a constant source term. This example again illustrates the ability of NN-OpInf to learn composable structure-preserving operators. The governing equations are
$$
\partial_t T(x,y,t) - \nabla \cdot \big(\kappa(T)\nabla T(x,y,t)\big) = 1, \quad (x,y)\in [0,1]^2,
$$
with Dirichlet boundary conditions $T(x,y,t) = 0$ on $\partial [0,1]^2$ and initial condition $T(x,y,0)=0$.
The diffusivity is parameterized as a smooth transition between $k_0$ and $k_1$,
$$
\kappa(T) = \tfrac{1}{2}(k_0+k_1) + \tfrac{1}{2}(k_1-k_0)\tanh\left(\frac{T-T_c}{w}\right),
$$
with $k_0=10^{-2}$, $w =  2 \times 10^{-2}$, $k_1 \sim U[0.2,0.4]$, and $T_c\sim U[0.3,0.4]$. 
The spatial discretization involves a uniform grid with $(n_x+1)\times(n_y+1)$ nodes and spacings $h_x=1/n_x$, $h_y=1/n_y$. A standard five-point finite-difference stencil is used with averaging of $\kappa(T)$ on cell faces to form a symmetric diffusion operator. Dirichlet boundary conditions are enforced weakly. For time integration, we use a Crank--Nicolson scheme. 

The spatial discretization results in a system that can be expressed as 
$$\dot\state = -\mathbf{A}(\state) \state + \mathbf{f},$$
where $\mathbf{A}(\state)$ is SPD. Thus, we consider the OpInf model form
\begin{equation}\label{eq:nnopinf_spd_f}
\reducedVelocity(\reducedState,\params) =  -\reducedSpdMatrixLowerNN(\reducedState,\params;\weights_1) \reducedSpdMatrixLowerNN^\intercal(\reducedState,\params; \weights_1) \reducedState + \skew{-4}\hat{\mathbf{b}}.
\end{equation}
We refer to this formulation as the NN-OpInf symmetric positive semi-definite with forcing (NN-OpInf-SPSD-f) ROM. We additionally consider P-OpInf-cA, P-OpInf-cAH, and NN-OpInf-NN. The Galerkin ROM is unavailable for this example.

Three setups are considered:
\begin{enumerate}
\item Reproductive: We collect snapshots and make predictions over $t \in [0,2.0]$ at a  single parameter realization $k_1 = 0.2$, $T_c = 0.3$.
\item Future-state prediction: We use the same parametric configuration but train only on $t \in [0,1.0]$, predicting over $t \in [0,2.0]$.
\item Parametric prediction: We simulate the FOM for four parameter samples drawn from the parameter domain and collect snapshots over $t \in [0,2.0]$. We train on the parameter grid $k_1 = [0.2,0.4]$, $T_c = [0.3,0.4]$ and test on four randomly drawn instances from the distribution.
\end{enumerate} 

Figure~\ref{fig:heat_training_error_converge} reports relative errors for the three configurations. We observe that NN-OpInf-SPSD-f outperforms both vanilla NN-OpInf-NN and the P-OpInf approaches in all cases. For both the reproductive and future-state configurations, NN-OpInf-SPSD-f is 5--10$\times$ more accurate than P-OpInf-cA. Both P-OpInf-cAH and NN-OpInf-NN result in unphysical solutions and do not improve as the basis dimension is refined. NN-OpInf-SPSD-f displays a general convergence trend with increasing basis dimension, but convergence is non-monotonic. In the parametric-predictive settings, NN-OpInf-SPSD-f outperforms P-OpInf and NN-OpInf-NN, but again does not display monotonic convergence. While monotonic convergence can be obtained by tweaking training settings (e.g., changing the learning rate, regularization) or increasing the ensemble size, this highlights challenges associated with neural network training and inference in the presence of generic nonlinearities. Nonetheless, NN-OpInf-SPSD-f clearly outperforms existing OpInf ROM methods.  
\begin{figure}
\begin{center}
\begin{subfigure}[t]{0.32\textwidth}
\includegraphics[trim={0cm 0cm 0cm 0cm},clip,width=1.0\linewidth]{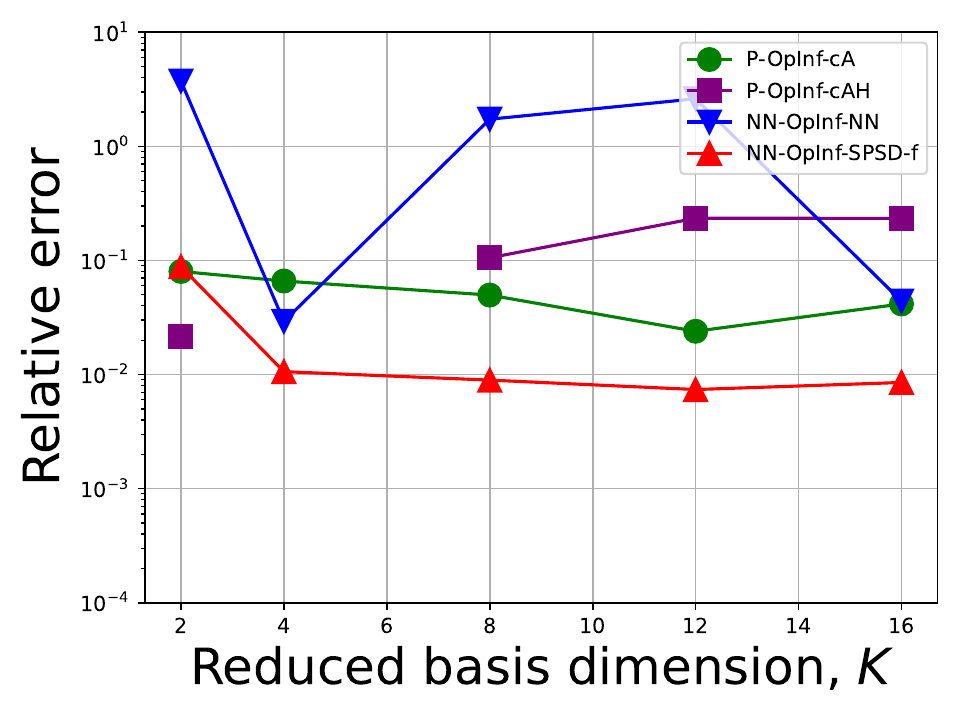}
\caption{Reproductive: relative error vs. basis dimension}
\label{fig:heat_reproductive_error_converge}
\end{subfigure}
\begin{subfigure}[t]{0.32\textwidth}
\includegraphics[trim={0cm 0cm 0cm 0cm},clip,width=1.0\linewidth]{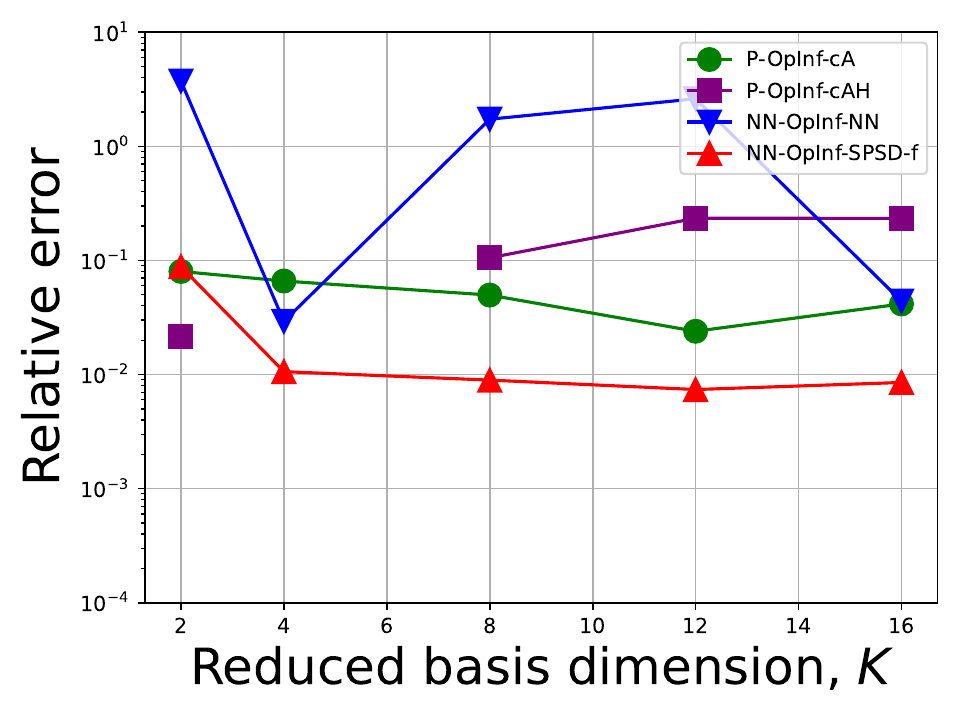}
\caption{Future-state: relative error vs. basis dimension}
\label{fig:heat_future_error_converge}
\end{subfigure}
\begin{subfigure}[t]{0.32\textwidth}
\includegraphics[trim={0cm 0cm 0cm 0cm},clip,width=1.0\linewidth]{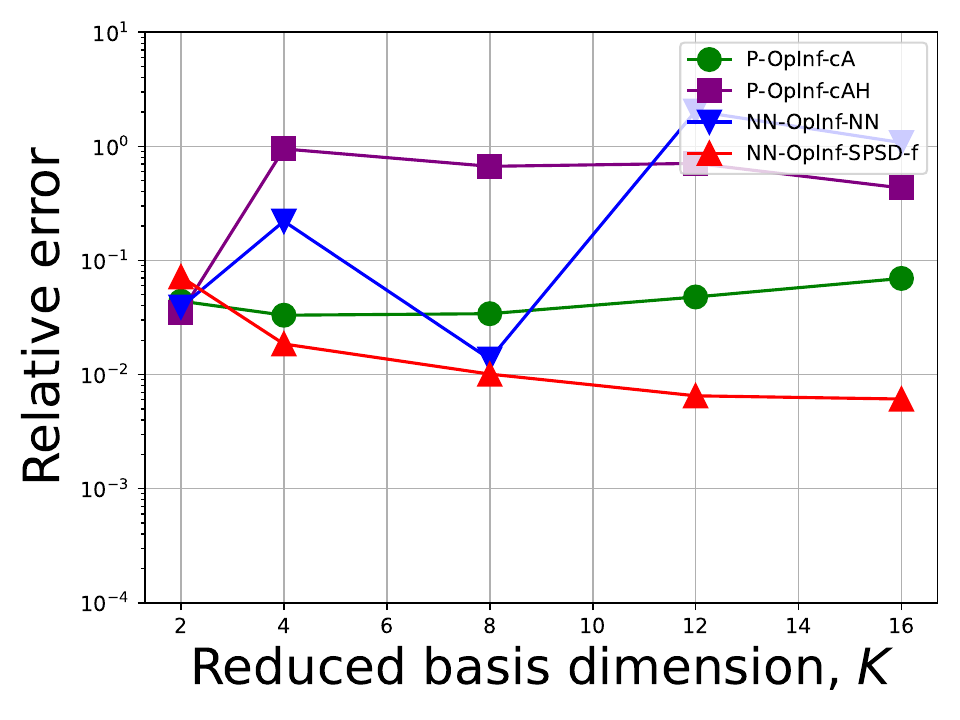}
\caption{Parametric: relative error vs. basis dimension on the testing set}
\label{fig:heat_parametric_error_converge}
\end{subfigure}
\caption{2D nonlinear heat conduction. Relative errors for various OpInf formulations. We note that the P-OpInf-cAH ROM at $K=4$ was unstable in the future-state configuration.}
\label{fig:heat_training_error_converge}
\end{center}
\end{figure}

To assess the sensitivity of the neural-network results to stochastic training, we
repeat the reproductive and future-state configurations using five independently
seeded realizations for each combination of reduced basis dimension
$K\in\{2,4,8,12,16\}$ and ensemble size $L\in\{1,2,5\}$. 
Figure~\ref{fig:heat_variance} shows the distributions of the relative space--time
error over the five realizations. Each box marks the interquartile range and
median, while the overlaid points show the individual realizations.
The unstructured NN-OpInf-NN model exhibits substantial seed-to-seed
variability in both configurations. NN-OpInf-SPSD-f exhibits substantially less variation and 
consistently produces more accurate results. This trend is present in both configurations, except at $K=2$, where all models have a relatively low learning capacity. 
\begin{figure}
\begin{center}
\begin{subfigure}[t]{0.95\textwidth}
\includegraphics[width=1.0\linewidth]{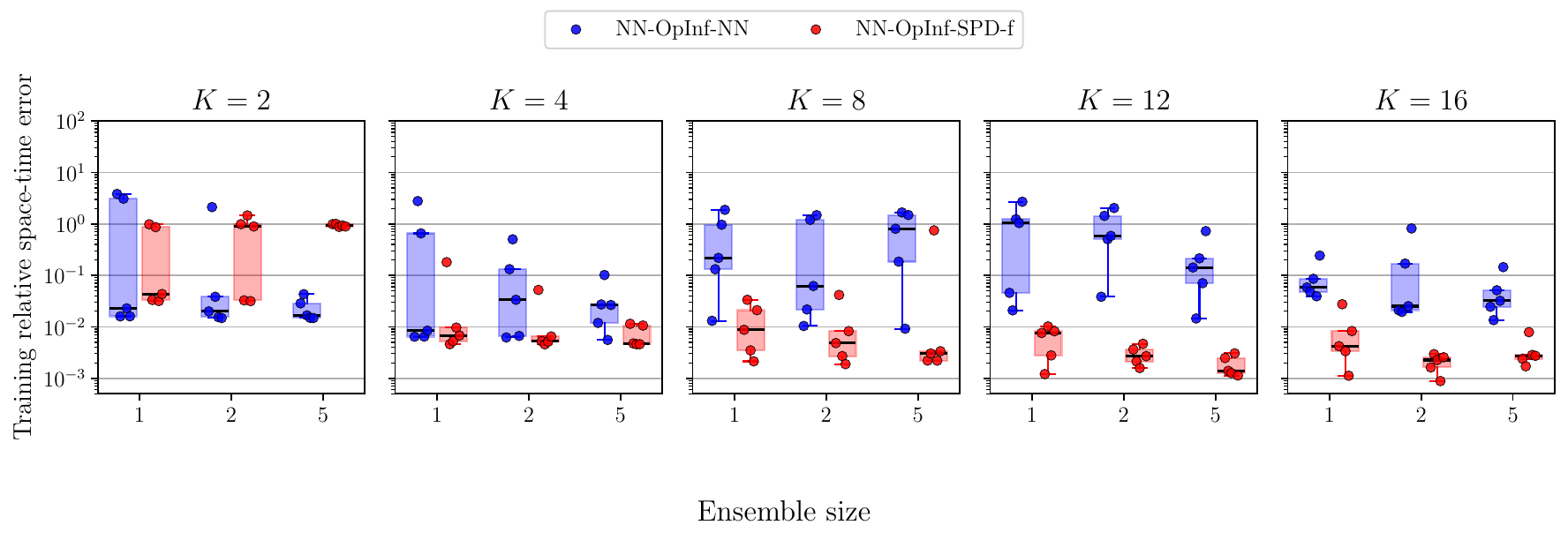}
\caption{Reproductive configuration}
\label{fig:heat_reproductive_variance_training}
\end{subfigure}
\begin{subfigure}[t]{0.95\textwidth}
\includegraphics[width=1.0\linewidth]{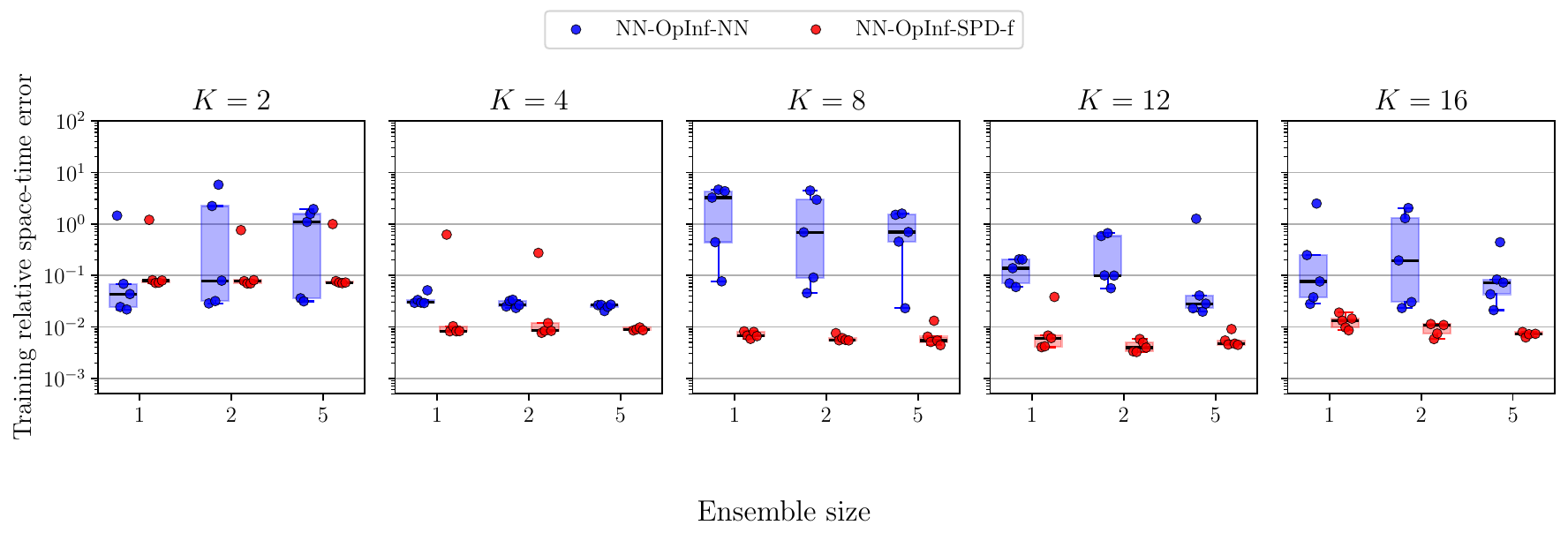}
\caption{Future-state configuration}
\label{fig:heat_future_variance_training}
\end{subfigure}
\caption{2D nonlinear heat conduction, reproductive configuration (top) and future-state configuration (bottom):
distributions of relative space--time error over five independently seeded
realizations for each reduced basis dimension $K$ and ensemble size $L$.}
\label{fig:heat_variance}
\end{center}
\end{figure}

\subsection{Premixed $\HTwo$-air flame}

The next example is a simplified model for a premixed $\HTwo$-air flame, which highlights the ability to learn an OpInf model with a parameterized nonlinear forcing term. This example has been studied in various references~\cite{LeCa20,GaWaZh20,BuWi10}, and non-intrusive neural-network-based ROMs were examined in Ref.~\cite{GaWaZh20}. The system is described by the PDEs on $\Omega \times [0,T] \times \paramDomain$, 
\begin{equation}
\begin{aligned}
\frac{\partial }{\partial t} \boldsymbol{u}(y,t,\params) + \boldsymbol{\beta} \cdot \nabla \boldsymbol{u} (y,t,\params) &= \kappa \Delta \boldsymbol{u} (y,t,\params) + \boldsymbol{R}(\boldsymbol{u}, \params), & \quad & y \in \Omega, \\
\boldsymbol{u} (y,t,\params) &= \boldsymbol{u}_d(y,t,\params), & \quad & y \in \Gamma_d, \\
\boldsymbol{\nabla}  \boldsymbol u (y,t,\params) \cdot \boldsymbol{n} &= 0, & \quad & y \in \Gamma_n, \\
\boldsymbol{u}(y,0,\params) &= \boldsymbol{u}_0(y), & \quad & y \in \Omega. \\
\end{aligned}
\end{equation}
The spatial domain is $\Omega \defeq \left[0,L_x\right] \times \left[0,L_y \right]$ with $L_x = 1.8$ $\mathrm{cm}$ and $L_y = 0.90$ $\mathrm{cm}$, and the temporal domain is $t \in [0, 0.06 \mathrm{s}]$.
The components of the state $\boldsymbol u \coloneqq \left[ \mathrm{Y}_{\mathrm{H}_2},\mathrm{Y}_{\mathrm{O}_2},\mathrm{Y}_{\mathrm{H}_2\mathrm{O}}, \Theta \right]$ are the mass fractions of the species $\mathrm{H}_2, \mathrm{O}_2$, and $\mathrm{H}_{2}\mathrm{O}$, and the temperature, respectively, $\boldsymbol \beta  = \left[ 50 , 0 \right]^\intercal\in \RR{2}$ is the advection speed with units $\mathrm{cm}\cdot \mathrm{s}^{-1}$, $\kappa = 2\,\mathrm{cm}^2 \cdot \mathrm{s}^{-1}$ is the diffusion parameter, and $\boldsymbol R$ is a nonlinear reaction source term of the form
\begin{equation}
\begin{aligned}
R_i &= -\nu_i \left( \frac{W_i}{\rho} \right) \left( \frac{\rho \mathrm{Y}_{\mathrm{H}_2}}{\mathrm{W}_{\mathrm{H}_2}}  \right)^{\nu_{\mathrm{H}_2}} \left( \frac{\rho \mathrm{Y}_{\mathrm{O}_2}}{\mathrm{W}_{\mathrm{O}_2}}  \right)^{\mathrm{\nu}_{\mathrm{O}_2}} A \exp \left( -\frac{E}{R \Theta} \right), \\
R_\Theta &= Q R_{\HTwoO}.
\end{aligned}
\end{equation}
Here, $i\in \{\HTwo,\OTwo,\HTwoO\}$ indexes the species,  
$\nu_{\HTwo},\nu_{\OTwo},\nu_{\HTwoO} = 2,1,-2$ are the stoichiometric coefficients, $W_{\HTwo}, W_{\OTwo}, W_{\HTwoO} = 2.016, 31.9, 18$ are the molecular weights with units $\mathrm{g} \cdot \mathrm{mol}^{-1}$, $\rho = 1.39 \times 10^{-3}$ $\mathrm{g} \cdot \mathrm{cm}^{-3}$ is the density, $R = 8.314$ $\mathrm{J} \cdot  \mathrm{mol}^{-1} \cdot \mathrm{K}^{-1}$ is the universal gas constant, and $Q = 9800 \; \mathrm{K}$ is the heat of reaction. 
As depicted in Figure~\ref{fig:air_flame_schematic}, the domain boundary $\partial \Omega$ is decomposed into six segments, $\partial \Omega = \cup_{i=1}^6 \Gamma_i$ with a Dirichlet part $\Gamma_d = \cup_{i=1}^3 \Gamma_i$ and a Neumann part $\Gamma_n = \cup_{i=4}^6 \Gamma_i$. The initial condition is $\boldsymbol u(y,0,\params) = \left[0,0,0,300 \right]^\intercal$. The boundary conditions on $\Gamma_d$ are
$$\boldsymbol u_d(y,t,\params) = \begin{cases}
\left[0,0,0,300 \right]^\intercal \qquad &y \in \Gamma_1 \cup \Gamma_3, \\
\left[0.0282,0.2259,0,950 \right]^\intercal \qquad &y  \in \Gamma_2. \\
\end{cases}
$$ 
\begin{figure}
\centering
\begin{tikzpicture}[font=\large, line cap=round, line join=round]

\def\W{9}   
\def\H{4.5} 

\draw[line width=2.2pt] (0,0) rectangle (\W,\H);

\node at (1.2,3.85) {$\Gamma_{1}$};
\node at (1.2,2.25) {$\Gamma_{2}$};
\node at (1.2,0.65) {$\Gamma_{3}$};
\node at (4.7,0.55) {$\Gamma_{4}$};
\node at (8.2,2.35) {$\Gamma_{5}$};
\node at (4.5,3.85) {$\Gamma_{6}$};

\tikzset{
  dim/.style={line width=0.9pt},
  dimArrow/.style={<->,>=Stealth,dim},
  dimTick/.style={dim}
}

\draw[dim] (-0.45,0) -- (-0.45,\H);
\draw[dimTick] (-0.55,0) -- (-0.35,0);
\draw[dimTick] (-0.55,1.5) -- (-0.35,1.5);
\draw[dimTick] (-0.55,3.0) -- (-0.35,3.0);
\draw[dimTick] (-0.55,\H) -- (-0.35,\H);

\draw[dimArrow] (-1.05,3.0) -- (-1.05,\H)
  node[midway,left=3pt] {$3~\mathrm{mm}$};
\draw[dimArrow] (-1.05,1.5) -- (-1.05,3.0)
  node[midway,left=3pt] {$3~\mathrm{mm}$};
\draw[dimArrow] (-1.05,0) -- (-1.05,1.5)
  node[midway,left=3pt] {$3~\mathrm{mm}$};

\node[rotate=90] at (-3.0,2.25) {$x_{2}$};

\draw[dim] (\W+0.45,0) -- (\W+0.45,\H);
\draw[dimTick] (\W+0.35,0) -- (\W+0.55,0);
\draw[dimTick] (\W+0.35,\H) -- (\W+0.55,\H);
\draw[dimArrow] (\W+1.05,0) -- (\W+1.05,\H)
  node[midway,right=3pt] {$9~\mathrm{mm}$};

\draw[dim] (0,-0.45) -- (\W,-0.45);
\draw[dimTick] (0,-0.35) -- (0,-0.55);
\draw[dimTick] (\W,-0.35) -- (\W,-0.55);
\draw[dimArrow] (0,-1.05) -- (\W,-1.05)
  node[midway,below=3pt] {$18~\mathrm{mm}$};

\node at (\W/2,-1.95) {$x_{1}$};

\end{tikzpicture}
\caption{Premixed $\HTwo$ air-flame example. Schematic of geometry and boundary conditions.}
\label{fig:air_flame_schematic}
\end{figure}
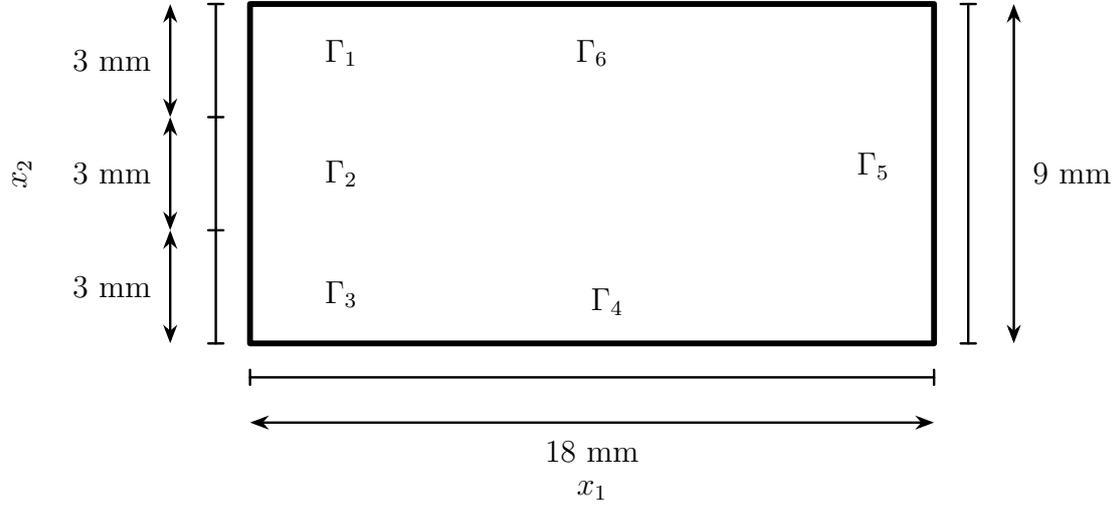

As in Ref.~\cite{GaWaZh20}, we consider two parameters: the pre-exponential factor $\mu_1 = A$ and the activation energy $\mu_2 = \mathrm{E}$. The parameter space is $\paramDomain = \left[ 2.3375 \times 10^{12} ,6.2 \times 10^{12} \right] \times \left[ 5625.5, 9000 \right].$ Note that the governing equations are nonlinear in the state and non-affine in the parameters, due to the presence of $\exp(-E/R\Theta)$ in the equation for $R_i$. 
The FOM is discretized with a second-order finite difference method on a $50 \times 25$ grid. We employ central differencing for diffusion and second-order upwind differencing for advection. Time integration uses second-order Crank--Nicolson with time step $1 \times 10^{-4}$, yielding $600$ snapshots per FOM sample. Figure~\ref{fig:air_flame_foms} presents FOM temperature snapshots at a randomly selected testing parameter instance. We examine the same NN-OpInf-PSD-f ROM as considered in Eq.~\eqref{eq:nnopinf_pd_f}, with the exception that the nonlinear forcing term is taken to be a fully connected feed forward neural network whose inputs are the parameters. We additionally examine P-OpInf-cA, P-OpInf-cAH, and NN-OpInf-NN models for comparison. Although P-OpInf-EP-cAH is also appropriate for this problem, we do not consider it here because the cost of the associated global KKT solve becomes prohibitive at the larger reduced dimensions used in this example.

We use a training/testing protocol similar to Ref.~\cite{GaWaZh20}. The training set $\paramDomainTrain$ is a uniform $4 \times 4$ grid in $\paramDomain$, while the testing set is $\paramDomainTest = \left\{ 2.98125 \times 10^{12}, 4.26875\times 10^{12},5.55625\times 10^{12} \right\} \times \left\{ 6.187917\times 10^{3},7.31275\times 10^{3},8.437583\times 10^{3} \right\}$. The only difference from Ref.~\cite{GaWaZh20} is that our testing set excludes the training samples; Ref.~\cite{GaWaZh20} used a $7 \times 7$ grid with $\paramDomainTrain \subset \paramDomainTest$.

\begin{figure}
\begin{center}
\begin{subfigure}[t]{0.32\textwidth}
\includegraphics[trim={0cm 0cm 0cm 0cm},clip,width=1.0\linewidth]{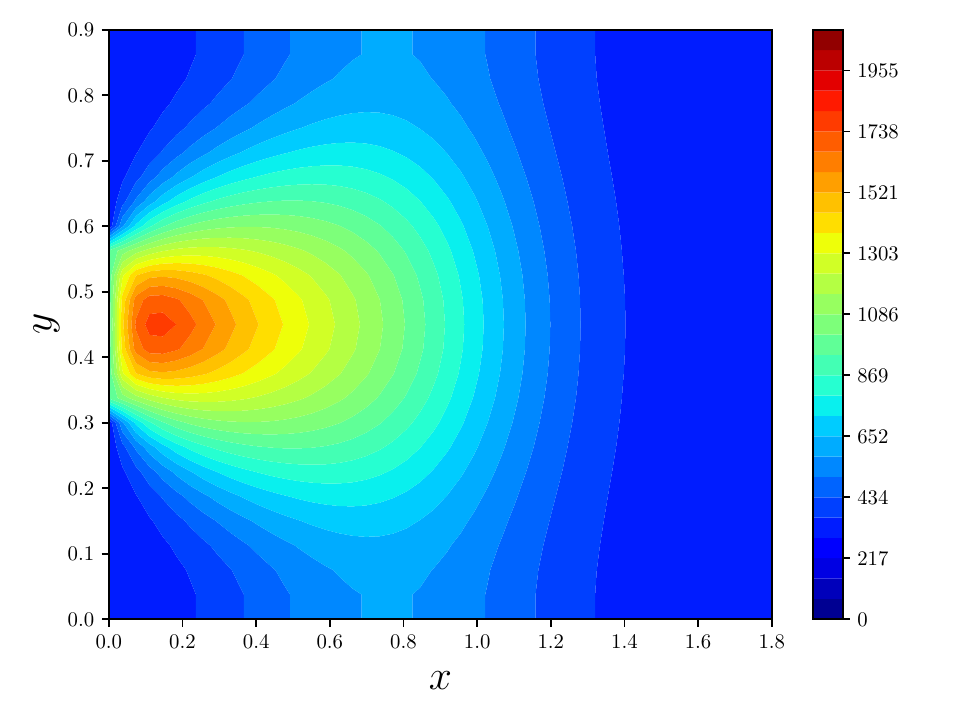}
\caption{$t=0.02$}
\end{subfigure}
\begin{subfigure}[t]{0.32\textwidth}
\includegraphics[trim={0cm 0cm 0cm 0cm},clip,width=1.0\linewidth]{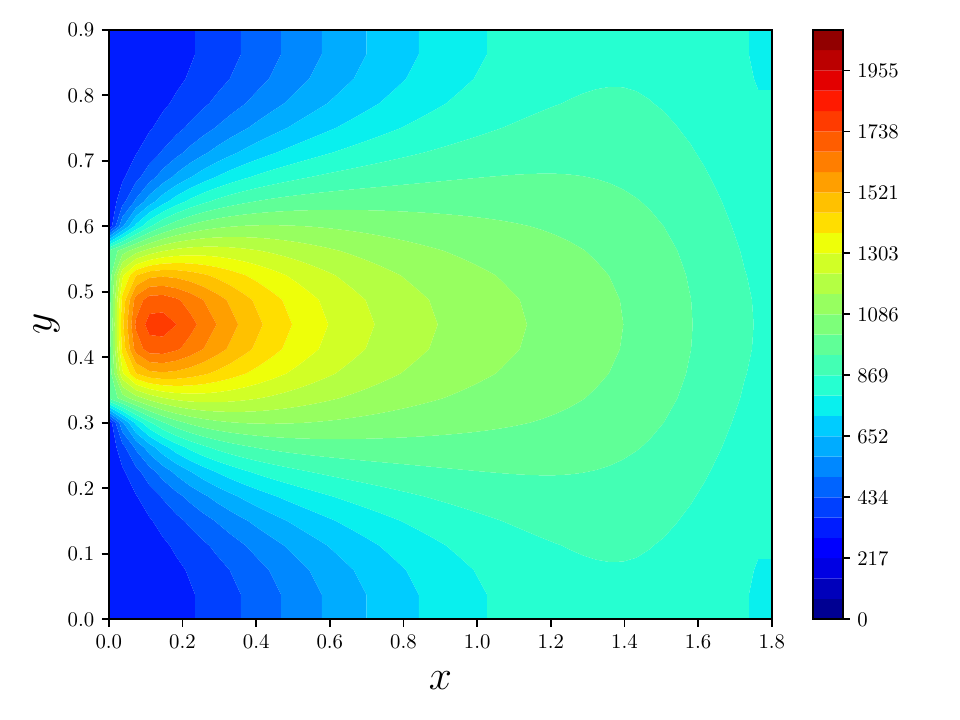}
\caption{$t=0.04$}
\end{subfigure}
\begin{subfigure}[t]{0.32\textwidth}
\includegraphics[trim={0cm 0cm 0cm 0cm},clip,width=1.0\linewidth]{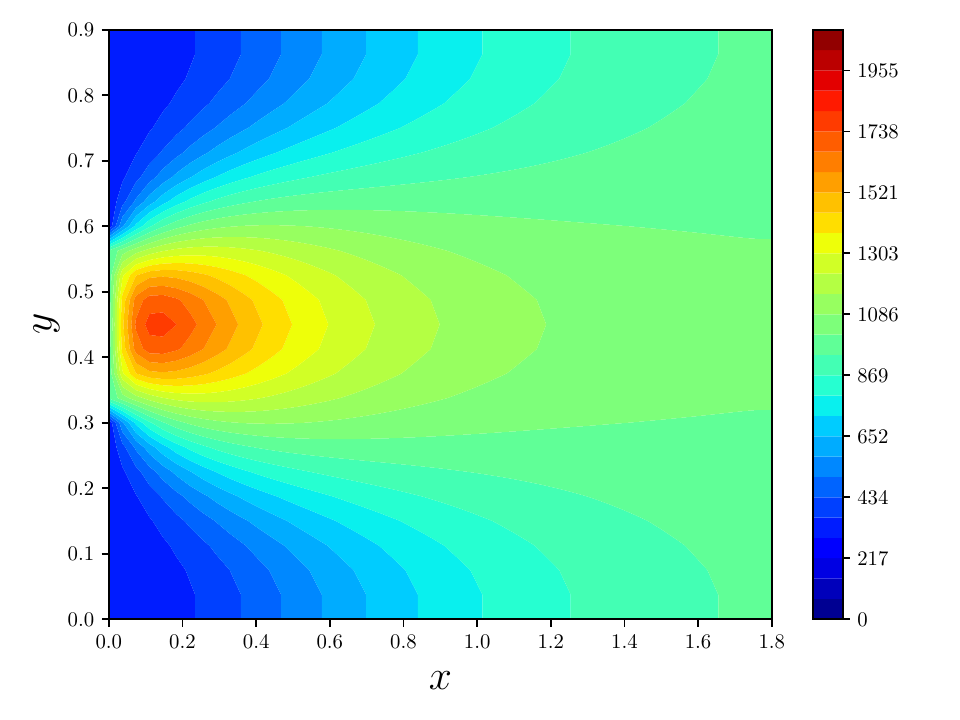}
\caption{$t=0.06$}
\end{subfigure}
\caption{Premixed $\HTwo$ air-flame example. FOM solutions for temperature at $t=0.02$ (left), $t=0.04$ (center), and $t=0.06$ (right) for $(\mu_1,\mu_2) = (5.55625 \times 10^{12},8.437583 \times 10^3)$.}
\label{fig:air_flame_foms}
\end{center}
\end{figure}

Figure~\ref{fig:air_flame_error} reports errors for the training set (left) and testing set (right) on the premixed $\HTwo$-air flame example. NN-OpInf-PSD-f outperforms NN-OpInf-NN by 5--10$\times$, and yields errors approximately an order of magnitude lower than those reported in Ref.~\cite{GaWaZh20}. Unlike the previous examples, the linear and quadratic P-OpInf models here are more accurate on the training set than NN-OpInf-PSD-f. While we note that a more accurate NN-OpInf-PSD-f ROM could be trained by tweaking the training settings, this result suggests that the polynomial terms dominate the dynamics and the use of an NN-OpInf model may be unnecessary. However, NN-OpInf-PSD-f yields slightly lower errors on the testing set than P-OpInf, demonstrating some benefit from the non-polynomial parameterization. At the highest basis dimension, NN-OpInf-PSD-f achieves roughly a $3\times$ lower error than the linear and quadratic P-OpInf models. We omit physical solution plots for these OpInf models as they are visually indistinguishable from the FOM.
 
\begin{figure}
\begin{center}
\begin{subfigure}[t]{0.49\textwidth}
\includegraphics[trim={0cm 0cm 0cm 0cm},clip,width=1.0\linewidth]{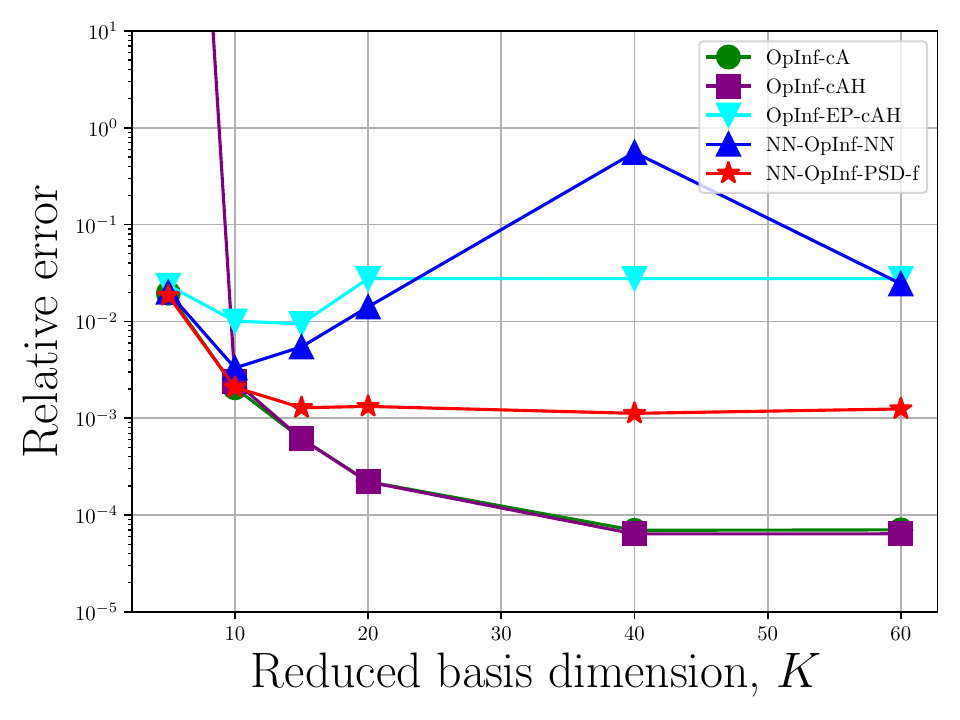}
\caption{Training parameter set, $\paramDomainTrain$}
\label{fig:ex1_training_error_converge_a}
\end{subfigure}
\begin{subfigure}[t]{0.49\textwidth}
\includegraphics[trim={0cm 0cm 0cm 0cm},clip,width=1.0\linewidth]{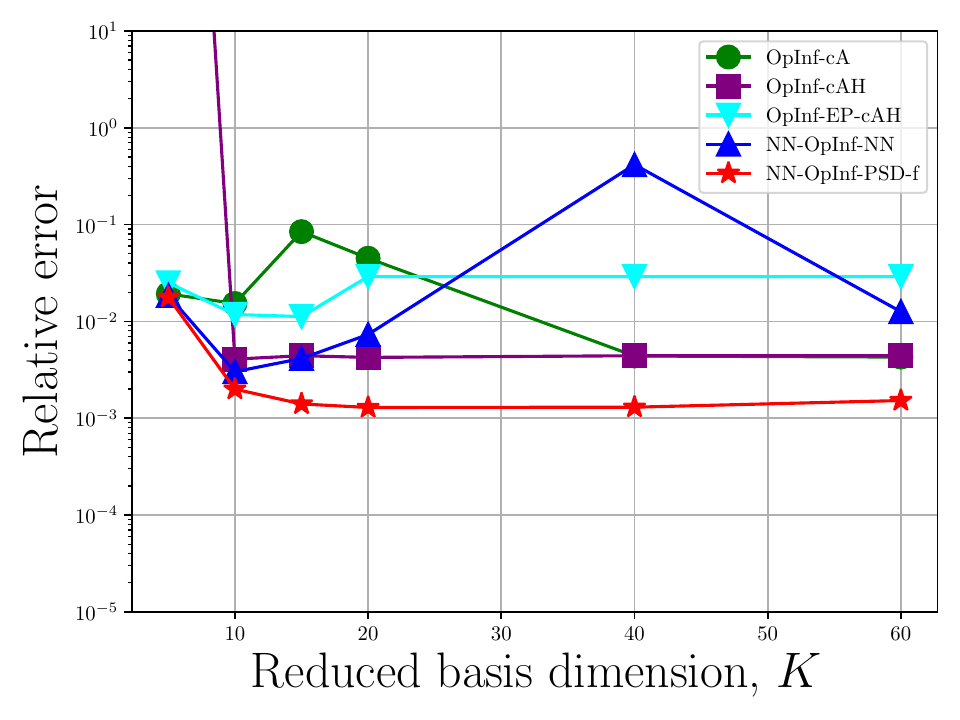}
\caption{Testing parameter set, $\paramDomainTest$}
\label{fig:ex1_training_error_converge_a}
\end{subfigure}

\caption{Premixed $\HTwo$ air-flame example. Relative errors for various OpInf formulations as a function of ROM dimension.}
\label{fig:air_flame_error}
\end{center}
\end{figure}

\subsection{3D hyper-elastic torsion problem}
The last application we consider arises in finite deformation solid mechanics, where the governing equations do not admit a polynomial structure for most material models. This example highlights the ability of NN-OpInf to infer a nonlinear OpInf model preserving the Lagrangian structure of the FOM. To this end, consider a body on the reference domain $\Omega \subset \RR{3}$ undergoing finite deformations described by $\boldsymbol x = \psi \left(\boldsymbol X, t \right)$, where $\psi : \Omega \times [0,T] \rightarrow W$ is the deformation map, $\boldsymbol X \in \Omega$ are the reference coordinates, $\boldsymbol x \in W$ are the deformed coordinates, and $W$ is the deformed domain. The dynamics are described by the Euler--Lagrange equations
\begin{equation}\label{eq:sd}
\rho_0 \ddot{\psi} = \nabla_{\boldsymbol X} \cdot P + \rho_0 \boldsymbol b ,
\end{equation} 
subject to the initial conditions,
\begin{equation} \label{eq:dynamic_elasticity_ics}
    \begin{array}{ll}
    \psi(\boldsymbol X, t_0) = \psi_0 \quad \text{ in } \Omega, & \quad\dot{\psi}(\boldsymbol X, t_0) = v_0 \quad \text{ in } \Omega. \\
      \end{array}
\end{equation}
In the above, $\rho_0$ is the mass density, $P$ is the first Piola--Kirchhoff stress, $\boldsymbol b$ is the body force per unit mass, and $\nabla_{\boldsymbol X}\cdot$ denotes the divergence with respect to the reference coordinates. We consider free boundary conditions and no body force ($\boldsymbol b = \boldsymbol 0$) in the example that follows. A Galerkin finite element discretization of Eq.~\eqref{eq:sd} yields the corresponding discrete system
$$
\mass \ddot\state   = \boldsymbol f(\state),
$$
where $\mass$ is the mass matrix and $\boldsymbol f$ are the internal forces. In terms of the elastic potential energy $V$, the internal forces satisfy $\boldsymbol f(\state)=-V'(\state)$.
We consider a Neohookean material model with Young's modulus $1 \times 10^9$ Pa, Poisson's ratio $0.25$, and density $\rho_0 = 1000$ kg/m$^3$. We emphasize that this material model gives rise to dynamics that are highly nonlinear, with nonlinearities that are far more complex than polynomial nonlinearities; see Appendix A.3 of \cite{TePaGr26} for details.
The problem setup models the distortion of a free beam subject to an initial velocity, which results in a high-degree of torsion; a similar example was considered in Refs.~\cite{Mota:2022,TePaGr26}. The beam is of size $0.05 \times 1.0 \times 0.05$ subject to an initial velocity profile $v_x = \alpha y z$, $v_y = \beta x z$, where $\alpha$ and $\beta$ control its magnitude. The system is discretized using (first-order) piecewise-linear polynomial elements of size $\Delta x = \Delta y = \Delta z = 0.01$ m, while time integration uses the (second-order) Newmark-$\beta$ scheme with $\Delta t = 1 \times 10^{-5}$ s, $\beta = 0.25$, and $\gamma=0.5$. 

The discretized governing equations extremize the mechanical Lagrangian,
$$\mathcal{L} ( \state,\stateDot ) = T(\stateDot) - V(\state) \coloneqq \frac{1}{2}\stateDot^\intercal \mass \stateDot - V(\state),$$
where $T(\stateDot) = \frac{1}{2} \stateDot^\intercal \mass \stateDot$ is the kinetic energy and $V(\state)$ is the potential energy. To construct an NN-OpInf ROM that maintains the Lagrangian structure of the FOM, we parameterize the right-hand side with the SPSD potential operator,
\begin{equation}\label{eq:nnopinf_lagrangian}
\ddot{\reducedState} = -\frac{\partial}{\partial \reducedState} \left[ \frac{1}{2}\reducedState^\intercal\reducedSpdMatrixLowerNN(\reducedState,\params;\weights) \reducedSpdMatrixLowerNN^\intercal(\reducedState,\params;\weights)\reducedState \right].
\end{equation}
We refer to this formulation as the NN-OpInf-SPSD-Potential ROM, and note that the resulting ROM satisfies the Euler--Lagrange equations
$$ \ddot{\reducedState}  + \hat{V}'(\reducedState) = \mathbf{0},$$
associated with the reduced Lagrangian
$$\hat{\mathcal{L}} ( \reducedState,\reducedStateDot ) = \frac{1}{2}\dot\reducedState^\intercal\dot\reducedState - \hat{V}(\reducedState), \qquad \hat{V}(\reducedState) = \frac{1}{2}\reducedState^\intercal\reducedSpdMatrixLowerNN(\reducedState,\params;\weights) \reducedSpdMatrixLowerNN^\intercal(\reducedState,\params;\weights)\reducedState.$$

We additionally consider two Lagrangian structure-preserving baselines. First, we implement the Lagrangian operator inference (L-OpInf) formulation of Ref.~\cite{ShKa24}, specialized to zero damping; we denote our implementation NN-OpInf-Linear-Lagrangian. Because we construct an $\mass$-orthonormal ROM basis, the reduced mass matrix is the identity and the model reduces to
\begin{equation}\label{eq:lopinf_torsion}
    \ddot{\reducedState} = -\widehat{\mathbf{K}}\reducedState,
    \qquad
    \widehat{\mathbf{K}} = \widehat{\mathbf{L}}\widehat{\mathbf{L}}^\intercal + \epsilon\mathbf{I},
    \qquad \epsilon > 0.
\end{equation}
We implement this model using an NN-OpInf SPSD operator with no state dependence and no hidden layers. Consequently, $\widehat{\mathbf{K}}$ is a constant SPD matrix, and the right hand side in Eq.~\eqref{eq:lopinf_torsion} is the negative gradient of the quadratic potential $\frac{1}{2}\reducedState^\intercal\widehat{\mathbf{K}}\reducedState$. Allowing the matrix to depend on the state would not, in general, produce the gradient of a scalar potential.

Second, we implement an analogue of the LOpInf-SpML formulation from Ref.~\cite{ShNaTo24} suited for the present case; we denote our implementation NN-OpInf-SpML-Lagrangian. This model augments the linear stiffness operator with a learned nonlinear potential $U_{\mathrm{nl}}$,
\begin{equation}\label{eq:lp_opinf_torsion}
    \ddot{\reducedState}
    = -\widehat{\mathbf{K}}\reducedState
      -\nabla_{\reducedState}U_{\mathrm{nl}}(\reducedState).
\end{equation}
Within NN-OpInf, Eq.~\eqref{eq:lp_opinf_torsion} is implemented as the additive composition of the constant SPD stiffness operator from Eq.~\eqref{eq:lopinf_torsion} and a scalar-potential neural operator whose output is $-\nabla_{\reducedState}U_{\mathrm{nl}}$. This implementation has the same linear-stiffness-plus-nonlinear-potential model form as LOpInf-SpML, but it does not reproduce the training procedure of Ref.~\cite{ShNaTo24}. That method first solves a constrained linear least-squares problem for symmetric positive-definite stiffness and damping operators. Holding these operators constant, it then trains polynomial-augmented neural representations of corrections to the kinetic energy, potential energy, and dissipation.\footnote{We also note that the second-stage optimization includes penalty constraints that enforce positivity of the total reduced kinetic energy and nonnegativity of the reduced dissipation function. The present example contains neither damping nor dissipation and uses an identity reduced mass matrix, so we do not learn kinetic-energy, mass-matrix, or dissipation corrections.} We jointly train the constant stiffness and a generic scalar neural potential in a single NN-OpInf optimization, enforcing only that the stiffness is SPD by its parameterization and that the nonlinear force is an exact potential gradient. Our NN-OpInf-SpML-Lagrangian results should therefore be interpreted as those of an analogous model implemented within NN-OpInf, rather than a reproduction of the full two-stage LOpInf-SpML algorithm.

We compare NN-OpInf-SPSD-Potential, NN-OpInf-Linear-Lagrangian, and NN-OpInf-SpML-Lagrangian with the P-OpInf-A, P-OpInf-AH, and NN-OpInf-NN ROMs. Due to the strong nonlinearities in this problem, we consider a slightly different training configuration than the previous examples. For our P-OpInf ROMs, we consider different regularization parameters for the linear and quadratic operators, and search over the grid
\[
\mathcal{G}
=
\lambda_{A} \times \lambda_H,
\]
where
\[
\lambda_{A} = \lambda_H  
=
\{0\}
\cup
\left\{10^k : k\in\{-5,-4,-3,-2,-1,0,1,2,3,4\}\right\}.
\]
In the NN-OpInf ROMs we set the number of neurons per layer to be twice the size of the ROM dimension, $n_{\mathrm{n}} = 2K$, and we use an ensemble size of $L=5$.

Two configurations are considered:
\begin{enumerate}
\item Reproductive ROM: We set $\alpha = \beta = 7500$, collecting snapshots and making predictions over $t \in [0,2.5 \times 10^{-3}]$. Snapshots are collected at every FOM time step for a total of $250$ snapshots.
\item Future-state prediction: Using the same configuration as above, we predict over the longer interval $t \in [0,5 \times 10^{-3}]$. 
\end{enumerate}

    Figure~\ref{fig:torsion_error_convergence} shows results for relative error versus basis dimension for the reproductive configuration (left) and future-state configuration (right). We observe that the NN-OpInf-SPSD-Potential model, which preserves the Lagrangian structure of the FOM, is substantially more accurate than the standard P-OpInf models, as well as NN-OpInf-Linear-Lagrangian and NN-OpInf-SpML-Lagrangian. Both NN-OpInf-Linear-Lagrangian and NN-OpInf-SpML-Lagrangian produce relative errors of order one across most basis dimensions, indicating that neither a constant linear stiffness operator nor its combination with the learned nonlinear potential under our one-stage training procedure adequately represents this problem. In particular, we observe that the NN-OpInf-Linear-Lagrangian model yields minimal dynamics for higher basis dimensions, indicating that the imposed structure is too restrictive for the linear operator. Additionally, we note that because our implementation of NN-OpInf-SpML-Lagrangian omits the two-stage optimization and second-stage constraints of Ref.~\cite{ShNaTo24}, these results should be interpreted as an assessment of the model-form proposed by LOpInf-SpML, rather than an assessment of the complete LOpInf-SpML method. The vanilla NN-OpInf-NN model is unstable. Figure~\ref{fig:torsion_paraview_t0025} shows the physical-space displacement solutions for the reproductive configuration, while Figure~\ref{fig:torsion_paraview_t0050} shows results for the future-state prediction configuration. We observe that the NN-OpInf-SPSD-Potential operator faithfully approximates the FOM trajectory for the reproductive configuration and, while it differs slightly in phase from the FOM in the future-state configuration, still provides a reasonable solution. While the P-OpInf-AH model provides a reasonable solution at the end of the reproductive regime, $t=2.5 \times 10^{-3}$, it is associated with a high error and yields unphysical predictions in the future-state configuration. The linear P-OpInf-A model performs poorly for all configurations.

\begin{figure}
\begin{center}
\begin{subfigure}[t]{0.49\textwidth}
\includegraphics[trim={0cm 0cm 0cm 0cm},clip,width=1.0\linewidth]{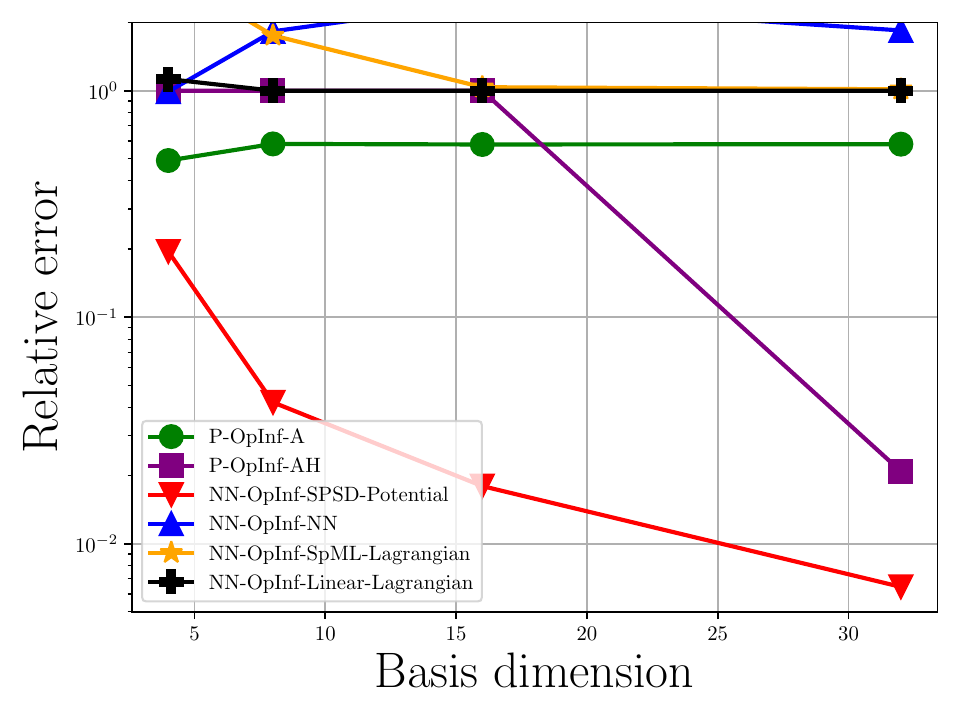}
\caption{Reproductive configuration.}
\end{subfigure}
\begin{subfigure}[t]{0.49\textwidth}
\includegraphics[trim={0cm 0cm 0cm 0cm},clip,width=1.0\linewidth]{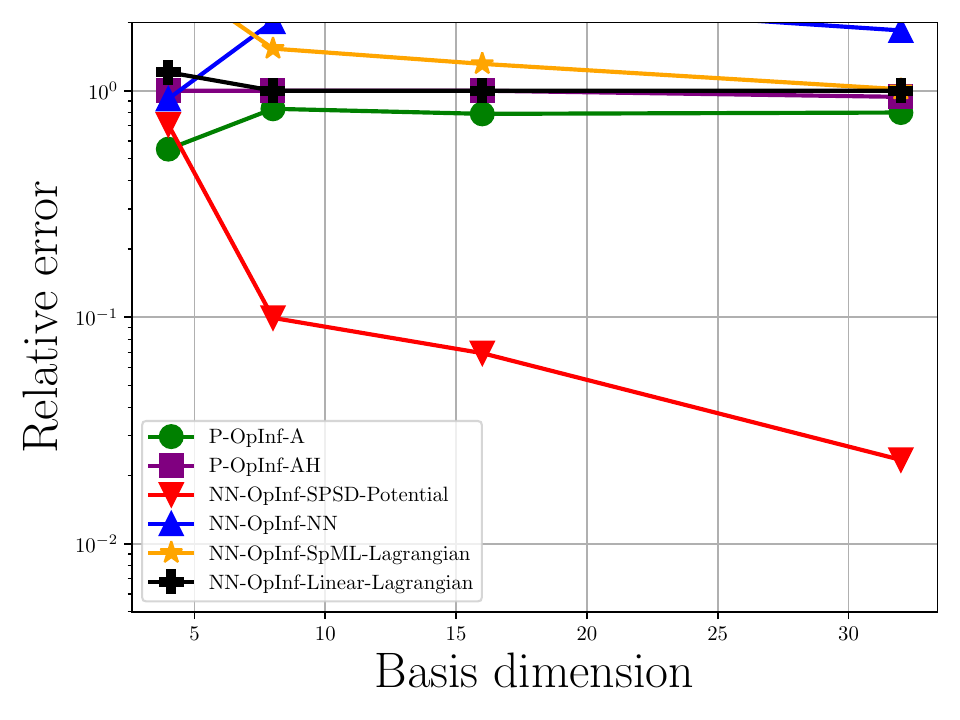}
\caption{Future-state prediction.}
\end{subfigure}
\caption{Torsion problem. Relative error versus basis dimension.}
\label{fig:torsion_error_convergence}
\end{center}
\end{figure}

\begin{figure}
\begin{center}
\begin{subfigure}[t]{0.13\textwidth}
\includegraphics[trim={25cm 0cm 25cm 0cm},clip,width=1.0\linewidth]{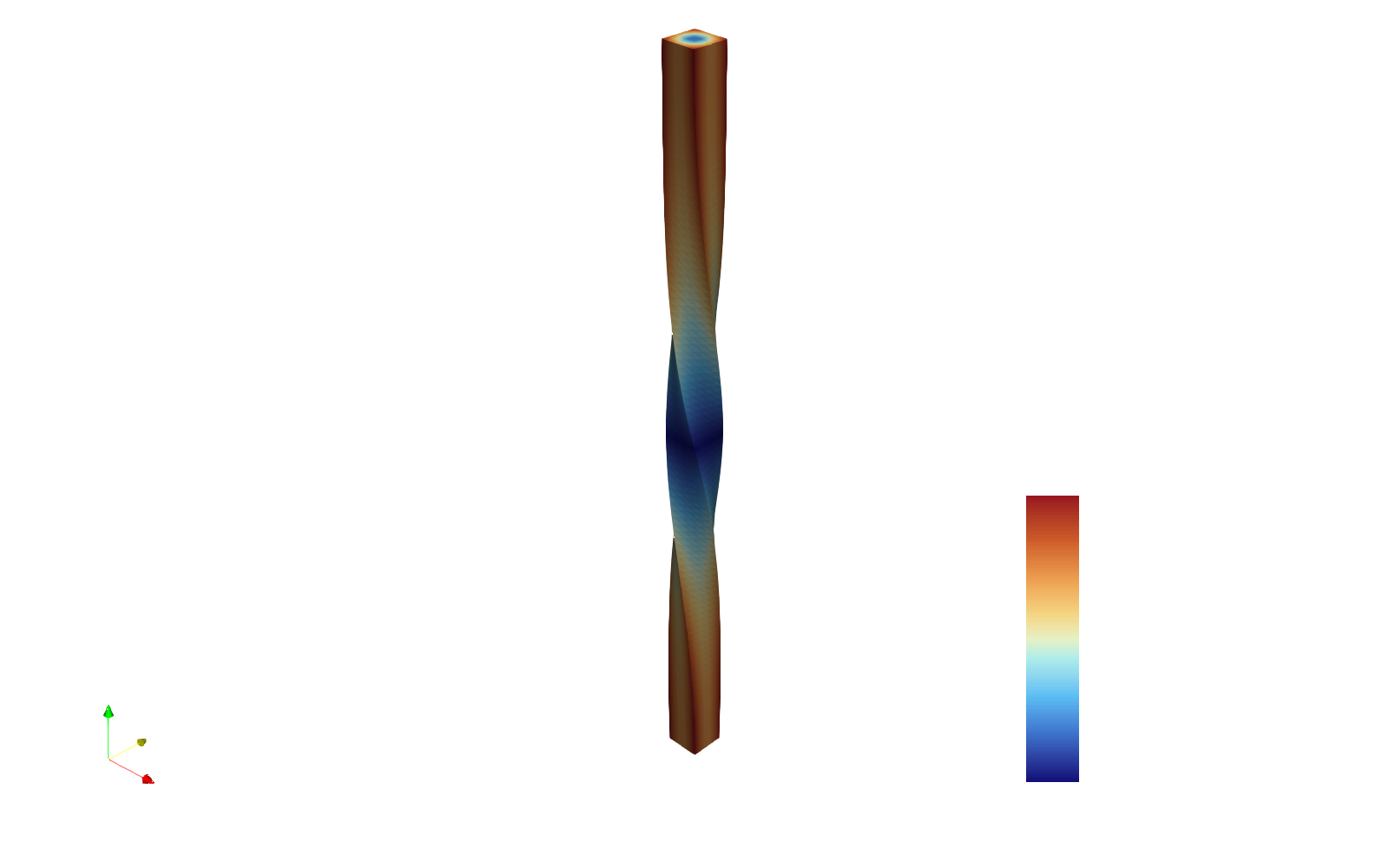}
\caption{FOM}
\end{subfigure}
\begin{subfigure}[t]{0.13\textwidth}
\includegraphics[trim={25cm 0cm 25cm 0cm},clip,width=1.0\linewidth]{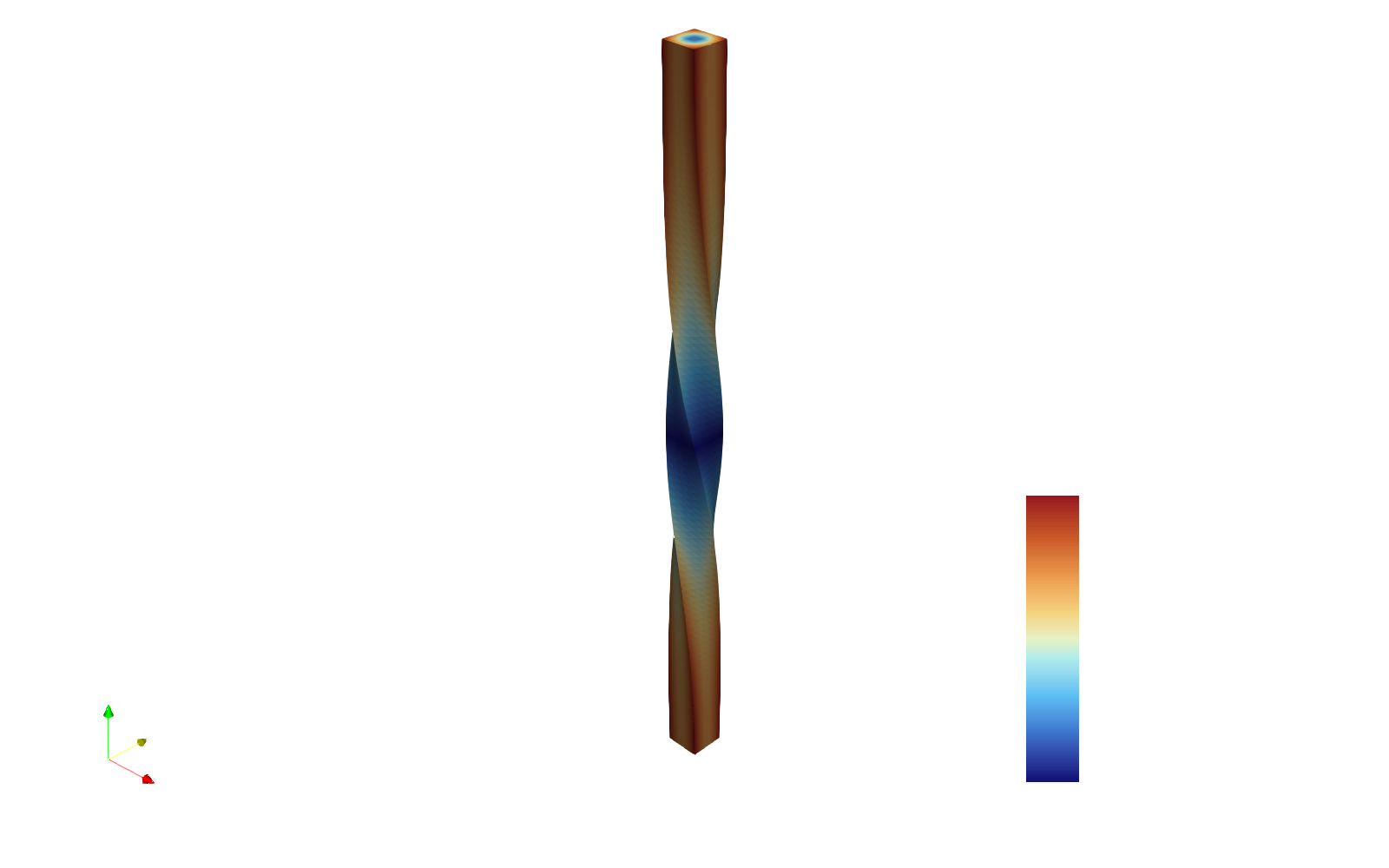}
\caption{NN-OpInf-SPSD-Potential}
\end{subfigure}
\begin{subfigure}[t]{0.13\textwidth}
\includegraphics[trim={25cm 0cm 25cm 0cm},clip,width=1.0\linewidth]{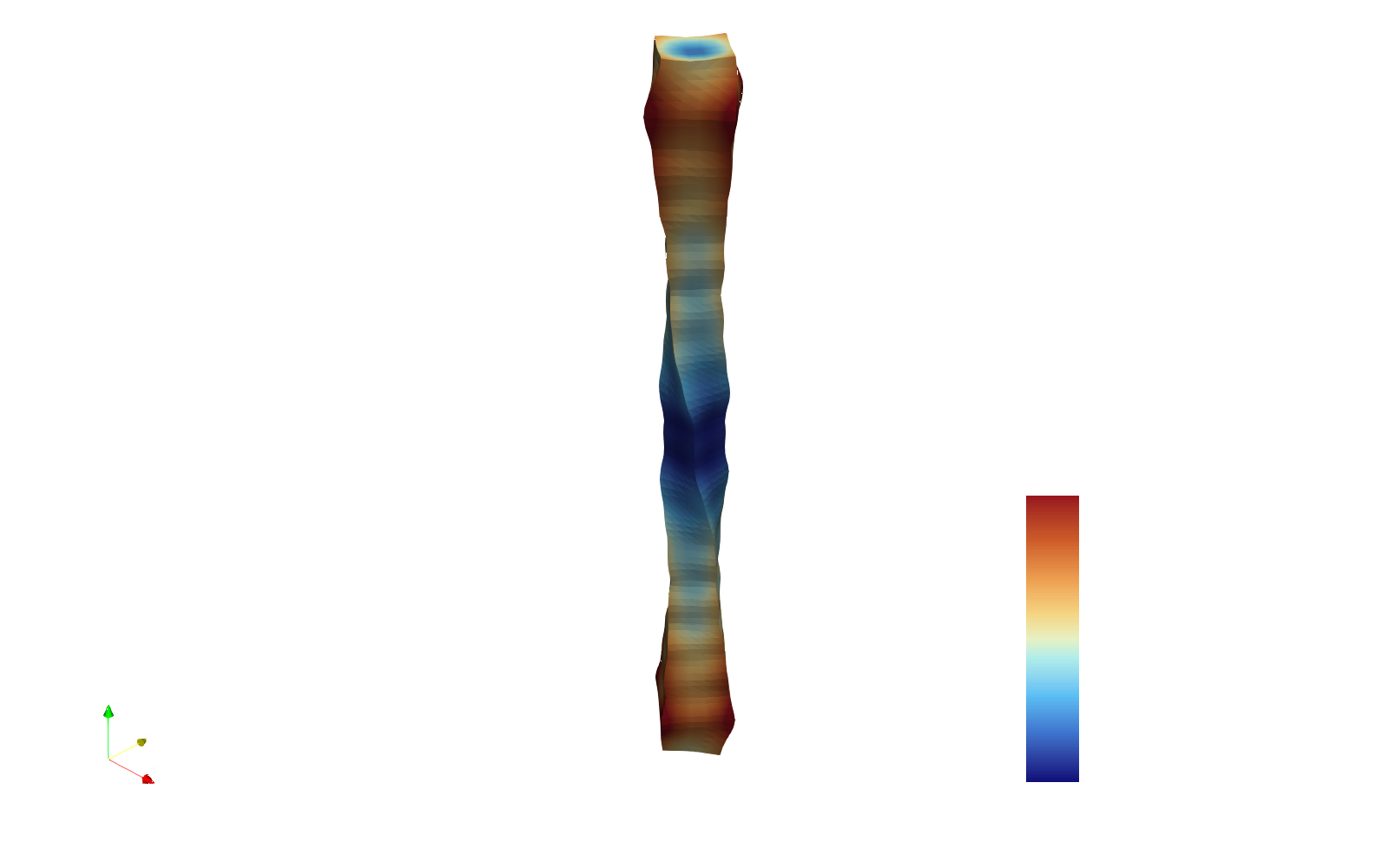}
\caption{P-OpInf-A}
\end{subfigure}
\begin{subfigure}[t]{0.13\textwidth}
\includegraphics[trim={25cm 0cm 25cm 0cm},clip,width=1.0\linewidth]{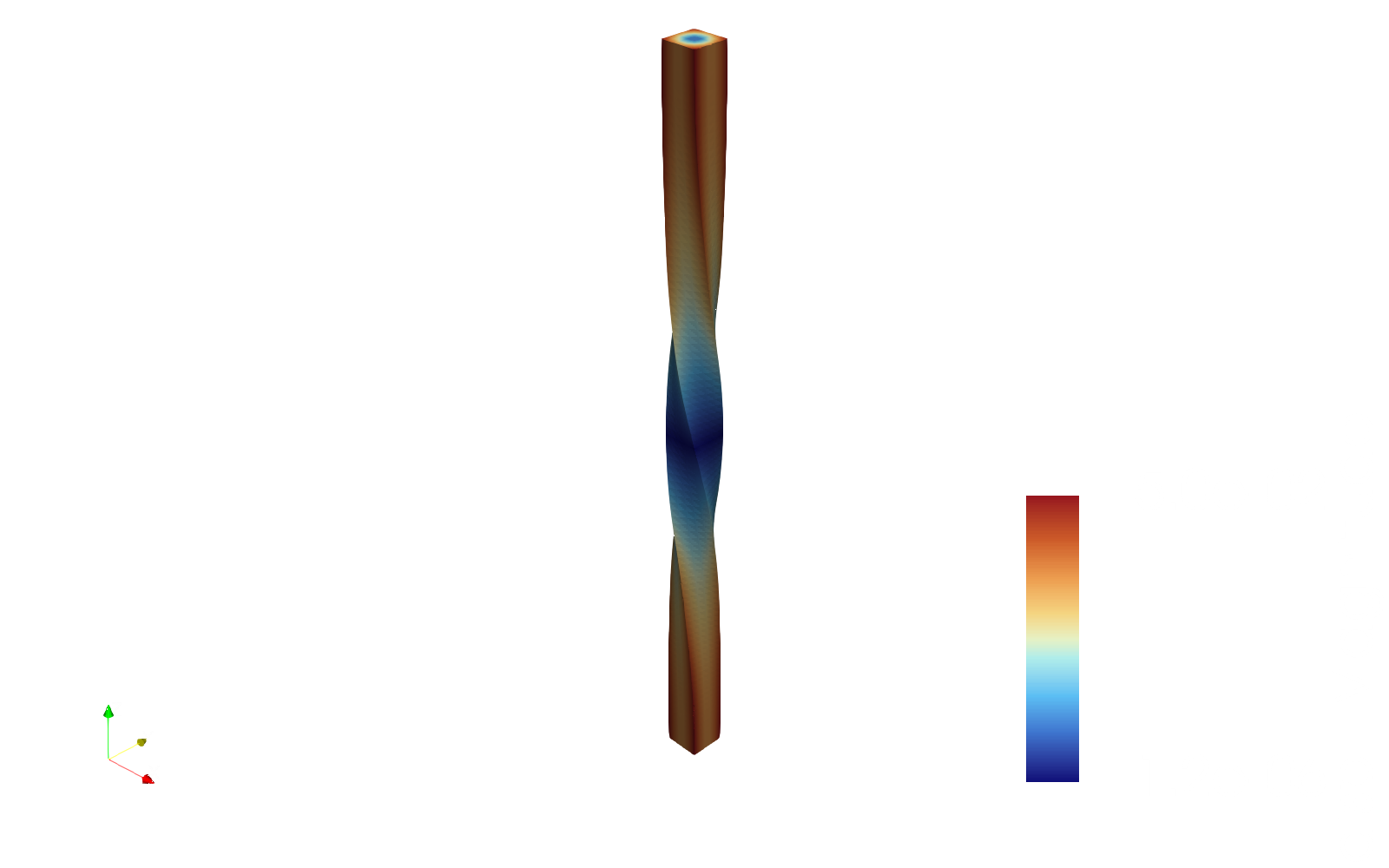}
\caption{P-OpInf-AH}
\end{subfigure}
\begin{subfigure}[t]{0.13\textwidth}
\includegraphics[trim={25cm 0cm 25cm 0cm},clip,width=1.0\linewidth]{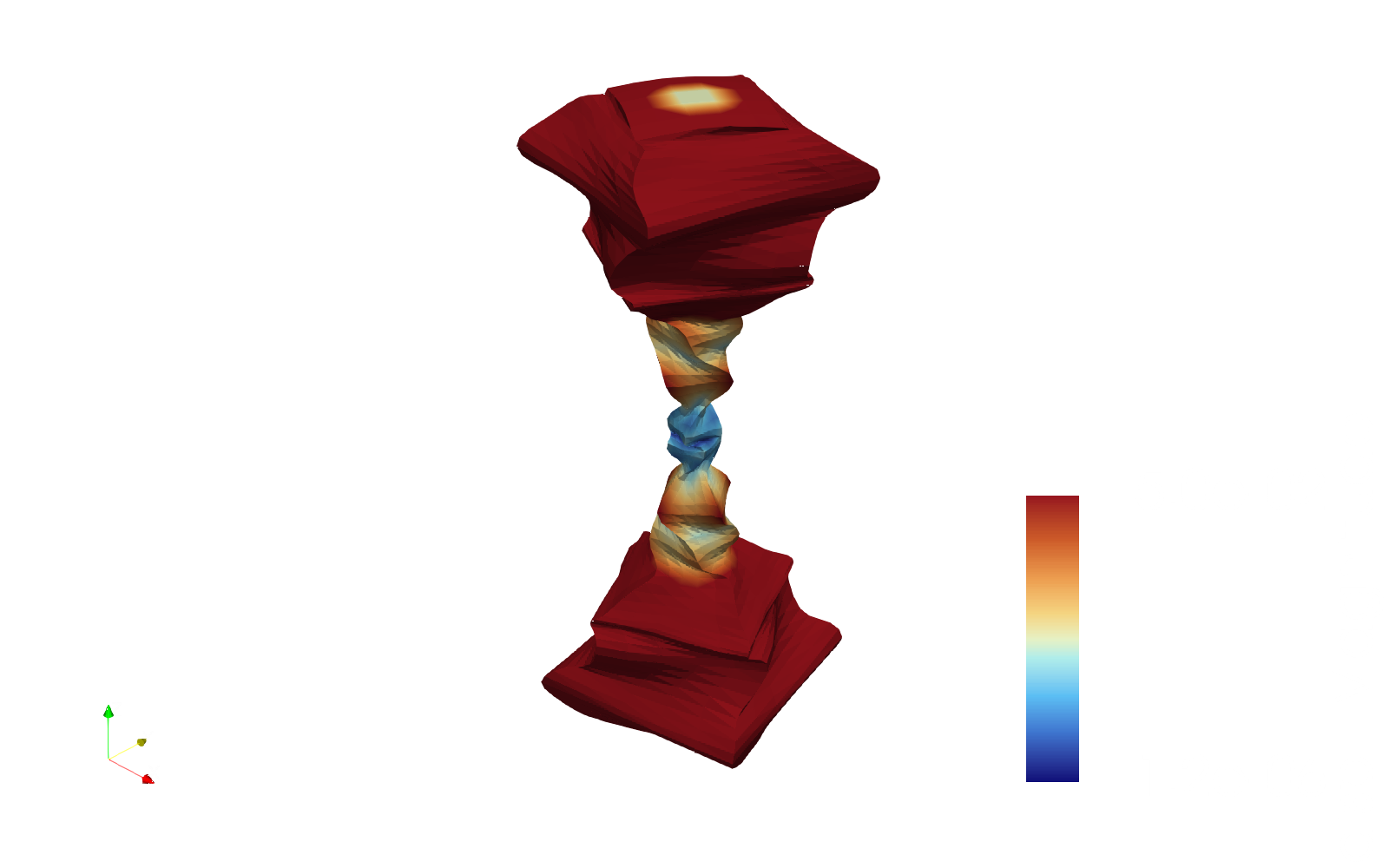}
\caption{NN-OpInf-NN}
\end{subfigure}
\begin{subfigure}[t]{0.13\textwidth}
\includegraphics[trim={25cm 0cm 25cm 0cm},clip,width=1.0\linewidth]{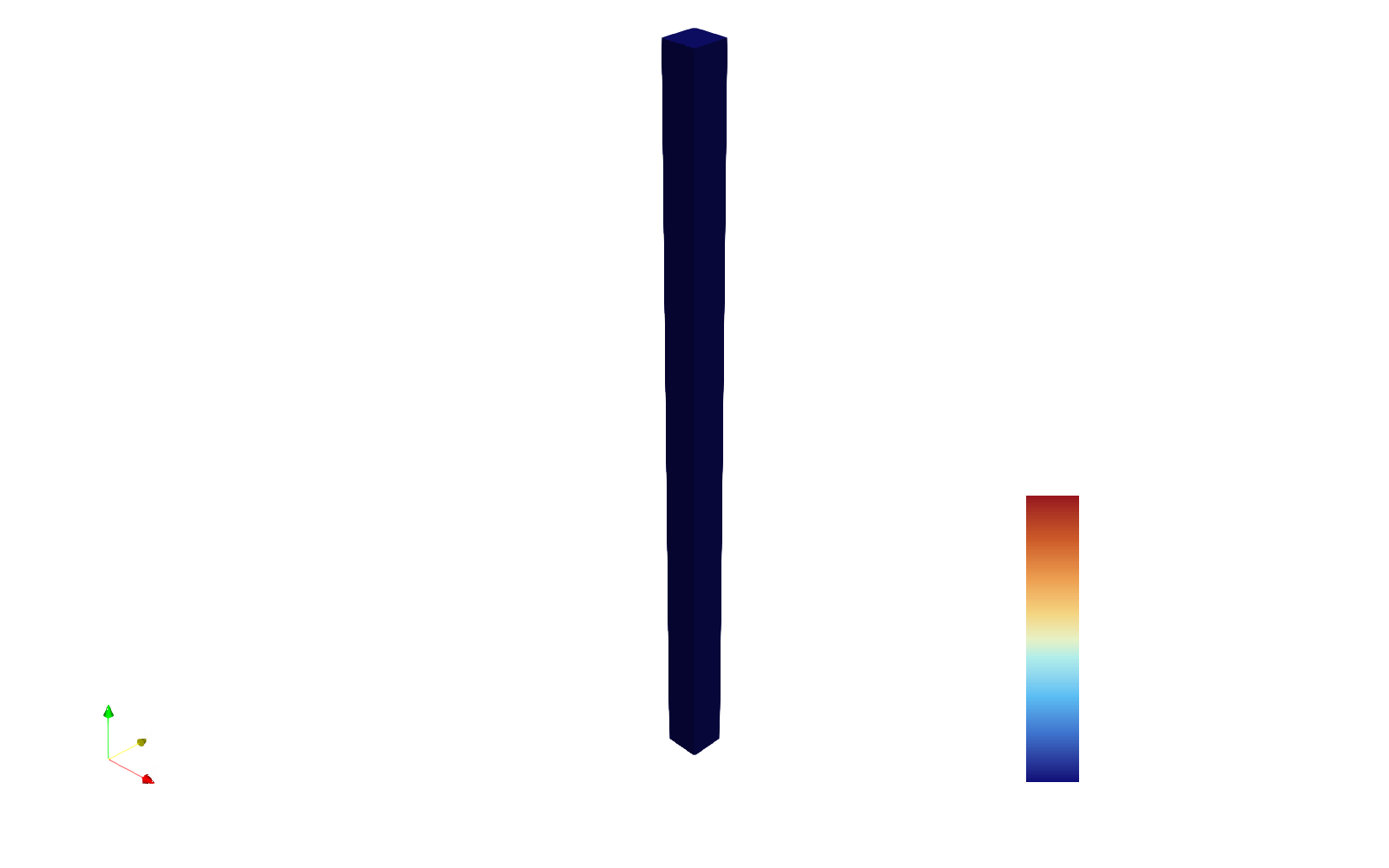}
\caption{NN-OpInf-Linear-Lagrangian}
\end{subfigure}
\begin{subfigure}[t]{0.13\textwidth}
\includegraphics[trim={25cm 0cm 25cm 0cm},clip,width=1.0\linewidth]{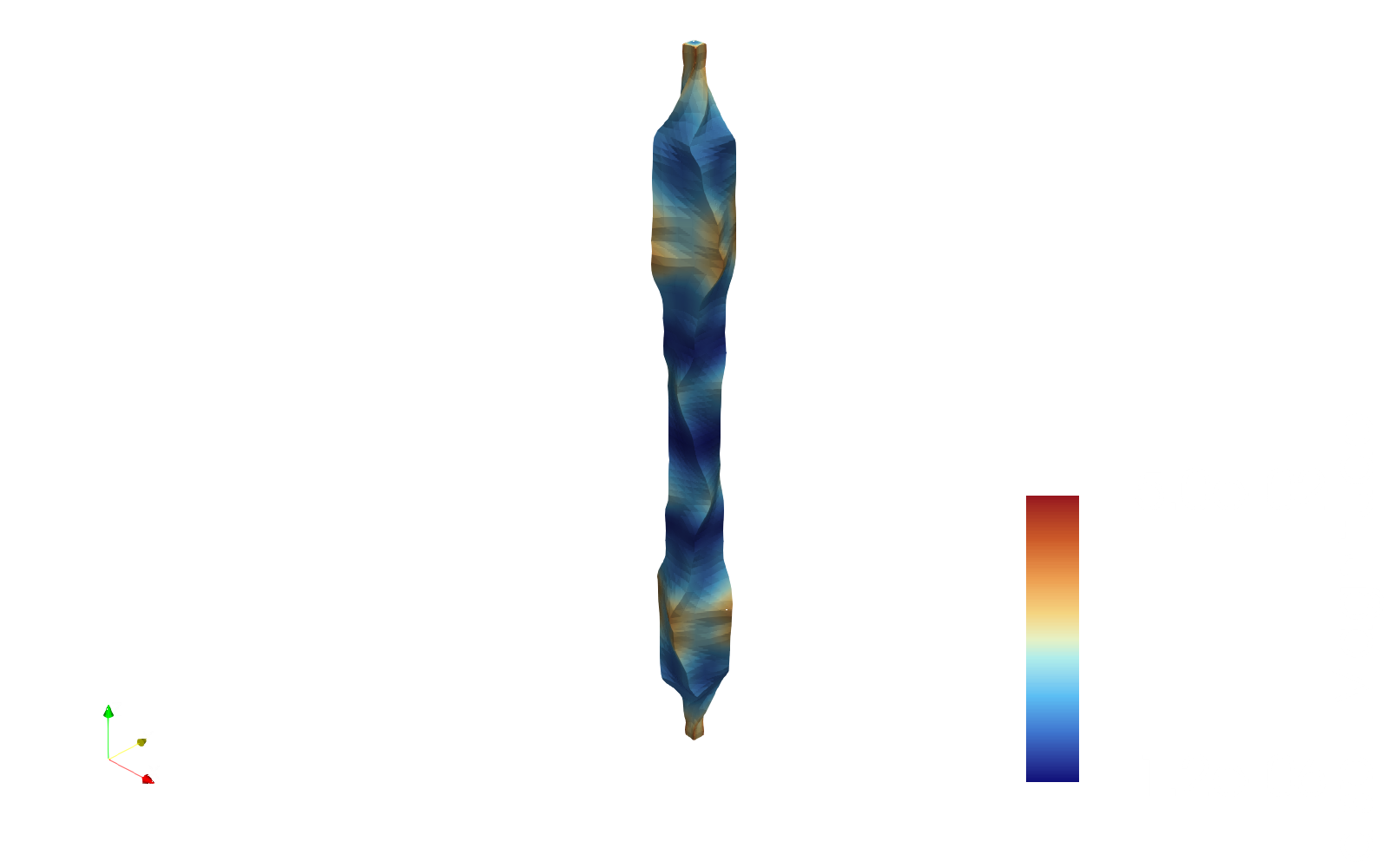}
\caption{NN-OpInf-SpML-Lagrangian}
\end{subfigure}
    \caption{Torsion problem. Displacement solutions at $t=2.5\times 10^{-3}$ s ($K=32$). The coloring is the displacement magnitude.}
\label{fig:torsion_paraview_t0025}
\end{center}
\end{figure}

\begin{figure}
\begin{center}
\begin{subfigure}[t]{0.13\textwidth}
\includegraphics[trim={25cm 0cm 25cm 0cm},clip,width=1.0\linewidth]{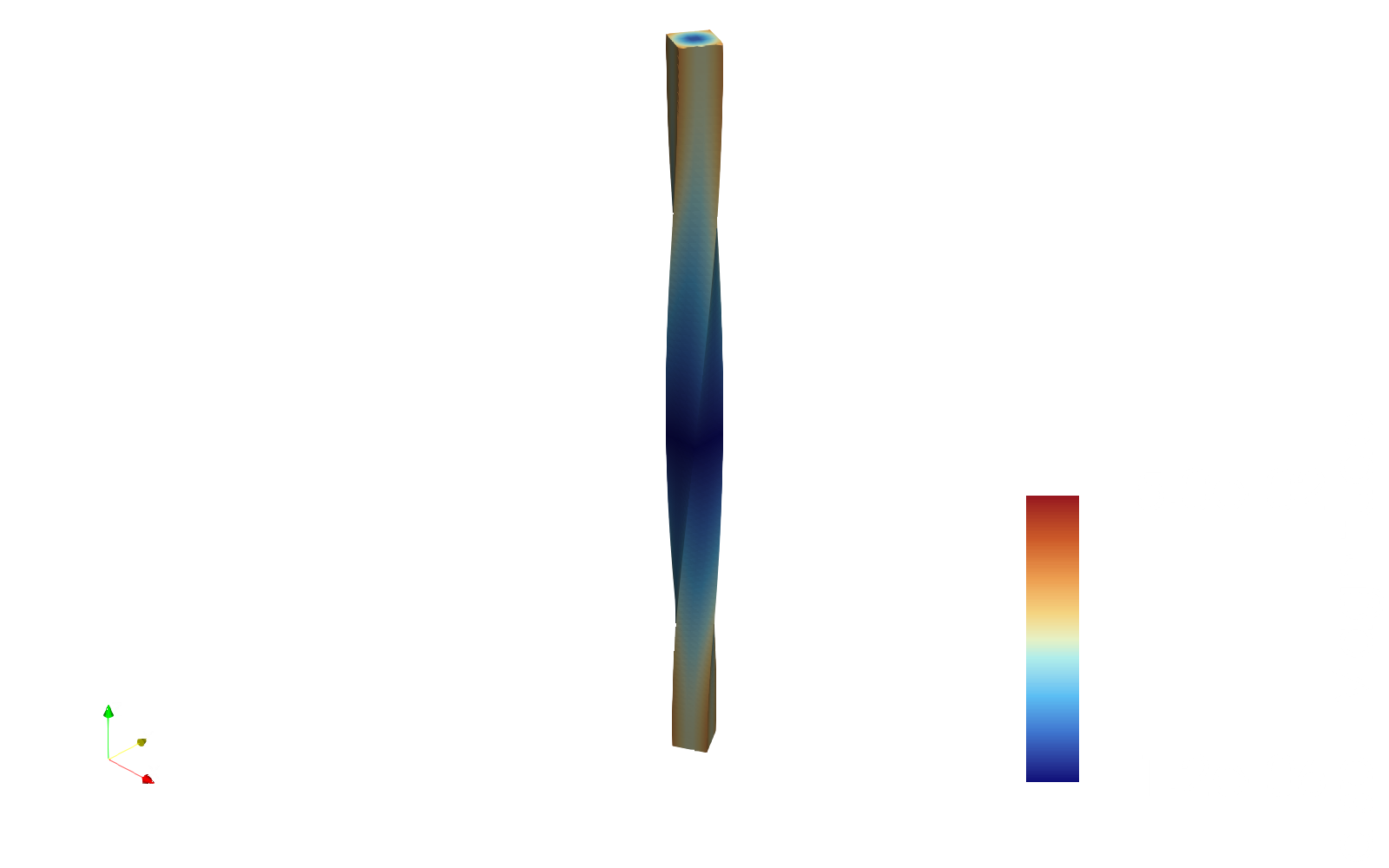}
\caption{FOM}
\end{subfigure}
\begin{subfigure}[t]{0.13\textwidth}
\includegraphics[trim={25cm 0cm 25cm 0cm},clip,width=1.0\linewidth]{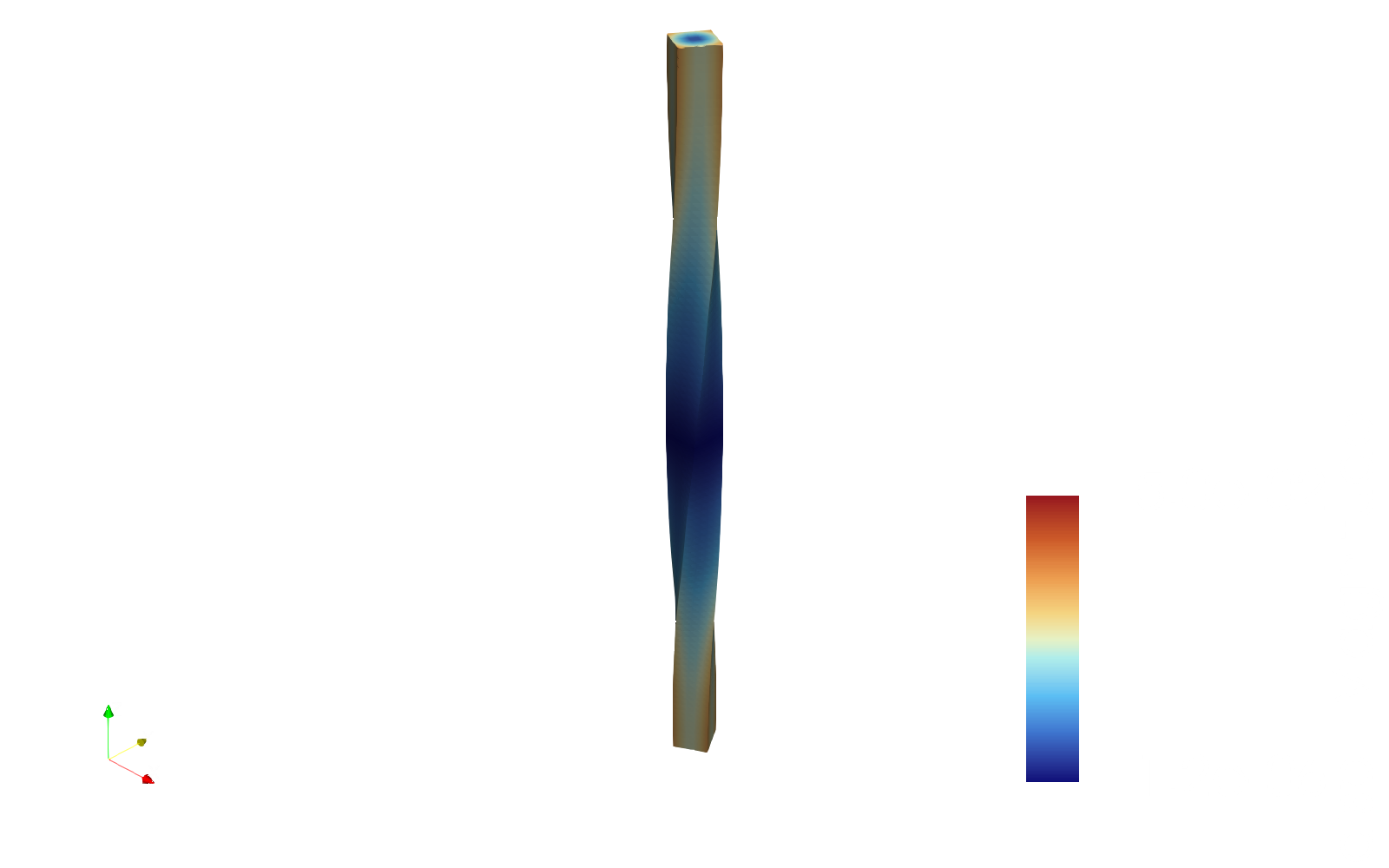}
\caption{NN-OpInf-SPSD-Potential}
\end{subfigure}
\begin{subfigure}[t]{0.13\textwidth}
\includegraphics[trim={25cm 0cm 25cm 0cm},clip,width=1.0\linewidth]{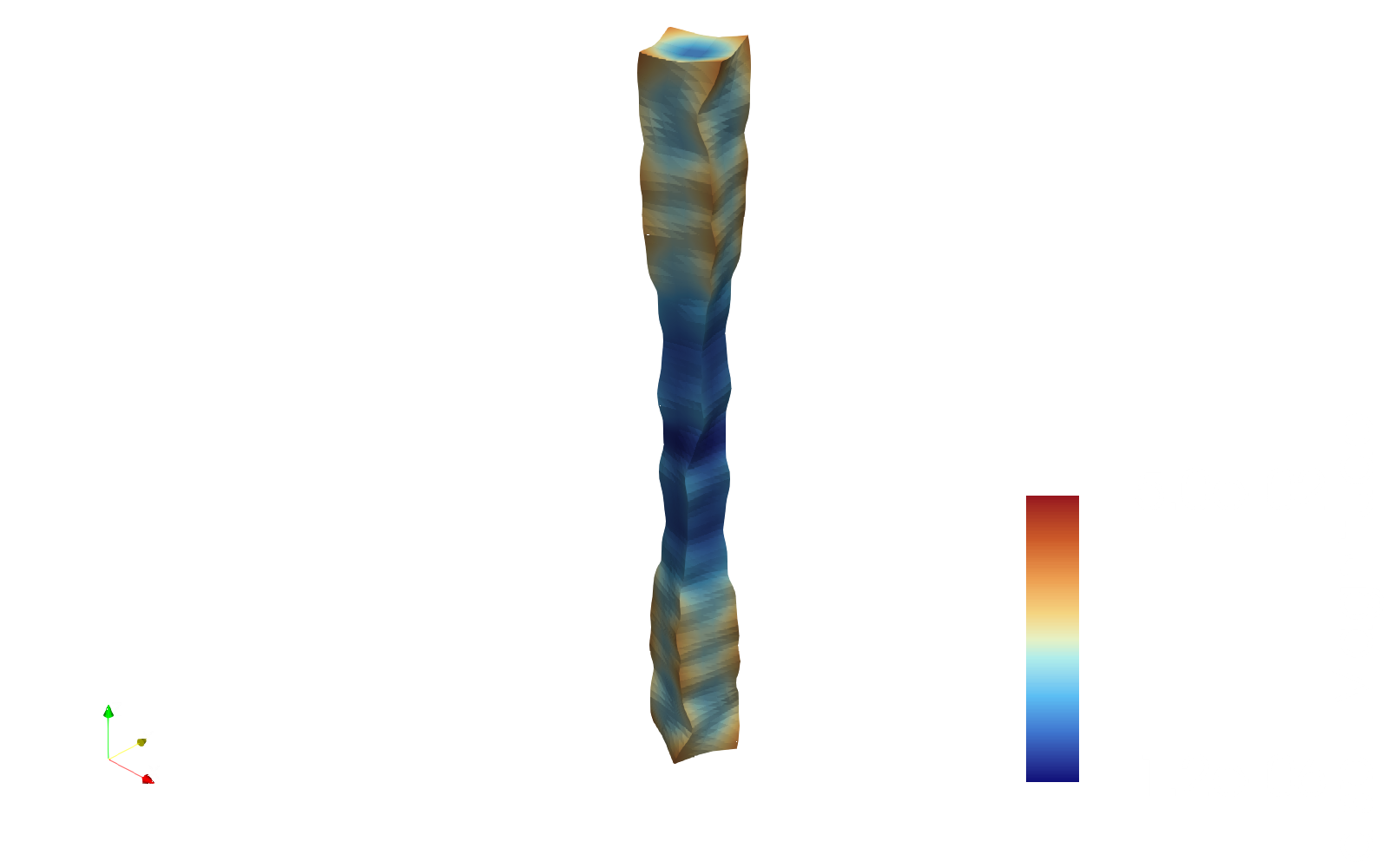}
\caption{P-OpInf-A}
\end{subfigure}
\begin{subfigure}[t]{0.13\textwidth}
\includegraphics[trim={25cm 0cm 25cm 0cm},clip,width=1.0\linewidth]{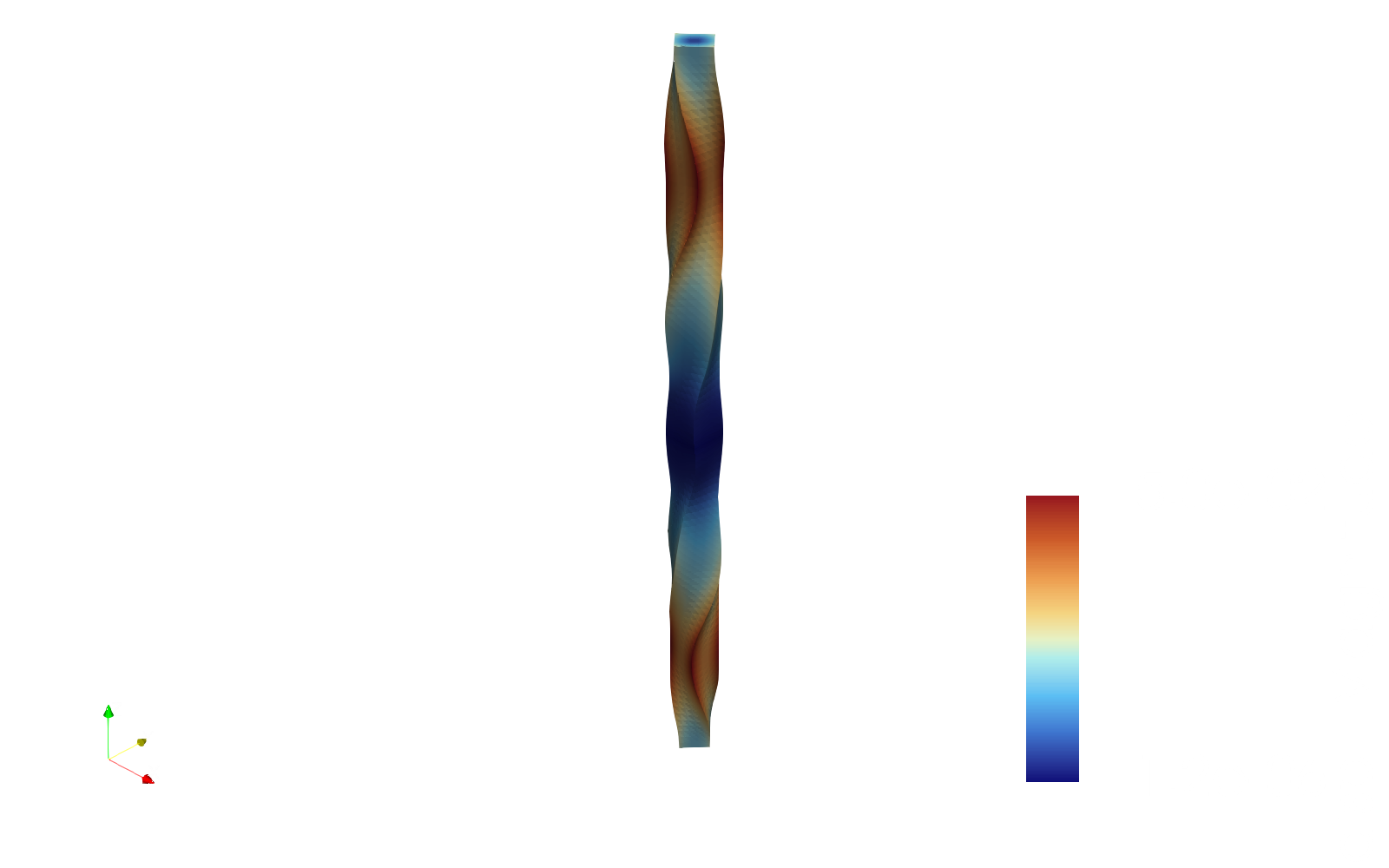}
\caption{P-OpInf-AH}
\end{subfigure}
\begin{subfigure}[t]{0.13\textwidth}
\includegraphics[trim={25cm 0cm 25cm 0cm},clip,width=1.0\linewidth]{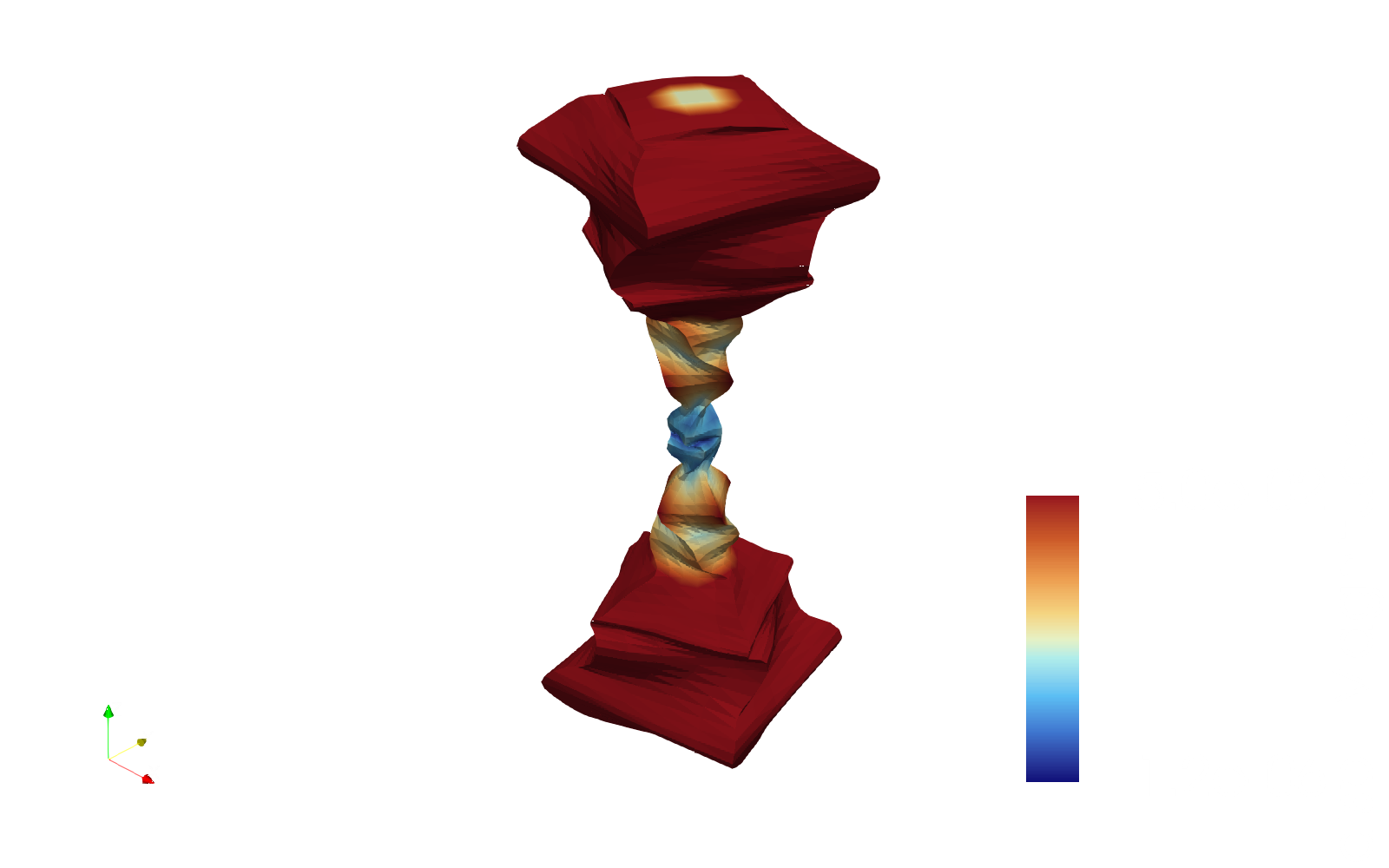}
\caption{NN-OpInf-NN}
\end{subfigure}
\begin{subfigure}[t]{0.13\textwidth}
\includegraphics[trim={25cm 0cm 25cm 0cm},clip,width=1.0\linewidth]{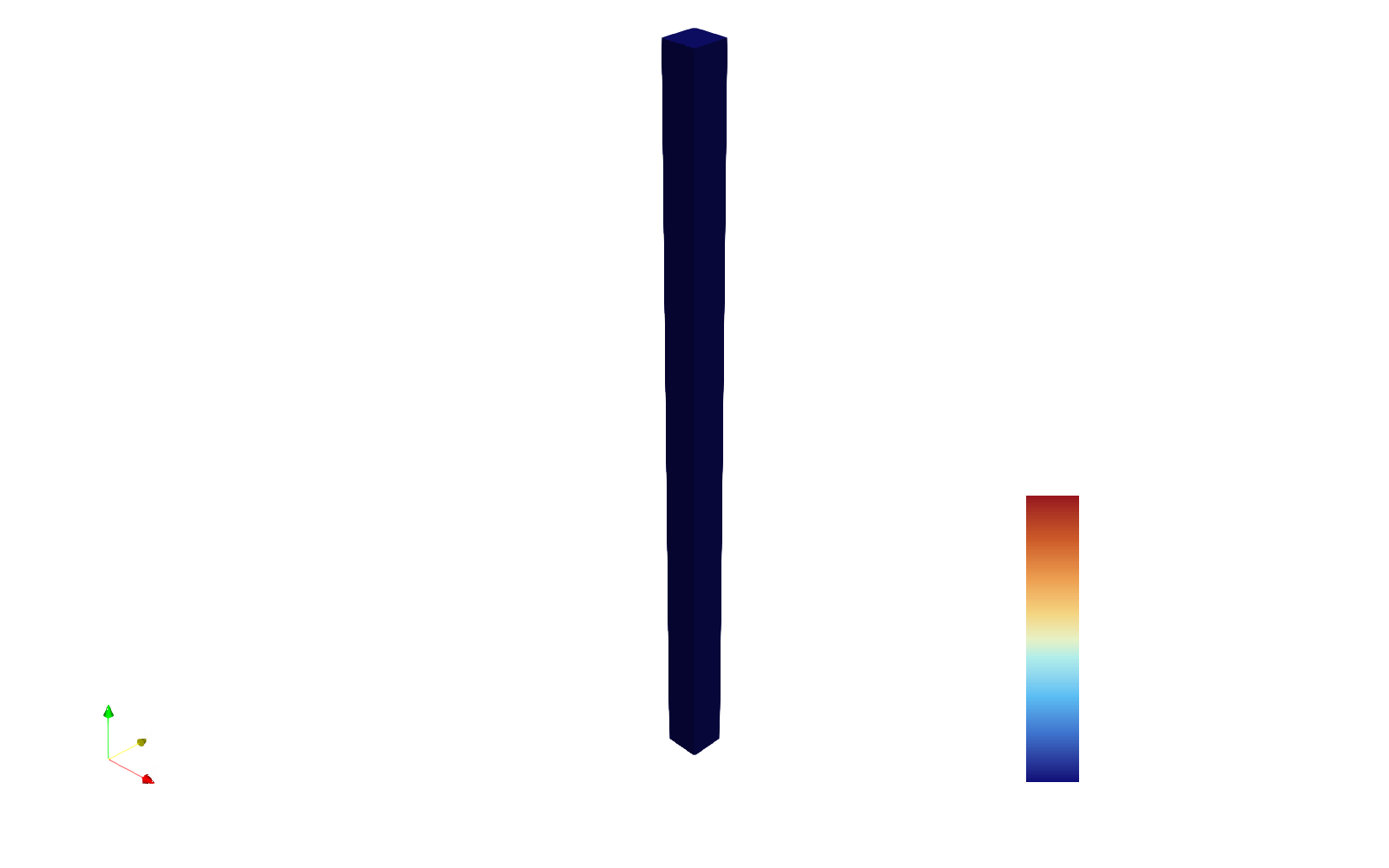}
\caption{NN-OpInf-Linear-Lagrangian}
\end{subfigure}
\begin{subfigure}[t]{0.13\textwidth}
\includegraphics[trim={25cm 0cm 25cm 0cm},clip,width=1.0\linewidth]{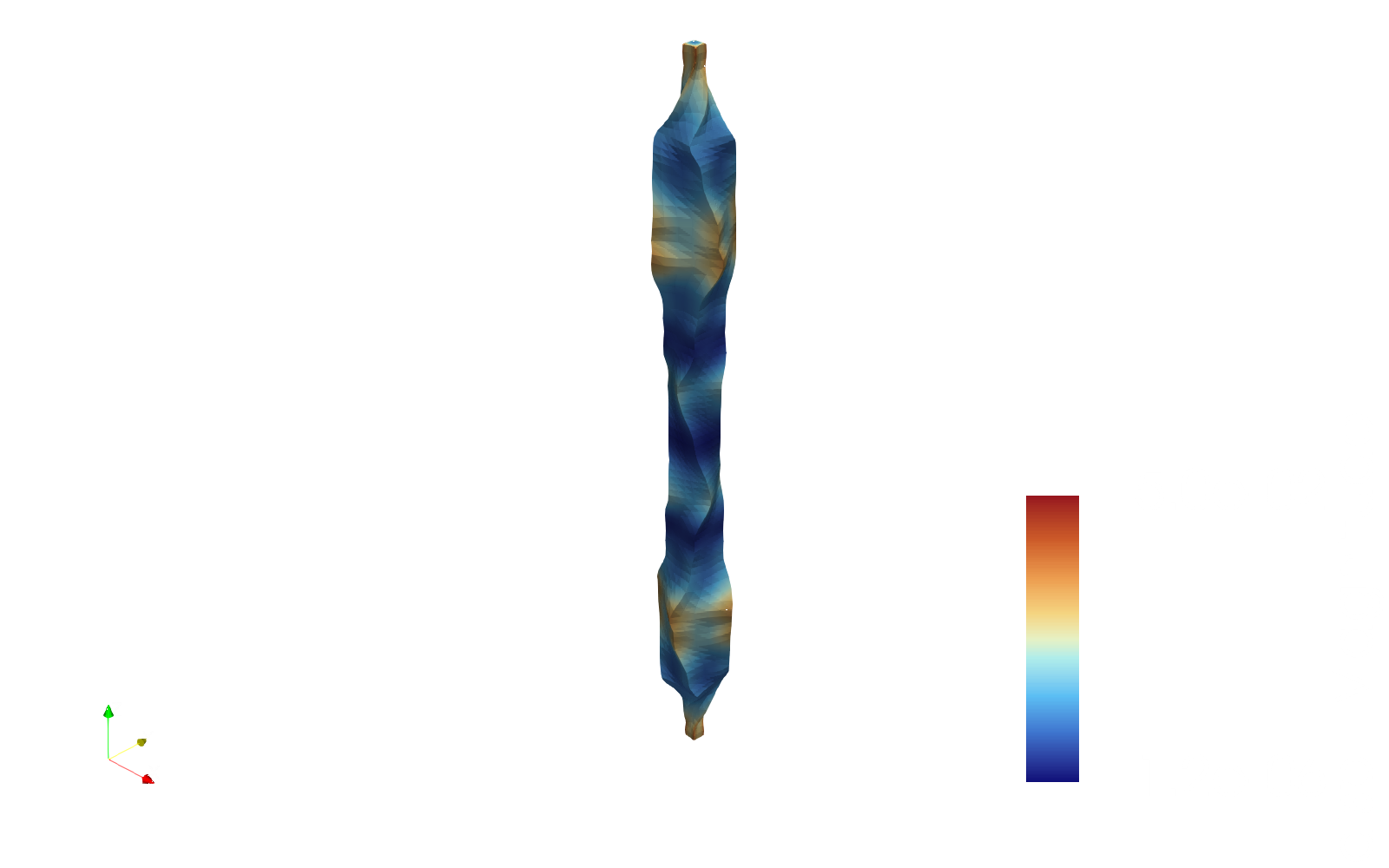}
\caption{NN-OpInf-SpML-Lagrangian}
\end{subfigure}

    \caption{Torsion problem. Displacement solutions at $t=3.75\times 10^{-3}$ s ($K=32$). The coloring is the displacement magnitude.}
\label{fig:torsion_paraview_t0050}
\end{center}
\end{figure}

\section{Conclusions}\label{sec:conclusions}
We presented NN-OpInf: a neural-network-based, structure-preserving, composable OpInf framework for non-intrusive reduced-order modeling. The approach learns a latent dynamics model from snapshot data by parameterizing a reduced velocity field with a sum of neural-network-based operators, each admitting a specified algebraic structure (e.g., skew-symmetry). The approach captures complex dynamics while preserving key local structure, providing a flexible mechanism for blending heterogeneous operator blocks while maintaining interpretability and physical constraints at the component level. We outlined the mathematical formulation of the approach as well as the open-source \code{NN-OpInf} software package in which the approach is implemented. 

We analyzed computational cost and convexity properties of NN-OpInf, including explicit FLOP comparisons against linear and quadratic P-OpInf and convexity results for linear structured operators. This analysis shows that the computational cost of the forward pass of NN-OpInf models is on the same order of magnitude as that for quadratic P-OpInf models, and that the primary increase in cost from NN-OpInf occurs in the (offline) training phase.  
This is due to NN-OpInf inference problems being non-convex in general, requiring iterative batch gradient-descent optimization algorithms where many forward and adjoint passes are executed.  However, note that this relative increase in training cost over P-OpInf reduces for large basis dimensions, while the online costs remain comparable to global quadratic models of the same dimension.  


Numerical experiments across several nonlinear and parametric systems demonstrate that the structured NN-OpInf models consistently improve accuracy and robustness relative to P-OpInf and vanilla neural network baselines, particularly in predictive regimes. The results also show that enforcing structure can stabilize ROM performance and improve generalization when polynomial structure is not adequate. Future work will focus on improved optimization strategies, richer operator libraries, dynamics-constrained training, uncertainty quantification for ROM predictions, and extensions to additional multiphysics systems.

\section{Acknowledgements}
Support for this work was received through Sandia National Laboratories’ Laboratory Directed Research
and Development (LDRD) program and through the U.S. Department of Energy, Office of Science,
Office of Advanced Scientific Computing Research, Mathematical Multifaceted Integrated Capability
Centers (MMICCs) program, under Field Work Proposal 22025291 and the Multifaceted Mathematics for
Predictive Digital Twins (M2dt) project. Additionally, the writing of this manuscript was funded in part
by Irina Tezaur’s Presidential Early Career Award for Scientists and Engineers (PECASE).
Sandia National Laboratories is a multi-mission laboratory managed and
operated by National Technology \& Engineering Solutions of Sandia, LLC
(NTESS), a wholly owned subsidiary of Honeywell International Inc.,
for the U.S. Department of Energy's National Nuclear Security
Administration (DOE/NNSA) under contract DE-NA0003525. This written
work is authored by an employee of NTESS. The employee, not NTESS,
owns the right, title and interest in and to the written work and is
responsible for its contents. Any subjective views or opinions that
might be expressed in the written work do not necessarily represent
the views of the U.S. Government. The publisher acknowledges that the
U.S. Government retains a non-exclusive, paid-up, irrevocable,
world-wide license to publish or reproduce the published form of this
written work or allow others to do so, for U.S. Government
purposes. The DOE will provide public access to results of federally
sponsored research in accordance with the DOE Public Access Plan.

\section{Declaration on the Use of Generative AI}
In preparing this manuscript, the authors used generative AI for manuscript formatting, assistance with table generation, general manuscript editing and review, and code documentation.

\end{document}